\newif\ifdraft \drafttrue
\newif\iffull \fulltrue

\documentclass[11pt,a4paper,final]{article}
\usepackage[left=1in,right=1in,top=1in,bottom=1in]{geometry}
\usepackage{authblk}

\usepackage[utf8]{inputenc} 
\usepackage[T1]{fontenc}    
\usepackage{url}            
\usepackage{booktabs}       
\usepackage{amsfonts}       
\usepackage{nicefrac}       
\usepackage{microtype}      
\usepackage{enumitem}      
\usepackage{algorithm}
\usepackage[noend]{algorithmic}
\usepackage{wrapfig} 
\makeatletter \@input{tex.flags} \makeatother 

\usepackage{amsmath, amssymb, amsthm}
\usepackage{mathtools}  
\usepackage{color}
\usepackage{xcolor} 
\usepackage{multicol}
\usepackage{array}
\usepackage{enumitem}   

\definecolor{DarkGreen}{rgb}{0.1,0.5,0.1}
\definecolor{DarkRed}{rgb}{0.5,0.1,0.1}
\definecolor{DarkBlue}{rgb}{0.1,0.1,0.5}
\usepackage[]{hyperref}
\hypersetup{
    unicode=false,          
    pdftoolbar=true,        
    pdfmenubar=true,        
    pdffitwindow=false,      
    pdftitle={},    
    pdfauthor={}
    pdfsubject={},   
    pdfnewwindow=true,      
    pdfkeywords={keywords}, 
    colorlinks=true,       
    linkcolor=DarkRed,          
    citecolor=DarkGreen,        
    filecolor=DarkRed,      
    urlcolor=DarkBlue,          
}

\usepackage{xspace}
\usepackage{cleveref}

\usepackage{thmtools} 
\usepackage{thm-restate}
 
\usepackage{natbib}
\usepackage{tikz}
\usepackage{bbding}
\usepackage{float}
\usepackage{dsfont}

\usepackage{hyperref}       

\usetikzlibrary{positioning}

\newcommand{\xhdr}[1]{\vspace{2mm} \noindent{\bf #1}}

\usepackage[suppress]{color-edits}
 \addauthor{sw}{blue}
 \addauthor{vs}{red}
 \addauthor{ls}{DarkGreen}
 \addauthor{dn}{magenta}

\newcommand\RR{\mathbb{R}}

\newcommand\cP{\mathcal{P}}

\newcommand\cU{\mathcal{U}}
\newcommand\cN{\mathcal{N}}

\newcommand\cC{\mathcal{C}}
\newcommand{\e}{\mathbf{e}}
\newcommand{\E}{\mathop{\mathbb{E}}}
\newcommand{\Prob}{\mathop{\mathbb{P}}}

\newcommand{\norm}[1]{\left\lVert#1\right\rVert}
\DeclarePairedDelimiter\abs{\lvert}{\rvert}%
\newcommand{\1}{\mathds{1}}

\newcommand{\bI}{\mathbf{I}}

\newcommand{\eps}{\varepsilon}
\def\epsilon{\varepsilon}

\DeclareMathOperator*{\argmin}{\mathrm{argmin}}
\DeclareMathOperator*{\argmax}{\mathrm{argmax}}

\newenvironment{proofsketch}{%
  \proof}{\endproof}

\newtheorem{theorem}{Theorem}[section]
\newtheorem{lemma}[theorem]{Lemma}

\newtheorem{claim}[theorem]{Claim}
\newtheorem{remark}[theorem]{Remark}
\newtheorem{corollary}[theorem]{Corollary}

\newtheorem{assumption}[theorem]{Assumption}
\newtheorem{example}[theorem]{Example}
\theoremstyle{definition}
\newtheorem{definition}[theorem]{Definition}

\title{Incentivizing Compliance with Algorithmic Instruments}
\author[1]{Daniel Ngo\thanks{Indicates equal contribution.}}
\author[1]{Logan Stapleton$^*$}
\author[2]{Vasilis Syrgkanis}
\author[3]{Zhiwei Steven Wu}
\affil[1]{University of Minnesota, \{ngo00054, stapl158\}@umn.edu}
\affil[2]{Microsoft Research, \{vasy\}@ microsoft.com}
\affil[3]{Carnegie Mellon University, \{zstevenwu\}@cmu.edu}

\date{\vspace{-1cm}}

\begin{document}
\maketitle











\begin{abstract}

Randomized experiments can be susceptible to selection bias due to potential non-compliance by the participants. While much of the existing work has studied compliance as a static behavior, we propose a game-theoretic model to study compliance as dynamic behavior that may change over time. In rounds, a social planner interacts with a sequence of heterogeneous agents who arrive with their unobserved private type that determines both their prior preferences across the actions (e.g., control and treatment) and their baseline rewards without taking any treatment. The planner provides each agent with a randomized recommendation that may alter their beliefs and their action selection. We develop a novel recommendation mechanism that views the planner's recommendation as a form of instrumental variable (IV) that only affects an agents' action selection, but not the observed rewards. We construct such IVs by carefully mapping the history --the interactions between the planner and the previous agents-- to a random recommendation. Even though the initial agents may be completely non-compliant, our mechanism can incentivize compliance over time, thereby enabling the estimation of the treatment effect of each treatment, and minimizing the cumulative regret of the planner whose goal is to identify the optimal treatment.
\end{abstract}


\section{Introduction}
In many applications, estimating the causal effect of a treatment or intervention is at the heart of a decision-making process. Examples include a study on the effect of a vaccine on immunity, an assessment of the effect of a training program on workers' efficiency, and a evaluation of the effect of a sales campaign on a company's profit. Many studies on causal effects rely on randomized experiments, which randomly assign each individual in a population to a treatment group or a control group and then estimate the causal effects by comparing the outcomes across groups. However, in many real-world domains, participation is voluntary, which can be susceptible to \emph{non-compliance}. For example, people may turn down a vaccine or a drug when they are assigned to receive the treatment \cite{matilda}. Another example is a randomized evaluation of the Job Training Partnership Act (JTPA) training program \cite{bloom}, where only 60 percent of the workers assigned to be trained chose to receive training, while roughly 2 percent of those assigned to the control group chose to receive training. In many cases, non-compliance can cause selection bias: for example, those who choose to receive the drug or vaccine in a randomized trial tend to be healthier, and those who join the training program might may more productive to begin with. 

Although non-compliance in randomized experiments has been well studied in many observational studies (see e.g. \citet{angrist_mostly_2008}), there has been little work that studies and models how compliance varies over time. In reality, however, participants' compliance behaviors may not be static: they may change according to their time-varying beliefs about the treatments. If the outcomes from the previous trials suggest that the treatments are effective, then the participants may become more willing to accept the recommendation. {For example, those initially weary about a new vaccine may change their mind once they see others take it without experiencing negative symptoms.}\footnote{A recent survey shows that many Americans choose to wait before deciding to receive a COVID-19 vaccine~\cite{covid2}.} 
Motivated by this observation, this paper studies the design of dynamic trial mechanisms that map history--the observations from previous trials--to a treatment recommendation and gradually incentivize compliance over time. 

In this paper, we introduce a game theoretic model to study the dynamic (non)-compliance behavior due to changing beliefs. In our model, there is a collection of treatments such that each treatment $j$ is associated with an unknown treatment effect $\theta_j$. We study an online learning game, in which a set of $T$ myopic agents arrive sequentially over $T$ rounds. Each agent $t$ has a private unobserved \emph{type} $u_t$, which determines their heterogeneous prior beliefs about the treatment effects. Each agent's goal is to select a treatment $j$ that maximizes the reward: $\theta_j + g^{(u_t)}_t$, where $g^{(u_t)}_t$ denotes the type-dependent baseline reward (without taking any treatment). We introduce a social planner whose goal is to estimate the effects of underlying treatments and incentive the agents to select the treatment that maximize long-term cumulative reward. Upon the arrival of each agent $t$, the planner provides the agent with a random treatment recommendation, which is computed by a policy that maps the history of interactions with the previous $(t-1)$ agents. While agent $t$ does not observe the previous history, they form a posterior belief over the treatment effects based on the recommendation and then select the action that maximizes their expected utility.

Under this model, we provide dynamic trial mechanisms that incentivize compliance over time and accurately estimate the treatment effects. The key technical insight is that the planner's random recommendation at each round can be viewed as an \emph{instrument} that is independent of the agent's private type and only influences the observed rewards through the agent's choice of action. By leveraging this observation, we can perform instrumental variable (IV) regression to recover the treatment effects, as long as some of the agents are compliant with the recommendations. To create compliance incentives, our mechanisms leverage techniques from the literature of \emph{incentivizing exploration} \cite{mansour2015bic, Slivkins19}. The key idea is \emph{information asymmetry}: since each agent does not directly observe the previous history, the planner has more information. By strategically mapping previous history to instruments, the planner can incentivize agents to explore treatments that are less preferred a-priori. 
 
We first focus on the binary action setting, where each agent can select treatment or control. Then we will extend our results to the $k$ treatments setting in \cref{sec:many-arms}. In the binary setting, we first provide two mechanisms that works with two initial non-compliance situations.

\xhdr{Complete non-compliance.} In \Cref{sec:sampling-control-treatment}, we consider a setting where the planner initially has no information about the treatment effect $\theta$, so all agents are initially non-compliant with the planner's recommendations. We provide \Cref{alg:sampling-control-treatment} which first lets initial agents choose their preferred arms, then constructs recommendations that incentivize compliance for some later agents. This enables treatment effect estimation through IV regression.

\xhdr{Partial compliance.} In \Cref{sec:racing-control-treatment}, we consider a setting where the planner has an initial estimate of the treatment effect $\theta$ (that may be obtained by running \Cref{alg:sampling-control-treatment}), so they can incentivize some agents to comply. We provide \Cref{alg:racing-two-types-mixed-preferences}, which can be viewed as the bandit algorithm \emph{active arm elimination} \cite{active-arms-elimination-2006} which uses IV estimates to compare treatments. Samples collected by \Cref{alg:sampling-control-treatment} provide an increasingly accurate estimate $\hat \theta$ and incentivize more agents to comply over time.

\xhdr{Regret minimization.} In \Cref{sec:combined-control-treatment}, we show that if the planner first runs \Cref{alg:sampling-control-treatment} to obtain an initial treatment effect estimate $\hat\theta$ and then runs \Cref{alg:racing-two-types-mixed-preferences} to amplify compliance, then he can achieve $\tilde O(\sqrt{T})$ regret w.r.t. the cumulative reward given by enforcing the optimal action for all agents. We then extend such a regret minimization policy to the setting with $k$ different treatments in \Cref{sec:many-arms}. 

\xhdr{Experiments.} Lastly, in \Cref{sec:experiment}, we complement our theoretical results with numerical simulations, which allow us to examine how parameters in agents' prior beliefs influence the convergence rate of our recommendation algorithm.

\subsection{Related Work}
We design mechanisms which strategically select instruments to incentivize compliance over time, so that we can apply tools from IV regression \citep{angrist2001iv, angrist1995tsls,angrist1996iv} to estimate causal effects. Although IV regression is an established tool to estimate causal effects where there is non-compliance in observational studies (see e.g. \citet{bloom,angrist06}), our results deviate significantly from previous works, due to the dynamic nature of our model. In particular, even if all agents are initially non-compliant, our mechanism can still incentivize compliance over time and estimate treatment effects ---whereas directly applying standard IV regression at the onset cannot.

Our work draws on techniques from the growing literature of incentivizing exploration (IE) \cite{KremerMP13,mansour2015bic, MansourSSW16, ImmorlicaMSW19, SS20}, where the goal is also to incentivize myopic agents to explore arms in a multi-armed bandit setting \citep{auer2002mab} using information asymmetry techniques from Bayesian persuasion~\cite{bp}. While our mechanisms are technically similar to those in \citet{mansour2015bic}, our work differs in several key aspects. First, prior work in IE ---including \citet{mansour2015bic}--- does not capture selection bias and cannot be directly applied in our setting to recover causal effects. The mechanism in \citet{mansour2015bic} aims to enforce full compliance (also called \emph{Bayesian incentive-compatibility}) that requires all agents to follow the planner's recommendations: as a result, the mechanism needs to cater to the type of agents that are most difficult to convince. By contrast, our mechanism relies only on the compliance of a partial subset of agents in order to obtain accurate estimates.

There has also been a line of work on mechanisms that incentivize exploration via payments \cite{iemoney, chen18a, fairie}. There are several known disadvantages of such payment mechanisms, including potential high costs and ethical concerns \cite{groth2010honorarium}. See \citet{slivkins17} for a detailed discussion.

Thematically, our work relates to work on ``instrument-armed bandits'' by \citet{Kallus2018InstrumentArmedB}, which also views arm recommendations as instruments. However, the compliance behavior (modeled as a fixed stochastic mapping from instrument to treatments) is static in  \citet{Kallus2018InstrumentArmedB}: it does not change over time ---even if the planner has obtained accurate estimate(s) of the treatment effect(s). By comparison, since all agents eventually become compliant in our setting, we can achieve sublinear regret w.r.t. the best treatment, which is not achievable in a static compliance model.

\section{Treatment-Control Model}\label{sec:model}
We study a sequential game between a \textit{social planner} and a sequence of \textit{agents} over $T$ rounds, where $T$ is known to the social planner. We will first focus on the binary setting with a single treatment, and study the more general setting of $k$ treatments in Section~\ref{sec:many-arms}. In the binary setting, the treatment of interest has unknown effect $\theta \in [-1, 1]$. In each round $t$, a new agent indexed by $t$ arrives with their \emph{private type} $u_t$ drawn independently from a distribution $\cU$ over the set of all private types $U$. Each agent $t$ has two actions to choose from: taking the treatment (denoted as $x_t=1$) and not taking the treatment, i.e. the control (denoted as $x_t = 0$). Upon arrival, agent $t$ also receives an action recommendation $z_t\in\{0,1\}$ from the planner.  After selecting an action $x_t\in\{0,1\}$, agent $t$ receives a reward $y_t\in\mathbb{R}$, given by
\begin{equation}
    \label{eq:reward-model}
    y_t = \theta{x_t} + g^{(u_t)}_t
\end{equation}
where $g^{(u_t)}_t$ denotes the confounding \emph{baseline reward} which depends on the agent's private type $u_t$; each is drawn from a sub-Gaussian distribution with a sub-Gaussian norm of $\sigma_g$. The social planner's goal is to estimate the treatment effect $\theta$ and maximize the total expected reward of all $T$ agents.

\xhdr{History and recommendation policy.}
The interaction between the planner and the agent $t$ is given by the tuple $(z_t, x_t, y_t)$. 
For each $t$, let $H_t$ denote the \textit{history} from round 1 to $t$, i.e. the sequence of interactions between the social planner and the first $t$ agents, such that $H_t:=\left((z_1, x_1, y_1), \ldots ,(z_t, x_t, y_t) \right)$. Before the game starts, the social planner commits to a recommendation policy $\pi = (\pi_t)_{t=1}^T$ where each $\pi_t \colon (\{0,1\}\times \{0,1\} \times \mathbb{R})^{t-1} \rightarrow \Delta(\{0,1\})$ is a randomized mapping from the history $H_{t-1}$ to recommendation $z_t$. Policy $\pi$ is fully known to all agents.

\xhdr{Beliefs, incentives, and action choices.}
Each agent $t$ knows their place $t$ in the sequential game, and their private type $u_t$ determines a \emph{prior belief} $\cP^{(u_t)}$, which is a joint distribution over the treatment effect $\theta$ and noisy error term $g^{(u)}$. Agent $t$ selects action $x_t$ as such:
\begin{equation}
    \label{eq:control-treatment-selection-function}
    x_t := \1\left[ \E_{\cP^{(u_t)}, \pi_t}\left[\theta \ | \ z_t, t \right] > 0 \right].
\end{equation}

An agent $t$ is \emph{compliant} with a recommendation $z_t$ if the agent chooses the recommended action, i.e. $x_t = z_t$. We'll also say that a recommendation is \emph{compliant} if $x_t = z_t$.

\Cref{fig:model-dag} shows the causal diagram for this setting.

\begin{figure}
    \centering
        \begin{tikzpicture}[node distance={15mm}, thick, main/.style = {draw, circle}] 
        \node[draw,circle,fill=lightgray,text=black] (1) {$z_t$}; 
        \node[draw,circle,fill=lightgray,text=black] (2) [right of=1] {$x_t$}; 
        \node[draw,circle] (3) [below of=2] {$u_t$};
        \node[draw,circle,fill=lightgray,text=black] (4) [right of=2] {$y_t$};
        \node[draw,circle] (5) [right of=3] {$g_t$};
        \draw[->] (1) -- (2);
        \draw[->] (3) -- (2);
        \draw[->] (3) -- (5);
        \draw[->] (5) -- (4);
        \draw[->] (2) -- node[above] {$\theta$} (4); 
        \end{tikzpicture} 
    \caption{Causal diagram of our setting at any time $t$. Grey nodes are observed, white unobserved. Recommendation $z_t$ influences treatment $x_t$, which influences outcome $y_t$. Unobserved type $u_t$ influences both treatment $x_t$ and outcome $y_t$ ---the latter via error term $g_t^{(u_t)}$.}
    \label{fig:model-dag}
\end{figure}
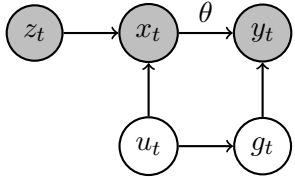

\subsection{Recommendations as Instruments}

Unlike the standard multi-armed bandit and previous models on incentivizing exploration \cite{mansour2015bic, MansourSSW16}, the heterogeneous beliefs in our setting can lead to selection bias. For example, agents who are willing to take the treatment may also have higher baseline rewards. Thus, simply comparing rewards across the treatment group ($x=1$) and the control group ($x=0$) will lead to a biased estimate of $\theta$. To overcome this selection bias, we will view the planner's recommendations as instruments and perform \textit{instrumental variable (IV) regression} to estimate $\theta$. There are two criteria for recommendation $z_t$ to be a valid instrument: (1) $z_t$ influences the selection $x_t$, and (2) $z_t$ is independent from the noisy baseline reward $g^{(u)}$. See \Cref{fig:model-dag} for a graphical explanation of how these criteria will be satisfied in our setting. Criterion (2) follows because planner chooses $z_t$ randomly, independent of the type $u_t$. Our goal is to design a recommendation policy to meet criterion (1).

\xhdr{Wald Estimator.}
Our mechanism periodically solves the following IV regression problem:
given a set $S$ of $n$ observations $(x_i, y_i, z_i)_{i=1}^n$, compute an estimate $\hat \theta_S$ of $\theta$. We consider the following two-stage least square (2SLS) or Wald estimator (which are equivalent for binary treatments):
\begin{equation}
    \label{eq:theta-hat}
    \hat{\theta}_S  = \frac{\sum_{i=1}^n (y_i - \bar{y})(z_i - \bar{z})}{\sum_{i=1}^n (x_i - \bar{x})(z_i - \bar{z})},
\end{equation}
where $\bar x, \bar y, \bar z$ denote the empirical means of variables $x_i$, $y_i$, and $z_i$ respectively.

While existing work on IV regression mostly focuses on asymptotic analyses, we provide a high-probability finite-sample {error bound} for  $\hat{\theta}_S$, which is required by our regret analysis and may be of independent interest.

\begin{restatable}[Finite-sample error bound for Wald estimator]{theorem}{treatmentapproximationbound}\label{thm:treatment-approximation-bound} 
Let $z_1, z_2, \ldots, z_n\in \{0 , 1\}$ be a sequence of instruments. Suppose there is a sequence of $n$ agents such that each agent $i$ has their private type $u_i$ drawn independently from $\cU$, selects action $x_i$ under instrument $z_i$, and receives reward $y_i$. Let sample set $S=(x_i,y_i,z_i)_{i=1}^n$. Let $A:\left(\{0,1\}^n\times\{0,1\}^n\times\RR^n\right)\rightarrow\RR$ denote the approximation bound for set $S$, such that
\[A(S,\delta):=\frac{2\sigma_g \sqrt{2n\log(2/\delta)}}{\left|\sum_{i=1}^n(x_i-\bar{x})(z_i-\bar{z})\right|}\] 
and the Wald estimator given by \eqref{eq:theta-hat} satisfies
\[\left|\hat{\theta}_S-\theta\right| \leq A(S,\delta)\]
with probability at least $1 - \delta$, for any $\delta \in (0,1)$.
\end{restatable}

\emph{Proof Sketch.} See \Cref{sec:approximation-bound-proof} for the full proof. The bound follows by substituting our expressions for $y_t, x_t$ into the IV regression estimator, applying the Cauchy-Schwarz inequality to split the bound into two terms (one dependent on $\{(g_t, z_t)\}^{\abs{S}}_{t=1}$ and one dependent on $\{(x_t, z_t)\}^{\abs{S}}_{t=1}$), and bound the second term with high probability.

Note that the error rate above depends on the covariance between the instruments $z$ and action choices $x$. In particular, when $\sum_{i=1}^n(x_i-\bar{x})(z_i-\bar{z})$ is linear in $n$, the error rate becomes $\tilde O(1/\sqrt{n})$. In the following sections, we will provide mechanisms that incentivize compliance so that the instruments $z$ are correlated with actions $x$, enabling us to achieve such an error rate.


\section{Overcoming Complete Non-Compliance}
\label{sec:sampling-control-treatment}

In this section, we present a recommendation policy that incentivizes compliance to enable IV estimation. We focus on a setting where the agents are initially completely non-compliant: since the planner has no information about the treatment effect in the initial rounds, the recommendations have no influence on agents' action selections. For simplicity of exposition, we will present our policy in a setting where there are two types of agents who are initially ``always-takers'' and ``never-takers.'' As we show later \Cref{sec:many-arms}, this assumption can be relaxed to have arbitrarily many types and also allow all types to be ``always-takers.''

Formally, consider two types of agents $i\in\{0,1\}$. For type $i$, let $p_i$ be the fraction of agents in the population, $\cP^{(i)}$ the prior beliefs, and $g^{(i)}$ the baseline reward random variables. Agents of type 1 initially prefer the treatment and type 0 agents prefer control: their prior means for $\theta$ satisfy $\mu^{(1)} = \E_{\cP^{(1)}}[\theta] > 0$ and $\mu^{(0)} = \E_{\cP^{(0)}}[\theta] < 0$.

Our policy (\Cref{alg:sampling-control-treatment}) splits into two stages. In the first stage, agents take their preferred action according to their prior beliefs: type 0 agents choose control and type 1 treatment. This allows us to collect $\ell_0$ and $\ell_1$ observations of rewards for $x=0$ and $x=1$, respectively. Let $\bar y^0$ and $\bar y^1$ denote the empirical average rewards for the two actions, respectively. Note that since the baseline rewards $g^{(u)}$ are correlated with the selections $x$, the difference $(\bar y^1 - \bar y^0)$ is a biased estimate for $\theta$. 

In the second stage, we use this initial set of reward observations to construct valid instruments which incentivize agents of one of the two types to follow both control and treatment recommendations. Without loss of generality, we focus on incentivizing compliance among type 0 agents. Since they already prefer control, the primary difficulty here is to incentivize type 0 agents to comply with treatment recommendations.\footnote{We could instead incentivize type 1 agents to take control. This would require 1) rewriting event $\xi$ so it indicates that the expectation of $\theta$ over $\cP^{(1)}$ must be negative and 2) rewriting \Cref{alg:sampling-control-treatment} so that control is recommended when exploring. We cannot incentivize both types to comply at the same time.} We leverage the following observation: according to the prior $\cP^{(0)}$ of type 0 agents, there is a non-zero probability that the biased estimate $(\bar y^1 - \bar y^0)$ is so large that $\theta$ must be positive.

Formally, consider the following event for the average rewards $\bar y^0$ and $\bar y^1$:
\begin{small}
\begin{equation}
\label{eq:xi}
    \hspace{-.5mm} \xi \hspace{-.5mm} = \hspace{-.75mm}  \bigg\{ \bar y^1 \hspace{-1.25mm} > \hspace{-.5mm} \bar y^0 + \sigma_g\hspace{-.55mm} \bigg(\hspace{-1.25mm} \sqrt{\frac{2\log(2/\delta)}{\ell_0}} \hspace{-.25mm}  + \hspace{-.25mm}  \sqrt{\frac{2\log(2/\delta)}{\ell_1}}\bigg) \hspace{-.5mm} +  G^{(0)} \hspace{-.5mm} + \frac{1}{2} \hspace{-1mm} \bigg\} \hspace{-.5mm} 
\end{equation}
\end{small}
where $G^{(0)}$ is a constant such that $G^{(0)}>\E_{\cP^{(0)}}[g^{(1)}-g^{(0)}]$ and $\sigma_g$ is the variance parameter for $g^{(0)}$ and $g^{(1)}$.

\begin{assumption}[Knowledge Assumption for \Cref{alg:sampling-control-treatment}]\label{ass:know-sampling}
Within \Cref{sec:sampling-control-treatment}, the following are common knowledge among agents and planner:\footnote{Assumptions do not hold elsewhere, unless explicitly stated.}
    \begin{enumerate}[itemsep=0mm]
        \item Type 0 agents prefer control and type 1 agents prefer treatment. The fraction of agents of type 0 in the population is $p_0\geq0$ and the fraction of type 1 is $p_1>0$.
        \item Type 0's prior treatment effect mean $\mu^{(0)}$ and the probability of event $\xi$, denoted $\Prob_{\cP^{(0)}}[\xi]$, over the prior $\cP^{(0)}$ of type 0.\footnote{These assumptions (as well as \Cref{ass:know-racing} and \Cref{ass:know-combined}) require only partial knowledge of the priors for compliant agents only. They are no more restrictive than the least restrictive (detail-free) assumptions of \citet{mansour2015bic}.}
    \end{enumerate}
\end{assumption}
 
 \lsdelete{ 
\begin{algorithm}[ht!]
    \caption{Overcoming complete non-compliance}
    \label{alg:sampling-control-treatment}
\begin{algorithmic}
   \STATE {\bfseries Input:} exploration probability $\rho\in(0,1)$, $\ell\in\mathbb{N}$ (assume w.l.o.g. $\rho\ell\in\mathbb{N}$), minimum first stage samples $\ell_0, \ell_1 \in \mathbb{N}$, and failure probability $\delta<\Prob_{\cP_0}[\xi]/8$
   \STATE \textbf{1st stage:} The first $\displaystyle2\max\left(\ell_0/p_0,\ell_1/p_1\right)$ agents are given no recommendation (they choose what they prefer)
   \vspace{-3mm} \STATE \textbf{2nd stage:} Based on at least $\ell_0$ control and $\ell_1$ treatment samples collected in the first stage:
   \IF{\begin{small}$\bar y^1 > \bar y^0 + \sigma_g\bigg(\sqrt{\frac{2\log(2/\delta)}{\ell_0}} + \sqrt{\frac{2\log(2/\delta)}{\ell_1}}\bigg) + G^{(0)} + \frac{1}{2}$\end{small}} 
   \vspace{-4.5mm} \STATE \hspace{3.5mm} $a^*=1$
   \ELSE
   \vspace{-4.5mm} \STATE \hspace{3.5mm} $a^*=0$
   \ENDIF
   \STATE From the next $\ell$ agents, pick $\rho\ell$ agents uniformly at random to be in the explore set $E$
   \FOR{the next $\ell$ rounds}
   \IF{agent $t$ is in explore set $E$} 
   \vspace{-4.5mm} \STATE  \hspace{44mm} $z_t = 1$
   \ELSE
   \vspace{-4.5mm} \STATE \hspace{3.5mm} $z_t = a^*$
   \ENDIF
   \ENDFOR
\end{algorithmic}
\end{algorithm}
}

\begin{algorithm}[ht!]
    \caption{Overcoming complete non-compliance}
    \label{alg:sampling-control-treatment}
\begin{algorithmic}
   \STATE {\bfseries Input:} exploration probability $\rho\in(0,1)$, $\ell\in\mathbb{N}$ (assume w.l.o.g. $\rho\ell\in\mathbb{N}$), minimum first stage samples $\ell_0, \ell_1 \in \mathbb{N}$, and failure probability $\delta<\Prob_{\cP_0}[\xi]/8$
   \STATE \textbf{1st stage:} The first $\displaystyle2\max\left(\ell_0/p_0,\ell_1/p_1\right)$ agents are given no recommendation (they choose what they prefer)
   \STATE \textbf{2nd stage:} Based on at least $\ell_0$ control and $\ell_1$ treatment samples collected in the first stage:
   \IF{\begin{small}$\bar y^1 > \bar y^0 + \sigma_g\bigg(\sqrt{\frac{2\log(2/\delta)}{\ell_0}} + \sqrt{\frac{2\log(2/\delta)}{\ell_1}}\bigg) + G^{(0)} + \frac{1}{2}$\end{small}} 
   \STATE $a^*=1$
   \ELSE
   \STATE $a^*=0$
   \ENDIF
   \STATE From the next $\ell$ agents, pick $\rho\ell$ agents uniformly at random to be in the explore set $E$
   \FOR{the next $\ell$ rounds}
   \IF{agent $t$ is in explore set $E$} 
   \STATE $z_t = 1$
   \ELSE
   \STATE $z_t = a^*$
   \ENDIF
   \ENDFOR
\end{algorithmic}
\end{algorithm}

We prove that \Cref{alg:sampling-control-treatment} is compliant for agents of type 0 as long as the exploration probability $\rho$ is less than some constant that  depends on prior $\cP^{(0)}$. When an agent of type 0 is recommended treatment, they do not know whether this is due to exploration or exploitation. However, with small enough $\rho$, their expected gain from exploiting exceeds the expected loss from exploring. Hence, the agents comply with the recommendation and take treatment.

\begin{restatable}[Type 0 compliance with \Cref{alg:sampling-control-treatment}]{lemma}{bicsamplingcontroltreatment}\label{lemma:bic-sampling-control-treatment} Under \Cref{ass:know-sampling}, any type 0 agent who  arrives in the last $\ell$ rounds of \Cref{alg:sampling-control-treatment} is compliant with any recommendation, as long as the exploration probability $\rho$ satisfies
    \begin{equation}
        \label{eq:rho}
        \rho \leq 1 + \frac{4\mu^{(0)}}{\Prob_{\cP^{(0)}}[\xi]-4\mu^{(0)}}
    \end{equation}
where the event $\xi$ is defined above in \Cref{eq:xi}. 
\end{restatable}

\begin{proofsketch} See \Cref{sec:sampling-stage-appendix} for the full proof. The proof follows by expressing the compliance condition for type 0 agents as different cases, depending on the recommendation. By keeping the exploration probability $\rho$ small with regard to type 0 agent's prior-dependent probability $\Prob_{\cP^{(0)}}[\xi]$ and the conditional expected treatment effect  $\E_{\cP^{(0)}}[\theta |\xi]$, the expected gain from exploiting is greater than the expected loss from exploring. Hence, type 0 agents would comply with the recommendation. We further simplify the condition on exploration probability $\rho$ by applying high probability bound on the samples collected from the 1st stage (where no recommendations were given). 
\end{proofsketch}

We also provide a separate accuracy guarantee for the treatment effect estimate $\hat{\theta}$ at the end of the \Cref{alg:sampling-control-treatment}.  

\begin{restatable}[Treatment Effect Confidence Interval after \Cref{alg:sampling-control-treatment}]{theorem}{samplingestimationbound}\label{thm:sampling-estimation-bound}
    With sample set $S_\ell=(x_i,y_i,z_i)_{i=1}^\ell$ of $\ell$ samples collected from the second stage of \Cref{alg:sampling-control-treatment} ---run with exploration probability $\rho$ small enough so that type 0 agents are compliant (see \Cref{lemma:bic-sampling-control-treatment}),--- approximation bound $A(S_\ell,\delta)$ satisfies the following, with probability at least $1-\delta$:
    \[\left|\hat{\theta}_{S_{\ell}}-\theta\right|\leq A(S_\ell,\delta) \leq \frac{2\sigma_g \sqrt{2\log(5/\delta)}}{\rho(1-\rho)p_0\sqrt{\ell}-(3-\rho)\sqrt{\frac{\rho\log(5/\delta)}{2(1-\rho)}}}\]
    for any $\delta\in(0,1)$. Recall $\sigma_g$ is the variance of $g^{(u_i)}$, $p_0$ is the fraction of compliant never-takers in the population of agents,\footnotemark and $A(S_\ell,\delta)$ is defined as in \Cref{thm:treatment-approximation-bound}.
\end{restatable}

\begin{proofsketch} See \Cref{sec:sampling-stage-appendix} for the full proof. Note that  \Cref{thm:treatment-approximation-bound} applies, so we only have to bound the denominator term which is dependent on $\{(x_t, z_t)\}^{\abs{S_\ell}}_{t=1}$. We assume that \Cref{alg:sampling-control-treatment} is initialized with parameters (see \Cref{lemma:bic-sampling-control-treatment}) such that type 0 agent is compliant. We bound the term dependent on $\{(x_t, z_t)\}^{\abs{S_\ell}}_{t=1}$ with high probability.
\footnotetext{We redefine $p_0$ here to be applicable to more general settings.}
\end{proofsketch}

\subsection{\Cref{alg:sampling-control-treatment} Extensions}\label{sec:sampling-extensions}
\Cref{alg:sampling-control-treatment} can be extended to handle more general settings:
\begin{enumerate}[itemsep=0mm]
    \item There can be arbitrarily many types of agents that do not share the same prior. In this case, let $\E_{\cP^{(u)}}[g^0]$ and $\E_{\cP^{(u)}}[g^1]$ denote the expected baseline rewards for never-takers and always-takers, respectively, over the prior $\cP^{(u)}$ of any type $u$ and $G^{(u)}>\E_{\cP^{(u)}}[g^1-g^0]$. Then, \Cref{alg:sampling-control-treatment} can still incentivize any never-taker type $u$ agents to comply as long as the planner has a lower bound on $\Prob_{\cP^{(u)}}[\xi^{(u)}]$, where $\xi^{(u)}$ is defined just as $\xi$ in \Cref{eq:xi}, except $G^{(0)}$ is replaced with $G^{(u)}$. \Cref{thm:sampling-estimation-bound} applies as is.
    \item All types can be always-takers (who prefer the treatment). The algorithm can incentivize some of the agents to take control with an event $\xi$ defined without $\bar{y}^0$ and flipped (i.e. the mean treatment reward is much lower than the expected baseline reward).\footnote{Also, \Cref{lemma:bic-sampling-control-treatment-extension} can be proved sans clean event $C_0$.}
\end{enumerate}

By \Cref{thm:sampling-estimation-bound}, samples collected from \Cref{alg:sampling-control-treatment} produce a confidence interval on the treatment effect $\theta$ which decreases proportionally to $1/\sqrt{t}$ by round $t$. However, it still decreases slowly because the exploration probability $\rho$ is small (see roughly how small in \Cref{sec:experiment}). In \Cref{sec:racing-control-treatment}, we give an algorithm for which this confidence interval improves quicker and works for arbitrarily many types.

\section{Overcoming Partial Non-Compliance}
\label{sec:racing-control-treatment}
 In this section, we present a recommendation policy which (1) capitalizes on partial compliance, eventually incentivizing all agents to comply, and (2) determines whether the treatment effect is positive or not (with high probability). \Cref{alg:racing-two-types-mixed-preferences} recommends control and treatment sequentially (one after the other). \Cref{lemma:bic-racing-control-treatment-0} gives conditions for partial compliance from the beginning of \Cref{alg:racing-two-types-mixed-preferences}, given access to initial samples which form a crude estimate of the treatment effect. \Cref{thm:racing-estimation-bound-first} demonstrates how rapidly this estimate improves throughout \Cref{alg:racing-two-types-mixed-preferences}, which solely depends on the fraction of compliant agents (and not on some fraction like $\rho$ with \Cref{alg:sampling-control-treatment}). More (and eventually all) types of agents progressively become compliant throughout \Cref{alg:racing-two-types-mixed-preferences}.
 
    \begin{assumption}[Knowledge Assumption for \Cref{alg:racing-two-types-mixed-preferences}]\label{ass:know-racing}
        Within \Cref{sec:racing-control-treatment}, the following are common knowledge among agents and planner:
        \begin{enumerate}[itemsep=0mm]
            \item The fraction of agents in the population who prefer control is $p_0\geq0$; that who prefer treatment is $p_1\geq0$.
            \item For each type $u$ and for some $\tau$ (which can differ per $u$), the probability $\tau\Prob_{\cP^{(u)}}[\theta>\tau]$ is known if $\E_{\cP^{(u)}}[\theta]<0$; or $\tau\Prob_{\cP^{(u)}}[\theta<-\tau]$ is known if $\E_{\cP^{(u)}}[\theta]\geq0$.
        \end{enumerate}
    \end{assumption}

    \begin{algorithm}[ht!]
        \caption{Overcoming partial compliance}
        \label{alg:racing-two-types-mixed-preferences}
    \begin{algorithmic}
    \STATE{\bfseries Input:} samples $S_0:=(x_i,z_i,y_i)_{i=1}^{|S_0|}$ which meet \Cref{thm:treatment-approximation-bound} conditions and produce IV estimate $\hat{\theta}_{S_0}$;\footnotemark\, time horizon $T$; number of recommendations of each action per phase $h$; approximation bound failure probability $\delta$;
    \STATE Split the remaining rounds (up to $T$) into consecutive phases of $h$ rounds each, starting with $q=1$;\\
    \STATE Let $\hat{\theta}_0=\hat{\theta}_{S_0}$ and $A_0=A(S_0,\delta)$;
        \WHILE{$|\hat{\theta}_{q-1}| \leq A_{q-1}$}
        \STATE The next $2h$ agents are recommended control and treatment sequentially (one after the other);\\
        \STATE Let $S_q$ be samples up to and including phase $q$, i.e.
        $S_q:=(x_i,z_i,y_i)_{i=1}^{|S_0|+hq}=S_{q-1}+\{\text{round } q \text{ samples}\}$\\
        \STATE Let $S_q^{\text{BEST}}$ be the samples with smallest approximation bound so far (from phase 1 to $q$), i.e.\\
        $S_q^{\text{BEST}}=\argmin_{S_r, 0\leq r\leq q} A(S_r,\delta)$;\\
        \STATE Define $\hat{\theta}_q=\hat{\theta}_{S_q^{\text{BEST}}}$ and $A_q=A(S_q^{\text{BEST}},\delta)$;\\
        \STATE $q = q+1$;
        \ENDWHILE
        
        For all remaining agents recommend $a^*=\1\left[\hat{\theta}_q>0\right]$.
    \end{algorithmic}
    \end{algorithm}
    
    \footnotetext{Operator $|\cdot|$ denotes the cardinality of a set.}

    We focus on a setting where agents are assumed to have been at least partially compliant in the past, such that we may form an IV estimate from the history. The social planner employs \Cref{alg:racing-two-types-mixed-preferences}, which is a modification of the \textit{Active Arms Elimination} algorithm \cite{active-arms-elimination-2006}. Treatment and control ``race'', i.e. are recommended sequentially, until the expected treatment effect is known to be negative or positive (with high probability). Then, the algorithm recommends the ``winner'' (the action with higher expected reward) for the remainder of the time horizon $T$.
 
    The compliance incentive works as such: when an agent is given a recommendation, they do not know whether it is because the action is still in the ``race'' or if the action is the ``winner''. When the algorithm is initialized with samples that form an IV estimate which is sufficiently close to the true treatment effect (according to the agent's prior), then the probability that any recommended action has ``won'' is high enough such that the agent's expected gain from taking a ``winning'' action outweighs the expected loss from taking a ``racing'' one. We formalize this in \Cref{lemma:bic-racing-control-treatment-0}.

    \begin{restatable}[\Cref{alg:racing-two-types-mixed-preferences} Partial Compliance]{lemma}{bicracingcontroltreatmentzero}\label{lemma:bic-racing-control-treatment-0}
        Recall that \Cref{alg:racing-two-types-mixed-preferences} is initialized with input samples $S_0=(x_i,y_i,z_i)_{i=1}^{|S_0|}$. For any type $u$ with the following prior preference (control or treatment), if $S_0$ satisfies the following condition, with probability at least $1-\delta$, then all agents of type $u$ will comply with recommendations of \Cref{alg:racing-two-types-mixed-preferences}:
        \[
        A(S_0,\delta) \leq \begin{cases}
             \tau\Prob_{\cP^{(u)}}[\theta>\tau]/4 & \text{ if } \E_{\cP^{(u)}}[\theta]<0;\\
             \tau\Prob_{\cP^{(u)}}[\theta<-\tau]/4 & \text{ if } \E_{\cP^{(u)}}[\theta]\geq0,
        \end{cases}
        \]
        for some $\tau\in(0,1)$, where $A(S_0,\delta)$ is the approximation bound for $S_0$ and any $\delta\in(0,1)$ (see \Cref{thm:treatment-approximation-bound}).
    \end{restatable}

    \begin{proofsketch} See \Cref{sec:bic-racing-type-0} for the full proof. The proof follows by using a ``clean event'' analysis where the IV estimated treatment effect $\hat{\theta}$ is close to the true treatment effect $\theta$. We split the conditional expected treatment effect $\E_{\cP^{(u)}}[\theta]$ into different cases for the value of $\theta$. With an IV estimate that is sufficiently close to the true treatment effect, the expected gain from exploiting (taking the ``winning'' action) is greater than the expected loss from exploring (taking a recommended action when the ``race'' is not over) and any agent of type $u$ will comply with recommendation.
    \end{proofsketch}
    
    When a nonzero fraction of agents comply from the beginning, the samples gathered in \Cref{alg:racing-two-types-mixed-preferences} provide treatment effect estimates $\hat{\theta}$ which become increasingly accurate over rounds. In the following \Cref{thm:racing-estimation-bound-first}, we provide a high probability guarantee on this accuracy.
    
    \begin{restatable}[Treatment Effect Confidence Interval from \Cref{alg:racing-two-types-mixed-preferences} with Partial Compliance]{theorem}{racingestimationboundfirst}\label{thm:racing-estimation-bound-first}
        With set $S=(x_i,y_i,z_i)_{i=1}^{|S|}$ of $|S|$ samples collected from \Cref{alg:racing-two-types-mixed-preferences} where $p_c$ is the fraction of compliant agents in the population, we form an estimate $\hat{\theta}_S$ of the treatment effect $\theta$. With probability at least $1-\delta$,
        \[\left|\hat{\theta}_S-\theta\right| \leq A(S,\delta) \leq \frac{8\sigma_g\sqrt{2\log(5/\delta)}}{p_c\sqrt{|S|}-\sqrt{50\log(5/\delta)}}\]
        for any $\delta\in(0,1)$, where $\sigma_g$ is the variance of $g^{(u_i)}$.
    \end{restatable}
    \begin{proofsketch} See \Cref{sec:racing-estimation-bound-first} for a full proof. Note that \Cref{thm:treatment-approximation-bound} applies, so we only to have to bound the denominator term which is dependent on $\{(x_t, z_t)\}^{\abs{S}}_{t=1}$. We assume that \Cref{alg:racing-two-types-mixed-preferences} is initialized with parameters such that $p_c>0$ fraction of the population complies with all recommendations. We bound the term dependent on $\{(x_t, z_t)\}^{\abs{S}}_{t=1}$ with high probability. 
    \end{proofsketch}
    
    Agents become compliant during \Cref{alg:racing-two-types-mixed-preferences} for the same reason others become compliant from the beginning: they expect that the estimate $\hat{\theta}$ is sufficiently accurate and it's likely they're getting recommended an action because it won the race. For large enough $T$, all agents will become compliant.\footnotemark\, Note that the accuracy improvement in \Cref{thm:racing-estimation-bound-first} relies solely on the proportion of agents $p_c$ who comply from the beginning of \Cref{alg:racing-two-types-mixed-preferences}, which relies on the accuracy of the approximation bound given by initial samples $S_0$. Thus, if the social planner can choose more accurate $S_0$, then the treatment effect estimate $\hat{\theta}$ given by samples from \Cref{alg:racing-two-types-mixed-preferences} becomes more accurate quicker. In \Cref{sec:combined-control-treatment}, we present a recommendation policy in which $S_0$ can be chosen by running \Cref{alg:sampling-control-treatment}.
    
    \footnotetext{See \Cref{lemma:racing-full-compliance-control-treatment} for details.}

\section{Combined Recommendation Policy}
\label{sec:combined-control-treatment}

    In this section, we present a recommendation policy $\pi_c$, which spans $T$ rounds and runs \Cref{alg:sampling-control-treatment,alg:racing-two-types-mixed-preferences} in sequence. This policy achieves $\tilde{O}(\sqrt{T})$ regret for sufficiently large $T$ and produces an estimate $\hat{\theta}$ which deviates from the true treatment effect $\theta$ by $O(1/\sqrt{T})$.\footnotemark\
    
    \footnotetext{We spare the reader the details of the exact bound. It can be deduced via \Cref{thm:sampling-estimation-bound,thm:racing-estimation-bound-first,thm:racing-estimation-bound-second} and \Cref{lemma:bic-racing-control-treatment-0,lemma:racing-full-compliance-control-treatment}.}
    
    \begin{assumption}[Knowledge Assumption for Policy $\pi_c$]\label{ass:know-combined}
        Within \Cref{sec:recommendation-policy-pi,sec:regret-control-treatment}, the following are common knowledge among agents and planner:
        \begin{enumerate}[itemsep=0mm]
            \item All prior-dependent constants given in \Cref{ass:know-racing}
            \item For each type $u$ which prefers control, prior mean $\mu^{(u)}$ and a lower bound on the probability $\Prob_{\cP^{(u)}}[\xi^{(u)}]$ (defined in Extension 1 of \Cref{alg:sampling-control-treatment} from \Cref{sec:sampling-extensions})
        \end{enumerate}
    \end{assumption}
    
    
    \subsection{Recommendation Policy $\pi_c$} 
    \label{sec:recommendation-policy-pi}
    Recommendation policy $\pi_c$ over $T$ rounds is given as such:
    \begin{enumerate}[label=\arabic*),itemsep=1mm]
        \item Run \Cref{alg:sampling-control-treatment} with exploration probability $\rho$ set to incentivize at least $p_{c_1}>0$ fraction of agents of the population who initially prefer control to comply in \Cref{alg:sampling-control-treatment} and $\ell$ to make at least $p_{c_2}>0$ fraction of agents comply in \Cref{alg:racing-two-types-mixed-preferences} (see \Cref{lemma:racing-compliance-sampling-ell-control-treatment}).
        \item Initialize \Cref{alg:racing-two-types-mixed-preferences} with samples from \Cref{alg:sampling-control-treatment}. At least $p_{c_2}$ fraction of agents comply in  \Cref{alg:racing-two-types-mixed-preferences}.
    \end{enumerate}
    We first provide conditions on $\ell$ to define policy $\pi_c$.
    
    \begin{restatable}[Lower bound on $\ell$ for Type $u$ Compliance in \Cref{alg:racing-two-types-mixed-preferences}]{lemma}{racingcompliancesamplingellcontroltreatment}\label{lemma:racing-compliance-sampling-ell-control-treatment}
        Recall that $S_\ell$ denotes the samples collected from the second stage of \Cref{alg:sampling-control-treatment}. Let $S_\ell$ be the input samples $S_0$ in \Cref{alg:racing-two-types-mixed-preferences}. Assume that $p_{c_1}$ proportion of agents in the population are compliant with recommendations of \Cref{alg:sampling-control-treatment} and length $\ell$ satisfies:
        \begin{small}
        \begin{equation}
            \label{eq:racing-compliance-proof-ell-condition}
            \ell \geq 
            \begin{cases}
            \left(\frac{\kappa_1}{\tau\Prob_{\cP^{(u)}}[\theta>\tau]}+\kappa_2\right)^2 & \text{ if } \underset{\cP^{(u)}}{\E}[\theta]<0\\
            \left(\frac{\kappa_1}{\tau\Prob_{\cP^{(u)}}[\theta<-\tau]}+\kappa_2\right)^2 & \text{ if } \underset{\cP^{(u)}}{\E}[\theta]\geq0
            \end{cases}
        \end{equation} 
        \end{small}
        for some $\tau\in(0,1)$ and where $\kappa_1:=\frac{8\sigma_g\sqrt{2\log(5/\delta)}}{p_{c_1}\rho(1-\rho)}$ and $\kappa_2:=(3-\rho)\sqrt{\frac{\rho\log(5/\delta)}{2(1-\rho)}}$ for any $\delta\in(0,1)$. Then any agent of type $u$ will comply with recommendations of \Cref{alg:racing-two-types-mixed-preferences}.
    \end{restatable}
    
    \begin{proofsketch} See \Cref{sec:racing-compliance-sampling-ell-control-treatment-proof} for the full proof. The proof follows by substituting the value of $\ell$ into the approximation bound \Cref{thm:sampling-estimation-bound} and simplifying. The compliance condition follows from \Cref{lemma:bic-racing-control-treatment-0}. 
    \end{proofsketch}
    
    Policy $\pi_c$ shifts from \Cref{alg:sampling-control-treatment} to \Cref{alg:racing-two-types-mixed-preferences} as soon as the condition on $\ell$ above is satisfied. This is because 1) the treatment effect estimate $\hat{\theta}$ get more accurate quicker and 2) less regret is accumulated in \Cref{alg:racing-two-types-mixed-preferences} than \Cref{alg:sampling-control-treatment}.


\subsection{Regret Analysis}\label{sec:regret-control-treatment}
The goal of recommendation policy $\pi_c$ is to maximize the cumulative reward of all agents. We measure the policy's performance through \textit{regret}. We are interested in minimizing regret, which is specific to the treatment effect $\theta$. Since agents' priors are not exactly known to the social planner, this pseudo-regret is correct for any realization of these priors and treatment effect $\theta$.
\begin{definition}
\label{def:pseudo-regret}[Pseudo-regret] The pseudo-regret of a recommendation policy is given as such:
	\begin{equation}
	R_\theta(T) = T\max(\theta,0) - \sum_{t=1}^T \theta x_t
	\end{equation}
\end{definition}

We present regret guarantees for recommendation policy $\pi_c$. First, policy $\pi_c$ achieves sub-linear pseudo-regret. 
\begin{restatable}[Pseudo-regret]{lemma}{expostregret}\label{lemma:expost-regret}
The pseudo-regret accumulated from policy $\pi_c$ is bounded for any $\theta\in[-1,1]$ as follows, with probability at least $1-\delta$ for any $\delta\in(0,1)$:
\begin{equation}
R_{\theta}(T) \leq L_1 + O(\sqrt{T\log(T/\delta)})
\end{equation}
for sufficiently large time horizon $T$, where the length of~\Cref{alg:sampling-control-treatment} is $L_1=\ell + 2\max\left(\frac{\ell_0}{p_0}, \frac{\ell_1}{p_1}\right)$.
\end{restatable}

\begin{proofsketch} See \Cref{sec:expost-proof} for the full proof. The proof follows by observing that \Cref{alg:racing-two-types-mixed-preferences} must end after some $\log(T)$ phases. We can bound the regret of the policy $\pi_c$ by at most that of \Cref{alg:racing-two-types-mixed-preferences} plus $\theta$ per each round of \Cref{alg:sampling-control-treatment}, or alternatively, we can upper bound it by $\theta$ per each round of the policy $\pi_c$. 
\end{proofsketch}

Policy $\pi_c$ also achieves sub-linear  regret, where the expectation is over the randomness in the priors of the agents. \Cref{lemma:expected-regret} provides a basic performance guarantee of our recommendation policy. 
\begin{restatable}[Regret]{lemma}{expectedregret}\label{lemma:expected-regret}
Policy $\pi_c$ achieves regret as follows:
\begin{equation}
\E[R(T)] = O(\sqrt{T\log(T)})
\end{equation}
for sufficiently large time horizon $T$.
\end{restatable}

\begin{proofsketch} See \Cref{sec:expected-regret-proof} for the full proof. The proof follows by observing that we can set the parameters in \Cref{alg:sampling-control-treatment} and \Cref{alg:racing-two-types-mixed-preferences} in terms of the time horizon $T$ while maintaining compliance throughout policy $\pi_c$.
\end{proofsketch}


These results are comparable to the pseudo-regret of the classic multi-armed bandit problem, with some added constants factors for the compliance constraints \cite{active-arms-elimination-2006}. The pseudo-regret of our policy $\pi_c$ is asymptotically equivalent to an extension of the detail-free recommendation algorithm of \cite{mansour2015bic}, which incentivizes full compliance for all types. However, our policy can finish in a more timely manner and has smaller prior-dependent constants in the asymptotic bound.

In \Cref{sec:many-arms}, we provide an extension of our model and policy $\pi_c$ to arbitrarily many treatments with unknown effects. We also provide similar regret guarantees.

\section{Many Treatments with Unknown Effects}
\label{sec:many-arms}
In this section, we introduce a setting which extends the previous binary treatment setting by considering $k$ treatments (and no control). We now consider a treatment effect vector $\theta\in\RR^k$; and $x,z\in\{0,1\}^k$ are one-hot encodings of the treatment choice and recommendation, respectively. We assume that $\E[g^{(u_i)}]=0$.\footnote{Without this assumption, we run into identifiability issues: we cannot reconstruct the individual treatment effects $\theta^1,\dots,\theta^k$ without fixing some mean $\E[g^{(u)}]$. Yet, for purposes of regret minimization, assuming $\E[g^{(u)}]=0$ does not change our results.} All other terms are defined similar to those in \Cref{sec:model}. Here, the reward $y_i\in\RR$ and action choice $x_i\in\{0,1\}^k$ at round $i$ are given as such:\footnotemark
\begin{equation}
    \begin{cases}
        y_i = \langle \theta, x_i \rangle + g^{(u_i)}\\
        x_i = \underset{1\leq\,j\leq\,k}{\argmax}\big(\E_{\cP^{(u_i)},\pi_i}[\theta^j|z_i,i]\big)
    \end{cases}
\end{equation}

\footnotetext{We bastardize notation by writing $x_i=j$ instead of $x_i=\e_j$ (the $k$-dimensional unit vector along the $j$th dimension).}

Given sample set $S=(z_i, x_i, y_i)_{i=1}^n$, we compute IV estimate $\hat{\theta}_S$ of $\theta$ as such:
\begin{equation}
    \label{eq:iv-estimate-general}
    \hat{\theta}_S=\left(\sum_{i=1}^nz_ix_i^\intercal\right)^{-1}\sum_{i=1}^nz_iy_i
\end{equation}

Next, we state finite sample approximation results which extend \Cref{thm:treatment-approximation-bound} to this general setting.

\begin{restatable}[Many Treatments Effect Approximation Bound]{theorem}{generaltreatmentapproximationbound}\label{thm:general-treatment-approximation-bound}
        Let $z_1,\dots,z_n\in\{0,1\}^k$ be a sequence of instruments. Suppose there is a sequence of $n$ agents such that each agent $i$ has private type $u_i$ drawn independently from $\cU$, selects $x_i$ under instrument $z_i$ and receives reward $y_i$. Let sample set $S=(x_i,y_i,z_i)_{i=1}^n$. The approximation bound $A(S,\delta)$ is given as such:\footnotemark
        \[ A(S,\delta) = \frac{\sigma_g\sqrt{2nk\log(k/\delta)}}{\sigma_{\min}\left(\sum_{i=1}^n z_ix_i^{\intercal}\right)},\]
        and the IV estimator given by \Cref{eq:iv-estimate-general} satisfies
        \[\norm{\hat{\theta}_S-\theta}_2\leq A(S,\delta)\]
        with probability at least $1-\delta$ for any $\delta\in(0,1)$. 
\footnotetext{The operator $\sigma_{\min}(\cdot)$ denotes the smallest singular value.}
\end{restatable}

\begin{proofsketch} See \Cref{sec:general-treatment-approximation-bound-proof} for the full proof. The bound follows by substituting our expressions for $y_t, x_t$ into the IV regression estimator, applying the Cauchy-Schwarz inequality to split the bound into two terms (one dependent on $\{(g^{(u_t)}_t, z_t)\}^{\abs{S}}_{t=1}$ and one dependent on $\{(x_t, z_t)\}^{\abs{S}}_{t=1}$). We bound the second term with high probability.
\end{proofsketch}

Next, we extend recommendation policy $\pi_c$ to $k$ treatments (see \Cref{def:policy-pi-many-arms} in \Cref{sec:many-arms-appendix} for details).\footnote{\Cref{alg:sampling-many-arms} extends \Cref{alg:sampling-control-treatment} and \Cref{alg:racing-many-arms} extends \Cref{alg:racing-two-types-mixed-preferences}.}

\begin{assumption}[Knowledge Assumption for General Policy $\pi_c$]\label{ass:know-general}
Within \Cref{sec:many-arms}, the following are common knowledge among agents and planner:
\vspace{-.2cm}
    \begin{enumerate}[itemsep=0mm]
        \item All agents share a preference ordering over all $k$ treatments, i.e. for any type $u$, the prior expected reward $\E_{\cP^{(u)}}[\theta^1] > \E_{\cP^{(u)}}[\theta^2] > \cdots > \E_{\cP^{(u)}}[\theta^k]$.\footnote{This ordering assumption is shared by \citet{mansour2015bic}.}
        \item Prior-dependent constants $\Prob_{\cP^{(u)}}[\xi^{(u)}]$ and $\Prob_{\cP^{(u)}}[G^v>\tau]$ for some $\tau\in(0,1)$ (see \Cref{sec:general-extension-appendix}).
    \end{enumerate}
\end{assumption}

\subsection{Extensions of \Cref{alg:sampling-control-treatment,alg:racing-two-types-mixed-preferences} and Recommendation Policy $\pi_c$ to $k$ Treatments}
\label{sec:general-extension}
We assume that every agent ---regardless of type--- shares the same prior ordering of the treatments, such that all agents prior expected value for treatment 1 is greater than their prior expected value for treatment 2 and so on. First, \Cref{alg:sampling-many-arms} is a generalization  of \Cref{alg:sampling-control-treatment} which serves the same purpose: to overcome complete non-compliance and incentivize some agents to comply eventually. The incentivization mechanism works the same as in \Cref{alg:sampling-control-treatment}, where we begin by allowing all agents to choose their preferred treatment ---treatment 1--- for the first $\ell$ rounds. Based on the $\ell$ samples collected from the first stage, we then define a number of events $\xi^{(u)}_j$ ---which are similar to event $\xi$ from \Cref{alg:sampling-control-treatment}--- that each treatment $j\geq2$ has the largest expected reward of any treatment and treatment 1 has the smallest, according to the prior of type $u$:
\begin{equation}
    \xi^{(u)}_i := \left( \bar{y}_{\ell}^1 + C \leq \min_{1<j<i} \bar{y}_{\ell}^j - C \ \text{ and } \ \max_{1<j<i} \bar{y}_{\ell}^j + C \leq \mu^{(u)}_i \right),
\end{equation}
where $C = \sigma_g\sqrt{\frac{2\log(3/\delta)}{\ell}} + \frac{1}{4}$ for any $\delta\in(0,1)$ and where $\bar{y}_{\ell}^1$ denotes the mean reward for treatment 1 over the $\ell$ samples of the first stage of \Cref{alg:sampling-many-arms}. Thus, if we set the exploration probability $\rho$ small enough, then some subset of agents will comply with all recommendations in the second stage of \Cref{alg:sampling-many-arms}.

\begin{algorithm}[ht!]
        \caption{Overcoming complete non-compliance for $k$ treatments}
        \label{alg:sampling-many-arms}
    \begin{algorithmic}
    \STATE {\bfseries Input:} exploration probability $\rho\in(0,1)$, minimum number of samples of any treatment $\ell\in\mathbb{N}$ (assume w.l.o.g. $(\ell/\rho)\in\mathbb{N}$), failure probability $\delta\in(0,1)$, compliant type $u$
   \STATE \textbf{1st stage:} The first $\ell$ agents are given no recommendation (they choose treatment 1)
   \FOR{each treatment $i > 1$ in increasing lexicographic order}
    \IF{$\xi^{(u)}_i$ holds, based on the $\ell$ samples from the first phase and any samples of treatment $2\leq j< i$ collected thus far} 
        \STATE $a^*_i = i$ 
    \ELSE
        \STATE $a^*_i=1$
    \ENDIF
    \STATE From the next $\ell/\rho$ agents, pick $\ell$ agents uniformly at random to be in the explore set $E$\footnotemark
    \FOR{the next $\ell$ rounds}
       \IF{agent $t$ is in explore set $E$} 
        \STATE  $z_t = 1$
       \ELSE
        \STATE $z_t = a^*$
       \ENDIF
       \ENDFOR
    \ENDFOR
    \end{algorithmic}
\end{algorithm}

\footnotetext{We set the length of each prhase $i$ of the second stage to be $\ell/\rho$ so that we get $\ell$ samples of each treatment $i$ and the exploration probability is $\rho$.}

Second, \Cref{alg:racing-many-arms} is a generalization  of \Cref{alg:racing-two-types-mixed-preferences}, which is required to start with at least partial compliance and more rapidly and incentivizes more agents to comply eventually. The incentivization mechanism works the same as in \Cref{alg:sampling-control-treatment}, where we begin by allowing all agents to choose their preferred treatment ---treatment 1--- for the first $\ell$ rounds. Based on the $\ell$ samples collected from the first stage, we then define a number of events ---which are similar to event $\xi$ from \Cref{alg:sampling-control-treatment}--- that each treatment $j\geq2$ has the largest expected reward of any treatment and treatment 1 has the smallest. Thus, if we set the exploration probability $\rho$ small enough, then some subset of agents will comply with all recommendations in the second stage of \Cref{alg:sampling-many-arms}.

\begin{algorithm}[ht!]
        \caption{Overcoming partial compliance for $k$ treatments}
        \label{alg:racing-many-arms}
    \begin{algorithmic}
    \STATE{\bfseries Input:} samples $S_0:=(x_i,z_i,y_i)_{i=1}^{|S_0|}$ which meet \Cref{thm:general-treatment-approximation-bound} conditions and produce IV estimate $\hat{\theta}_{S_0}$, time horizon $T$, number of recommendations of each action per phase $h$, failure probability $\delta\in(0,1)$
    \STATE Split the remaining rounds (up to $T$) into consecutive phases of $h$ rounds each, starting with $q=1$;\\
    \STATE Let $\hat{\theta}_0=\hat{\theta}_{S_0}$ and $A_0 = A(S_0,\delta)$
    \STATE Initialize set of active treatments: $B = \{ \text{all treatments} \}$.\\ 
    \WHILE{$|B| > 1$}
        \STATE Let $\hat{\theta}^*_{q-1} = \max_{i \in B} \hat{\theta}_{q-1}^i$ be the largest entry $i$ in $\hat{\theta}_{q-1}$
        \STATE Recompute $B = \left \{ \text{treatments } i: \hat{\theta}_{q-1}^* - \hat{\theta}_{q-1}^i \leq A_{q-1} \right \}$;
        \STATE The next $|B|$ agents are recommended each treatment $i \in B$ sequentially in lexicographic order;
        \STATE Let $S_q^{\text{BEST}}$ be the sample set with the smallest approximation bound so far, i.e.  $S_q^{\text{BEST}}=\argmin_{S_r, 0\leq r\leq q} A(S_r,\delta)$;\\
        \STATE Define $\hat{\theta}_q=\hat{\theta}_{S_q^{\text{BEST}}}$ and $A_q=A(S_q^{\text{BEST}},\delta)$;\\
        \STATE $q = q+1$
    \ENDWHILE
    \STATE For all remaining agents, recommend $a^*$ that remains in $B$.
\end{algorithmic}
\end{algorithm}

\begin{definition}[General recommendation policy $\pi_c$ for $k$ treatments]
\label{def:policy-pi-many-arms}
Recommendation policy $\pi_c$ over $T$ rounds is given as such:
    \begin{enumerate}[label=\arabic*)]
        \item Run \Cref{alg:sampling-many-arms} with exploration probability $\rho$ set to incentivize at least $p_{c_1}>1/k$ fraction of agents of the population to comply in \Cref{alg:sampling-many-arms}. Let $S_0$ be the sample set given from \Cref{alg:sampling-many-arms}. By \Cref{cor:approximation-bound-pairwise} and \Cref{thm:racing-compliance-many-arms}, we can use $S_0$ as the initial samples in \Cref{alg:racing-many-arms} to incentivize compliance for any arm $a$ if the approximation bound $A(S_0,\delta)$ given by $S_0$ is small enough (see \Cref{thm:racing-compliance-many-arms}). Thus, we run \Cref{alg:sampling-many-arms} long enough (i.e. we set $\ell$ large enough) so that the approximation bound is small enough and at least $p_{c_2}>1/k$ fraction of agents comply with recommendations of every treatment in \Cref{alg:racing-many-arms}.
        \item Initialize \Cref{alg:racing-many-arms} with samples $S_0$ from \Cref{alg:sampling-many-arms}. At least $p_{c_2}>1/k$ fraction of agents comply with recommendations from \Cref{alg:racing-many-arms} from the beginning and until time horizon $T$.
    \end{enumerate}
\end{definition}

Similar to the control-treatment setting, we provide the compliance lemmas and proof sketches for \Cref{alg:sampling-many-arms} and \Cref{alg:racing-two-types-mixed-preferences}. 

\begin{restatable}[\Cref{alg:sampling-many-arms} compliance]{lemma}{bicgeneralsampling} \label{lemma:bic-general-sampling} Let event $\xi^{(u)}$ be defined such that $\Prob[\xi^{(u)}] = \min_i \Prob[\xi^{(u)}_i]$. In \Cref{alg:sampling-many-arms}, any type $u$ agent who arrives in the last $\ell/\rho$ rounds of \Cref{alg:sampling-many-arms} is compliant with any recommendation if $\mu^{(u)}_i > 0$ for all $1<i\leq k$, and the exploration probability $\rho$ satisfies:
\begin{equation}
    \rho \leq 1 + \frac{8\big(\mu^{(u)}_j - \mu^{(u)}_i\big)}{\Prob_{\pi_c, \cP^{(u)}}[\xi^{(u)}]}
\end{equation}
\end{restatable}

\begin{proofsketch}
Let the recommendation policy $\pi$ here be \Cref{alg:sampling-many-arms}. 

{\bf Part I (Compliance with recommendation for treatment $i>1$):} We first argue that an agent $t$ of type $u$ who is recommended treatment $i$ will not switch to any other treatment $j$. For treatments $j > i$, there is no information about treatments $i$ or $j$ collected by the algorithm and by assumption, we have $\mu^{(u)}_i \geq \mu^{(u)}_j$. Hence, it suffices to consider when $j<i$. We want to show that
    \begin{equation}
        \label{eq:bic-sample-many-arms-types}
        \E_{\pi_t,\cP^{(u)}}[\theta^j - \theta^1 | z_t = \e_i] \Prob_{\pi_t,\cP^{(u)}}[z_t = \e_i] \geq 0.
    \end{equation}

{\bf Part II (Recommendation for treatment 1):} When agent $t$ is recommended treatment 1, they know that they are not in the explore group $E$. Therefore, they know that the event $\neg \xi^{(u)}_i$ occurred. Thus, in order to prove that \Cref{alg:sampling-many-arms} is BIC for an agent of type $u$, we need to show the following for any treatment $j > 1$:
\begin{equation}
    \label{eq:sampling-multi-arm-type-rec-1}
    \E_{\pi_t,\cP^{(u)}}[\theta^1 - \theta^j | z_t=\e_1]\Prob_{\pi_t,\cP^{(u)}}[z_t=\e_1] = \E_{\pi_t,\cP^{(u)}}[\theta^1 - \theta^j | \neg \xi^{(u)}_i]\Prob_{\pi_t,\cP^{(u)}}[\neg \xi^{(u)}_i]\geq 0
\end{equation}

We omit the remainder of this proof due to its similarity with the proof of \Cref{lemma:bic-sampling-control-treatment}
\end{proofsketch}

\begin{restatable}[\Cref{alg:racing-many-arms} Partial Compliance]{lemma}{bicracingmanytypessecond}
\label{lemma:bic-racing-many-arms}
    Recall that \Cref{alg:racing-many-arms} is initialized with input samples $S_0=(x_i,y_i,z_i)_{i=1}^{|S_0|}$. For any type $u$, if $S_0$ satisfies the following condition, then with probability at least $1-\delta$ all agents of type $u$ will comply with recommendations of \Cref{alg:racing-many-arms}:
        \[A(S_0,\delta) \leq \tau \Prob_{\cP^{(u)}}[\min_{a,b}\left(\left|\theta^a - \theta^b\right|\right) > \tau]/4\]
    for some $\tau\in(0,1)$, where $A(S_0,\delta)$ is the approximation bound for $S_0$ and any $\delta\in(0,1)$ (see \Cref{thm:general-treatment-approximation-bound}).
\end{restatable}

\begin{proofsketch}
Let the recommendation policy $\pi$ here be \Cref{alg:racing-many-arms}. We want to show that for any agent at time $t$ with a type $i<u$ in the racing stage and for any two treatments $a,b \in B:$
\begin{align*}
    \E_{\pi_t,\cP^{(u)}}[\theta^a - \theta^b | z_t=\e_a]\Prob_{\pi_t,\cP^{(u)}}[z_t=\e_a] \geq 0
\end{align*}

We omit the remainder of this proof due to its similarity with the proof of  \Cref{lemma:bic-racing-control-treatment-0}.

\end{proofsketch}

In order to incentivize agents of any type $u$ to comply with general extensions \Cref{alg:sampling-many-arms,alg:racing-many-arms}, we (again) set exploration probability $\rho$ and length $\ell$ to satisfy some compliance conditions relative to $\Prob_{\cP^{(u)}}[\xi^{(u)}]$ and $\tau\Prob_{\cP^{(u)}}[G^v>\tau]$, respectively (see \Cref{sec:general-extension-appendix}). We present the (expected) regret from the $k$ treatment extension of policy $\pi_c$ next.

\begin{restatable}[Regret of Policy $\pi_c$ for $k$ Treatments]{lemma}{regretmanyarms}\label{lemma:regret-many-arms}
An extension of policy $\pi_c$ achieves (expected) regret as follows:
\begin{equation}
    \E[R(T)] = O\left(k\sqrt{kT\log(kT)} \right)
\end{equation}
for sufficiently large time horizon $T$.
\end{restatable}

\begin{proofsketch} See \Cref{sec:general-extension-appendix} for the full proof. The proof follows the same structure as that of \Cref{lemma:expected-regret}.
\end{proofsketch}

Though our analysis covers a more general $k$ treatment setting than \citet{mansour2015bic} (capturing non-compliance and selection bias), our policy $\pi_c$ accumulates asymptotically comparable regret in terms of $T$. See \Cref{sec:many-arms-appendix} for all other results. Next, in \Cref{sec:experiment}, we implement \Cref{alg:sampling-control-treatment} experimentally.


\section{Numerical Experiments}
\label{sec:experiment}
In this section, we present experiments to evaluate \Cref{alg:sampling-control-treatment}. We mention previously in the paper that this approximation bound decreases slowly throughout \Cref{alg:sampling-control-treatment}, because the exploration probability $\rho$ in \Cref{alg:sampling-control-treatment} is small. Here, we are interested in (1) \textit{how small} the exploration probability $\rho$ is and (2) \textit{how slowly} the approximation bound on the absolute difference $|\theta-\hat{\theta}|$ decreases as \Cref{alg:sampling-control-treatment} progresses (where $\hat{\theta}$ is based on samples from \Cref{alg:sampling-control-treatment}). These are important to study, because this slow improvement in accuracy is the primary source of inefficiency (in terms of sample size) for policy $\pi_c$, which accumulates linear regret during \Cref{alg:sampling-control-treatment} (see \Cref{lemma:expost-regret}) for marginal improvements in estimation accuracy. This motivates the social planner to move to \Cref{alg:racing-two-types-mixed-preferences} ---where the estimation accuracy increases much quicker--- as soon as possible in policy $\pi_c$. Yet, there is also a tradeoff for moving to \Cref{alg:racing-two-types-mixed-preferences} too quickly: if \Cref{alg:sampling-control-treatment} is not run for long enough, then only a small portion of agents may comply in \Cref{alg:racing-two-types-mixed-preferences}. In order to better inform the choice of hyperparameters in policy $\pi_c$ (specifically, the compliance paramters $p_{c_1}$ and $p_{c_2}$), we empirically estimate these quantities experimentally. We defer experiments on \Cref{alg:racing-two-types-mixed-preferences} to the appendix.\footnote{The code is available \href{https://github.com/DanielNgo207/Incentivizing-Compliance-with-Algorithmic-Instruments}{here}.}

\paragraph{Experimental Description.}
We consider a setting with two types of agents: type 0 who are initially never-takers and type 1 who are initially always-takers. We let each agent's prior on the treatment effect be a truncated Gaussian distribution between $-1$ and $1$. The noisy baseline reward $g_t^{(u_t)}$ for each type $u$ of agents is drawn from a Gaussian distribution $\mathcal{N}(\mu_{g^{(u)}}, 1)$, with its mean $\mu_{g^{(u)}}$ also drawn from a Gaussian prior. We let each type of agent have equal proportion in the population, i.e. $p_0=p_1=0.5$. We are interested in finding the probability of event $\xi$ (as defined in \Cref{eq:xi}) and the exploration probability $\rho$ (as defined in \Cref{eq:rho}). Instead of deriving an explicit formula for $\Prob_{\cP_0}[\xi]$ to calculate the exploration probability $\rho$, we estimate it using Monte Carlo simulation by running the first stage of \Cref{alg:sampling-control-treatment} for $1000$ iterations and aggregating the results. After this, \Cref{alg:sampling-control-treatment} is run with the previously-found exploration probability $\rho$ over an increasing number of rounds. We repeatedly calculate the IV estimate of the treatment effect and compare it to a naive OLS estimate (that regresses the treatment onto the reward) over the same samples as a benchmark.
\begin{figure}[ht]
\centering
\includegraphics[width=.75\linewidth]{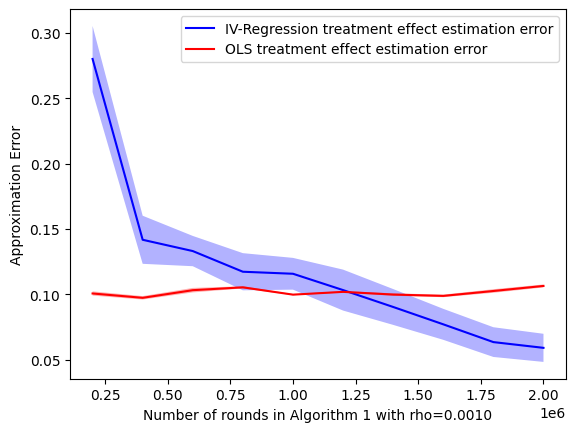}
\caption{Approximation bound using IV regression and OLS during
\Cref{alg:sampling-control-treatment} with $\rho = 0.001$. Results are averaged over 5 runs; light blue error bars represent one standard error. The OLS estimate converges to a value with approximation error around 0.1; whereas the approximation error of our IV estimate steadily decreases over time.}
\label{fig:plot}
\vspace{-.5cm}
\end{figure}

\paragraph{Results.}
In \Cref{fig:plot}, we compare the approximation bound on $|\theta-\hat{\theta}|$ between IV estimate $\hat{\theta}$ versus via a naive estimate for a specific, chosen $\rho=0.001$. In our experiments, the exploration probability $\rho$ generally lies within $[0.001, 0.008]$. In \Cref{fig:plot}, we let hidden treatment effect $\theta=0.5$, type 0 and type 1 agents' priors on the treatment effect be $\cN(-0.5, 1)$ and $\cN(0.9, 1)$ ---each truncated onto $[-1,1]$,--- respectively. We also let the mean baseline reward for type 0 and type 1 agents be $\mu_{g^{(0)}} \sim \mathcal{N}(0, 1)$ and $\mu_{g^{(1)}} \sim \mathcal{N}(0.1, 1)$, respectively. These priors allow us to set the exploration probability $\rho = 0.001$ for \Cref{fig:plot}, where IV regression consistently outperforms OLS for any reasonably long run of \Cref{alg:sampling-control-treatment}. 

To our knowledge, these experiments are the first empirical evaluation of an Incentivizing Exploration algorithm. \Cref{fig:plot} shows the effect of small exploration probability $\rho = 0.001$: we need to run \Cref{alg:sampling-control-treatment} for a while to produce a decently accurate causal effect estimate.\footnote{We suspect this weakness is likely endemic to previous works in Incentivizing Exploration, as well.}



\section{Conclusion}
In this paper, we present a model for how (non)-compliance changes over time based on beliefs and new information. We observe that recommendations which incentivize (at least partial) compliance can be treated as instrumental variables (IVs), enabling consistent treatment effect estimation via IV regression, even in the presence of non-compliance and confounding. Finally, we provide recommendations mechanisms which provably incentivize compliance and achieve sublinear regret.

\xhdr{Ethical Considerations around Incentivizing Exploration.} 

In this paper, we provide recommendation algorithms for incentivizing exploration (IE), each which consist of two parts with their own ethical considerations: 1) recommending for exploration and 2) selectively disclosing information. First, in certain contexts, recommending for exploration of a treatment with unknown effects could place undue risk or harm onto an individual in order to benefit society or maximizing well-being overall.\footnote{In moral terms, this may cause a fissure between a utilitarian and a Kantian or purveyor of individual rights.} Because of this, in certain medical contexts, incentivizing exploration may violate the principle of nonmaleficence. Second, incentivizing exploration requires selective information disclosure. While selective disclosure may be required in studies where information is protected or proprietary, it may be (morally) wrong in other settings. Restricting full information about the history of a treatment for the purpose of changing an individual's behavior may be manipulative or deceptive and may restrict their autonomy. We do not provide ethical considerations to discourage incentivizing exploration outright: there may be settings where the benefits outweigh the harms. We leave it to study designers and policymakers to weigh these in specific settings.

\xhdr{Future Work.}
Here, we focused on a setting where the causal model is linear and there is no treatment modification by the private type (so all agents share the same treatment effect $\theta$). Future work may extend our results to non-linear settings and settings with treatment effect heterogeneity. We may also relax the (somewhat unrealistic) assumptions 1) that the social planner knows key prior-dependent constants about all agents and 2) that agents fully know their prior, the recommendation mechanism, and can exactly update their posterior over treatment effects. Finally, our empirical results invite further work to improve the practicality of incentivizing exploration mechanisms to allow for more frequent exploration and lessen the number of samples needed.

\section*{Acknowledgement} 
We thank the members of the Social AI Group for their comments on drafts of this work. Zhiwei Steven Wu was supported by the NSF FAI Award \#1939606, a Google Faculty Research Award, a J.P. Morgan Faculty Award, a Facebook Research Award, and a Mozilla Research Grant. We also thank Nicole Immorlica and Akshay Krishnamurthy for their discussions. 

\bibliographystyle{plainnat}
\bibliography{./refs}

\appendix
\onecolumn

\section{Theorems and Lemmas}
\dndelete{
\begin{theorem} (Gaussian Tail Bound)
    \label{thm:gaussian}
Let $X_1, \dots, X_n$ be independently and identically distributed Gaussian variables such that $X_i \sim \cN(\mu, \sigma^2)$ for $i\in[1,n]$, mean $\mu$, and variance $\sigma$. Let $\displaystyle \bar{X} = \frac{1}{n} \sum_{i=1}^n X_i \sim \cN(\mu, \sigma^2/n)$. Then,
\begin{equation}
    \Prob\left[ \left| \bar{X} - \mu \right| \geq \epsilon \right] \leq 2\exp\left\{-\frac{n\epsilon^2}{2\sigma^2}\right\}.
\end{equation}
\end{theorem}

\begin{corollary} (High-probability bound on the sum of $n$ i.i.d. random Gaussian variables)
    \label{thm:high-prob-gaussian}
    Let $\sum_{i=1}^n X_i$ be the sum of Gaussian variables such that $X_i \sim \cN(\mu, \sigma)$. Note that $\sum_{i=1}^n X_i = n\bar{X}$. Then, with some confidence level $\delta > 0,$ with probability at least $1-\delta$,
    \begin{equation}
        \left| \sum_{i=1}^n X_i - \mu \right| \leq \sigma \sqrt{2n\log(2/\delta)}.
    \end{equation}
\end{corollary}

\begin{theorem} (Chi-Squared Tail Bound)
    \label{thm:chi-squared}
Let $X_1, \dots, X_n$ be i.i.d. Gaussian variables such that $X_i \sim \cN(0, \sigma^2)$ for $i\in[1,n]$, mean 0, and variance $\sigma$. Then, $\displaystyle Y := \sum_{i=1}^n X_i^2$ is a Chi-squared variable with $n$ degrees of freedom. Then, we can upper bound $Y$ as such:
\begin{equation}
    \Prob\left[ Y/n \geq \sigma^4(1+\epsilon)^2 \right] \leq \exp\left\{-n \epsilon^2 / 2\right\}.
\end{equation}
\end{theorem}

\begin{corollary} (High-probability bound on the sum of $n$ i.i.d. random Chi-squared variables)
    \label{thm:high-prob-chi-squared}
    Let $\displaystyle Y := \sum_{i=1}^n X_i^2$ be a Chi-squared variable with $n$ degrees of freedom. Then, with some confidence level $\delta > 0,$ with probability at least $1-\delta$,
    \begin{equation}
        \sqrt{Y} < \sigma^2 \left(\sqrt{n} + 2\sqrt{\log(1/\delta)} \right)
    \end{equation}
\end{corollary}}

\begin{theorem} (Chernoff Bound for unbounded sub-Gaussian random variables)
\label{thm:unbounded-chernoff}
Let $X_1, \dots, X_n$ be independent sub-Gaussian random variables with parameter $\sigma$. Let $\overline{X}=\frac{1}{n}\sum_{i=1}^n X_i$. For all $\epsilon>0,$

$$\Prob\left[\left|\overline{X}\right| \geq \epsilon \right] \leq \exp\left\{ \frac{-n\epsilon^2}{2\sigma^2} \right\}.$$
\end{theorem}

\begin{corollary} (High probability bound on the sum of unbounded sub-Gaussian random variables)
\label{thm:high-prob-unbounded-chernoff}
    For any $\delta \in (0,1),$ with probability at least $1-\delta,$
   $$ \left| \overline{X} \right| < \sigma \sqrt{\frac{2 \log(1/\delta)}{n}} $$
\end{corollary}

\begin{theorem} (Chernoff/Hoeffding's inequality) 
\label{thm:bounded-chernoff}
Let $X_1, \ldots, X_n$ be independent and bounded random variables such that $a \leq X_i \leq b$ for all i. Then
\begin{equation*}
    \Prob\left[\frac{1}{n} \sum_{i=1}^n X_i - \E[X_i] \geq \epsilon \right] \leq \exp\left(\frac{-2n\epsilon^2}{(b-a)^2}\right)
\end{equation*}
\end{theorem}

\begin{corollary}(High probability upper bound on the sum of bounded random variables)
\label{thm:high-prob-bounded-chernoff}
For any $\delta \in (0, 1)$, with probability at least $1 - \delta$,
\begin{equation*}
    \E[X] -  \frac{1}{n} \sum_{i=1}^n X_i \leq (b-a)\sqrt{\frac{\log(1/\delta)}{2n}},
\end{equation*}
where $X_i \in [a,b]$ for all $i$ from $1$ to $n$.
\end{corollary}

\dndelete{\begin{theorem} (Multiplicative Chernoff inequality)
Let $X_1, \dots, X_n$ be independent and identically distributed random variables such that $X_i \in [0,1]$. Then, 
\begin{equation}
    \Prob\left[\frac{1}{n}\sum_{i=1}^n X_i - \E[X_i] > \tau \E[X_i]\right] \leq \exp(-\E[X_i] n \tau^2/3)
\end{equation}
\end{theorem}

\begin{corollary} (High probability upper bound on the sum of bounded random variables) 
\label{cor:multiplicative-chernoff}
For any $\delta \in (0,1)$, with probability at least $1 - \delta$,
\begin{equation}
    \abs{\frac{1}{n} \sum_{i=1}^n X_i - \E[X_i]} \leq \sqrt{\frac{3\ln(1/\delta)}{\E[X_i] n}} \E[X_i]
\end{equation}
\end{corollary}

\begin{lemma} (Spectral norm)
\label{thm:spectral}
    For any $m\times n$-dimensional matrix $A \in \mathbb{R}^{m\times n}$, the \textbf{spectral norm} of $A$ is defined as the square root of the maximum eigenvalue of the covariance matrix $A^{\intercal}A$, i.e.
    $$\norm{A}_2 := \sqrt{\lambda_{\max}(A^{\intercal}A)}.$$
    If $A$ is symmetric and positive semi-definite, then
    $$\norm{A}_2 = \lambda_{\max}(A).$$
\end{lemma}
\begin{corollary}(Spectral norm of inverse)
\label{thm:inversespectral}
    For any invertible, symmetric, positive semi-definite matrix $A$, the spectral norm of the inverse $A^{-1}$ equals the the reciprocal of the minimum eigenvalue of $A$, i.e.
    $$\norm{A^{-1}}_2 = \lambda_{\max}(A^{-1}) = \frac{1}{\lambda_{\min}(A))}.$$
\end{corollary}}

\begin{lemma} (Cauchy-Schwarz Inequality)
\label{thm:cauchy-schwarz}
    For any $n$-dimensional vectors $u,v \in \mathbb{R}^n,$ the $L^2-$norm of the inner product of $u$ and $v$ is less than or equal to the $L^2-$norm of $u$ times the $L^2-$norm of $v$, i.e.
    $$\norm{\langle u,v \rangle}_2 \leq \norm{u}_2\cdot \norm{v}_2.$$
    Alternatively, for any $m\times n$-dimensional matrices $A \in \mathbb{R}^{m\times n}$ and $n$-dimensional vector $v\in \mathbb{R}^n$, the $L^2-$norm of the dot product of $A$ and $v$ is less than or equal to the spectral norm of $A$ times the $L^2-$norm of $v$, i.e.
    $$\norm{Av}_2 \leq \norm{A}_2\cdot\norm{v}_2.$$
\end{lemma}

\begin{theorem} (Matrix Chernoff)
\label{thm:matrix-chernoff}
    Consider a finite sequence $X_k$ of independent, random, self-adjoint matrices with common dimension d. Assume that:
    \begin{equation*}
        0 \leq \lambda_{\min}(X_k) \quad\text{and } \quad\lambda_{\max}(X_k) \leq \omega \quad\text{for each index k.} 
    \end{equation*}
    Introduce the random matrix $Y = \sum_k X_k$. Define the minimum eigenvalue $\mu_\text{min}$ and maximum eigenvalue $\mu_\text{max}$ of the expectation $\E[Y]$.
    \begin{align*}
        \mu_\text{min} &= \lambda_{\min}\left\{\E[Y]\right\} = \lambda_{\min}\left\{\sum_k\E[X_k] \right\}, \quad\text{and}\\
        \mu_{\max} &= \lambda_{\max}\left\{ \E[Y] \right\} = \lambda_{\max}\left\{\sum_k\E[X_k] \right\}
    \end{align*}
    Then, for $\theta > 0,$
    \begin{align*}
        \E[\lambda_{\min} (Y)] &\geq \frac{1 - e^{-\theta}}{\theta}\mu_{\min} - \frac{1}{\theta}L\log d, \quad\text{and}\\
        \E[\lambda_{\max} (Y)] &\leq \frac{ e^\theta - 1}{\theta}\mu_{\max} + \frac{1}{\theta}L\log d
    \end{align*}
    Furthermore, 
    \begin{align*}
        \Prob[\lambda_{\min}(Y) \leq (1- \epsilon)\mu_{\min}] &\leq d\left[\frac{e^{-\epsilon}}{(1- \epsilon)^{1-\epsilon}} \right]^{\mu_{\min}/\omega} \quad\text{for } \epsilon \in [0, 1)\\
        \Prob[\lambda_{\max}(Y) \leq (1 + \epsilon)\mu_{\max}] &\leq d\left[\frac{e^{\epsilon}}{(1+ \epsilon)^{1+\epsilon}} \right]^{\mu_{\max}/\omega} \quad\text{for } \epsilon \geq 0
    \end{align*}
    
\end{theorem}
\dndelete{
\begin{theorem} (High probability bound on Spectral Norm)
\label{thm:high-prob-spectral}
    Assume that $X_i \in \RR^{d_1 x d_2}$ are sampled i.i.d. Let $d = \min\{d_1, d_2\}$. Suppose $\norm{X}_2 \leq M$ almost surely. Then with probability greater than $1 - \delta$,
    \begin{equation*}
        \norm{\frac{1}{n} \sum_{i = 1}^n X_i - \E[X]}_2 \leq \frac{6M}{\sqrt{n}} \left(\sqrt{\log d} + \sqrt{\log \frac{1}{\delta}} \right)
    \end{equation*}
\end{theorem}

\begin{theorem} (Weyl's inequality):
\label{thm:weyl}
Let $X, Y \in \mathbb{R}^{n\times n}$ be symmetric matrices where X has eigenvalues
\begin{equation*}
    \lambda_1(X) \geq \ldots \lambda_n(X)
\end{equation*}
and Y has eigenvalues:
\begin{equation*}
    \lambda_1(Y) \geq \ldots \lambda_n(Y)
\end{equation*}
Then for all $i = 1, \ldots, n$: 
\begin{align*}
    \lambda_i(X) + \lambda_n(Y) \leq \lambda_i(X+Y) \leq \lambda_i(X) + \lambda_1(Y)
\end{align*}
\end{theorem}
}
\begin{theorem} (Union bound):
\label{thm:union}
For a countable set of events $A_1, A_2, \dots$, we have 
\begin{align*}
    \Prob\left[\bigcup_i A_i\right] \leq \sum_i \Prob(A_i)
\end{align*}

\end{theorem}
\dndelete{
\begin{theorem} (Matrix Concentration bound) 
\label{thm:matrix-concentration}
Assume that $X_i \in \mathbbm{R}^{d_1 \times d_2}$ are sampled i.i.d. Let $d = \min{d_1, d_2}$. Suppose $\norm{X}_2 \leq M$ almost surely. Then with probability greater than $1 - \delta$,
\begin{align*}
    \norm{\frac{1}{n} \sum_{i=1}^n X_i - \E[X]}_2 \leq \frac{6M}{\sqrt{n}} \left( \sqrt{\log(d)} + \sqrt{\log\left(\frac{1}{\delta}\right)} \right)
\end{align*}
\end{theorem}

\begin{theorem} (Holder's inequality)
\label{thm:holder}
Let 
\begin{equation*}
    \frac{1}{p} + \frac{1}{q} = 1
\end{equation*}
with $p, q > 1$. Then, we have
\begin{equation*}
    \norm{fg}_1 \leq \norm{f}_p\norm{g}_q
\end{equation*}

\begin{lemma} (Mahalanobis norm)
\label{thm:mahalanobis-norm}
Let $M\in \mathbb{R}^{n\times n}$ be a real matrix and $x \in \mathbb{R}^n$ be a vector. The Mahalanobis norm of $x$ is defined as such:
    \begin{equation}
        \norm{x}_M := \sqrt{x^{\intercal} M x}.
    \end{equation}
\end{lemma}

Note that if $M=\bI,$ the identity matrix, then the Mahalanobis norm is equivalent to the Euclidean norm, i.e. $\norm{x}_{\bI} = \norm{x}_2$.

\begin{theorem} (Mahalanobis bound) (Theorem 20.4 from \cite{lattimore2018bandits})

\label{thm:mahalanobis-norm-bound}
If 
\begin{enumerate}
    \item Noise term $a_t$ is conditionally 1-subgaussian:\\
    for all $\alpha \in \mathbb{R}$ and $t \geq 1$, $\E[\exp(\alpha (\nu_t + \gamma_t))| \mathcal{F}_{t-1}] \leq \exp \left( \frac{\alpha^2}{2} \right)$\\
    where $\mathcal{F}_{t-1}$ is such that $z_1, a_1, \dots, z_{t-1}, a_{t-1}, z_t$ are $\mathcal{F}_{t-1}$-measurable.
    \item The noise $a_t$ is mean-zero.
\end{enumerate}

Let $V_t(\lambda) = \lambda \bI + \sum_{t=1}^n \tilde{z}_t \tilde{z}_t^\intercal$ and $\delta \in (0,1)$. Then, with probability at least $1 - \delta$, it holds that for all $\lambda > 0$, 
\begin{align}
   \Prob\left[\text{exists  } t \in \mathbb{N}: \norm{\sum_{t=1}^n \tilde{z}_ta_t }^2_{V_t(\lambda)^{-1}} \geq 2\log\left(\frac{1}{\delta}\right) + \log\left(\frac{\det V_t(\lambda)}{\lambda^d}\right) \right] \leq \delta
\end{align}
\end{theorem}

\begin{lemma} (Elliptical Potential Lemma) (Lemma 19.4 from \cite{lattimore2018bandits})
\label{lemma:epl}
Let $V_0 \in \mathbb{R}^{d \times d}$ be positive definite and $a_1, \dots, a_n \in \mathbb{R}$ be a sequence of vectors with $\norm{a_t}_2 \leq L \leq \infty$ for all $t \in [n]$, $V_t = V_0 + \sum_{s\leq t} a_s a_s^\intercal$. Then
\begin{align}
    \frac{\det V_t(\lambda)}{\lambda^d} \leq \left( trace \left(\frac{V_t(\lambda)}{\lambda d} \right) \right)^d \leq \left(1 + \frac{nL^2}{\lambda d} \right)^d
\end{align}
\end{lemma}
\newpage
\end{theorem}}
\section{IV Estimator Proof for Control-Treatment Setting}
\label{sec:iv-estimator}
Recall that our reward model can be stated as the following equation:
\begin{equation}
    y_i = \theta x_i + g^{(u_i)}_i
\end{equation}

To analyze the Wald estimator, we introduce two conditional probabilities that an agent chooses the treatment given a recommendation $\hat{\gamma}_0$ and $\hat{\gamma}_1$, given as proportions over a set of $n$ samples $(x_i,z_i)_{i=1}^n$ and formally defined as
\begin{align*} \hat{\gamma}_0 = \hspace{-.15cm}\underset{(x_i,z_i)_{i=1}^n}{\hat{\Prob}}\hspace{-.3cm}[x_i = 1| z_i = 0] = \frac{\sum_{i=1}^n x_i(1-z_i)}{\sum_{i=1}^n (1-z_i)^2} \quad \text{and} \quad
    \hat{\gamma}_1 = \hspace{-.15cm}\underset{(x_i,z_i)_{i=1}^n}{\hat{\Prob}}\hspace{-.3cm}[x_i = 1| z_i = 1]  = \frac{\sum_{i=1}^n x_iz_i}{\sum_{i=1}^n z_i^2}
\end{align*}
Then, we can write the action choice $x_i$ as such:
\begin{align*}
    x_i 
    &= \hat{\gamma}_1 z_i + \hat{\gamma}_0 (1 - z_i) + \eta_i\\
    &= \hat{\gamma} z_i + \hat{\gamma}_0 + \eta_i
\end{align*}
where $\eta_i = x_i - \hat \gamma_1 z_i - \hat \gamma_0 (1 - z_i)$ and $\hat{\gamma} = \hat{\gamma}_1 - \hat{\gamma}_0$ is the in-sample \textit{compliance coefficient}. Now, we can rewrite the reward $y_i$ as 
\begin{align*}
    y_i &= \theta\left(\hat{\gamma} z_i + \hat{\gamma}_0 + \eta_i\right) + g^{(u_i)}_i\\
    &= \underbrace{\theta\,\hat{\gamma}}_{\beta} z_i + \theta \hat{\gamma}_0 + \theta \eta_i + g^{(u_i)}_i
\end{align*}
Let operator $\bar{\cdot}$ denote the sample mean, e.g. $\bar{y}:=\frac{1}{n}\sum_{i=1}^n y_i$ and $\bar{g}:=\frac{1}{n}\sum_{i=1}^n g^{(u_i)}_i$. $\bar{\eta} = \frac{1}{n}\sum_{i=1}^n \eta_i = 0$, by definition.

Then,
\begin{align*}
    \bar{y} &= \beta\bar{z} + \theta\hat{\gamma}_0 + \theta\bar{\eta} + \bar{g} + \bar{\epsilon}
\end{align*}
Thus, the centered reward and treatment choice at round $i$ are given as:
\begin{equation}
    \begin{cases}
    \label{eq:centered-reward-treatment-choice}
    y_i - \bar{y} = \theta(x_i-\bar{x}) + g^{(u_i)}_i-\bar{g}\\
    y_i - \bar{y} = \beta(z_i-\bar{z}) + \theta(\eta_i-\bar{\eta}) + g^{(u_i)}_i-\bar{g}\\
    x_i - \bar{x}_i = \hat{\gamma}(z_i-\bar{z}) + \eta_i-\bar{\eta}
    \end{cases}
\end{equation}
This formulation of the centered reward $y_i-\bar{y}$ allows us to express and bound the error between the treatment effect $\theta$ and its instrumental variable estimate $\hat{\theta}_S$, which we show in the following \Cref{thm:treatment-approximation-bound}.
\label{sec:approximation-bound-proof}

\treatmentapproximationbound*

\begin{proof}
Given a sample set $S=(x_i,y_i,z_i)_{i=1}^n$ of size $n$, we form an estimate of the treatment effect $\hat{\theta}_S$ via a Two-Stage Least Squares (2SLS). In the first stage, we regress $y_i-\bar{y}$ onto $z_i-\bar{z}$ to get the empirical estimate $\hat{\beta}_S$ and $x_i-\bar{x}$ onto $z_i-\bar{z}$ to get $\hat{\gamma}_S$ as such:
\begin{equation}
    \hat{\beta}_S := \frac{\sum_{i=1}^n (y_i - \bar{y})(z_i - \bar{z})}{\sum_{i=1}^n (z_i -\bar{z})^2} \hspace{1cm}\text{and}\hspace{1cm} 
    \hat{\gamma}_S:= \frac{\sum_{i=1}^n (x_i - \bar{x})(z_i - \bar{z})}{\sum_{i=1}^n (z_i -\bar{z})^2}
    \label{eq:beta-gamma-hat}
\end{equation}
In the second stage, we take the quotient of these two empirical estimates as the predicted treatment effect $\hat{\theta}_S$, i.e.
\begin{align}
    \hat{\theta}_S = \frac{\hat{\beta}_S}{\hat{\gamma}_S} \nonumber
    &= \left(\frac{\sum_{i=1}^n (y_i - \bar{y})(z_i - \bar{z})}{\sum_{i=1}^n (z_i -\bar{z})^2}\right) \left(\frac{\sum_{i=1}^n (z_i -\bar{z})^2}{\sum_{i=1}^n (x_i - \bar{x})(z_i - \bar{z})}\right) \nonumber\\
    &= \frac{\sum_{i=1}^n (y_i - \bar{y})(z_i - \bar{z})}{\sum_{i=1}^n (x_i - \bar{x})(z_i - \bar{z})}
\end{align}

Next, we can express the absolute value of the difference between the true treatment effect $\theta$ and the IV estimate of the treatment effect $\hat{\theta}_S$ given a sample set $S$ of size $n$ as such:
\begin{align}
    \left| \hat{\theta}_S - \theta \right|
    &= \left| \frac{\sum_{i=1}^n (y_i - \bar{y})(z_i - \bar{z})}{\sum_{i=1}^n (x_i - \bar{x})(z_i - \bar{z})} - \theta \right| \nonumber\\
    &= \left| \frac{\sum_{i=1}^n \left(\theta(x_i-\bar{x}) + g^{(u_i)}_i-\bar{g}\right)(z_i - \bar{z})}{\sum_{i=1}^n (x_i - \bar{x})(z_i - \bar{z})} - \theta \right| \tag{by \Cref{eq:centered-reward-treatment-choice}} \nonumber\\
    &= \left| \theta + \frac{\sum_{i=1}^n \left(g^{(u_i)}_i-\bar{g}\right)(z_i - \bar{z})}{\sum_{i=1}^n (x_i - \bar{x})(z_i - \bar{z})} - \theta \right| \nonumber\\
    &= \frac{\left|\sum_{i=1}^n \left(g^{(u_i)}_i-\bar{g}\right)(z_i - \bar{z})\right|}{\left|\sum_{i=1}^n (x_i - \bar{x})(z_i - \bar{z})\right|} \label{eq:last-line-estimator}
\end{align}

In order to complete our proof, we demonstrate an upper bound on the numerator \\
$\displaystyle\left|\sum_{i=1}^n \left(g^{(u_i)}_i-\bar{g}\right)(z_i - \bar{z})\right|$ of \Cref{eq:last-line-estimator} in the last line above. We do so in \Cref{eq:approximation-bound-numerator}.
\end{proof}

\begin{lemma}\label{eq:approximation-bound-numerator}
For all $\delta \in (0,1)$, with probability at least $1-\delta$, we have
\begin{equation}
    \left|\sum_{i=1}^n \left(g^{(u_i)}_i-\bar{g}\right)(z_i - \bar{z})\right| \leq 2\sigma_{g}\sqrt{2n\log(2/\delta)}
\end{equation}
if the set of $g^{(u_i)}_i$ are i.i.d. sub-Gaussian random variables with sub-Gaussian norm $\sigma_g$.
\end{lemma}

\begin{proof} We can rewrite the left hand side as follows
    \begin{align*}
        &\left|\sum_{i=1}^n \left(g^{(u_i)}_i-\bar{g}\right)(z_i - \bar{z})\right|\\
        &= \left|\sum_{i=1}^n \left(g^{(u_i)}_i-\E[g^{(u)}] + \E[g^{(u)}]-\bar{g}\right)(z_i - \bar{z})\right|\\
        &= \left|\sum_{i=1}^n \left(g^{(u_i)}_i-\E[g^{(u)}]\right)z_i - \sum_{i=1}^n \left(g^{(u_i)}_i-\E[g^{(u)}]\right)\bar{z} + \sum_{i=1}^n \left(\E[g^{(u)}]-\bar{g}\right)(z_i - \bar{z})\right|\\
        &= \left|\sum_{i=1}^n \left(g^{(u_i)}_i-\E[g^{(u)}]\right)z_i - \sum_{i=1}^n \left(g^{(u_i)}_i-\E[g^{(u)}]\right)\bar{z}\right| \tag{since $\sum_{i=1}^n(z_i - \bar{z})=0$}\\
        &\leq \left|\sum_{i=1}^n \left(g^{(u_i)}_i-\E[g^{(u)}]\right)z_i\right| + \left|\sum_{i=1}^n \left(g^{(u_i)}_i-\E[g^{(u)}]\right)\right| \tag{by the triangle inequality and $|\bar{z}|\leq1$}
    \end{align*}
    Now, if $g^{(u_i)}_i$ is sub-Gaussian, then the last line in the system of inequalities above is given as:
    \begin{align*}  
        &\left|\sum_{i=1}^n \left(g^{(u_i)}_i-\E[g^{(u)}]\right)z_i\right| + \left|\sum_{i=1}^n \left(g^{(u_i)}_i-\E[g^{(u)}]\right)\right|\\
        &\ \leq \left|\sigma_g\sqrt{2n_1\log(1/\delta_1)}\right| + \left|\sigma_g\sqrt{2n\log(1/\delta_2)}\right| \tag{by \Cref{thm:high-prob-unbounded-chernoff}, where $n_1:=\sum_{i=1}^nz_i$}\\
        &\ \leq 2\sigma_g\sqrt{2n\log(2/\delta)} \tag{since $n_1\leq\,n$ and by \Cref{thm:union}, where $\delta_1 = \delta_2 = \delta/2$}
    \end{align*}
    
    This recovers the stated bound and finishes the proof for \Cref{thm:treatment-approximation-bound}. \swdelete{If $g^{(u_i)}_i$ is bounded by $\Upsilon$, i.e. $|g^{(u_i)}_i|\leq\Upsilon$, then:
    \begin{align*}  
        &\left|\sum_{i=1}^n \left(g^{(u_i)}_i-\E[g^{(u)}]\right)z_i\right| + \left|\sum_{i=1}^n \left(g^{(u_i)}_i-\E[g^{(u)}]\right)\right|\\
        &\ \leq \left|2\Upsilon\sqrt{\frac{n_1\log(1/\delta_1)}{2}}\right| + \left|2\Upsilon\sqrt{\frac{n\log(1/\delta_2)}{2}}\right| \tag{by Chernoff Bound, where $n_1:=\sum_{i=1}^nz_i$}\\
        &\ \leq 4\Upsilon\sqrt{\frac{n\log(2/\delta)}{2}} \tag{since $n_1\leq\,n$ and by Union Bound, where $\delta_1 = \delta_2 = \delta/2$}\\
        &\ \leq 4\Upsilon\sqrt{2n\log(2/\delta)} \tag{since $\frac{n\log(2/\delta)}{2} \leq 2n\log(2/\delta)$}
    \end{align*}}
\end{proof}

Next, we demonstrate a lower bound on the denominator of  \Cref{thm:treatment-approximation-bound}, in terms of the level of compliance at each phase of \Cref{alg:sampling-control-treatment,alg:racing-two-types-mixed-preferences}.

\begin{theorem}[Lower bound on $\left|\sum_{i=1}^n(x_i-\bar{x})(z_i-\bar{z})\right|$ for a type 0 compliant sample set]\label{thm:approximation-bound-control-treatment-denominator}
    Let $S = (x_i,y_i,z_i)_{i=1}^n$ denote a sample set which satisfies the conditions of \Cref{thm:treatment-approximation-bound}. Furthermore, assume that there are $p_c$ fraction of agents in the population who would be compliant. Recall that $\bar z = \frac{1}{n}\sum_{i=1}^n z_i$ and $\bar x = \frac{1}{n}\sum_{i=1}^n x_i$. Then, the denominator of the approximation bound $A(S,\delta)$ (from \Cref{thm:treatment-approximation-bound}) is lower bounded as such:
    \small
    \[\left|\sum_{i=1}^n(x_i-\bar{x})(z_i-\bar{z})\right| \geq \begin{cases}
        n\bar{z}(1-\bar{z}) & \text{if $p_c=1$ (i.e. if all agents are compliant)};\\
        n\bar{z}(1-\bar{z})p_c-(3-\bar{z})\sqrt{\frac{n\bar{z}\log(3/\delta)}{2(1-\bar{z})}} & \text{otherwise}.
    \end{cases}\]
    \normalsize
    The second case above occurs with probability at least $1-\delta$ for any $\delta\in(0,1)$.
\end{theorem}
\begin{proof}
In this theorem, we formulate the denominator of the approximation bound in \Cref{thm:treatment-approximation-bound} in terms of $\bar{z}$, since $\bar{z}$ is determined by the social planner. For any type $u$, let $u\in U_c$ denote that agents of type $u$ comply; let $u\in U_0$ denote that agents of type $u$ are never-takers (agents which prefer control, according to their prior); and let $u\in U_1$ denote that agents of type $u$ are always-takers (agents which prefer treatment, according to their prior). Let $p_0$ and $p_1$ be the fractions of never-takers and always-takers, respectively.

Next, we expand the binomial in the denominator and arrive at the following simplified form:

\begin{equation}
\label{eq:approximation-bound-denominator-proof-1}
    \left|\sum_{i=1}^n(x_i-\bar{x})(z_i-\bar{z})\right|
    = \left|\sum_{i=1}^nx_iz_i-\bar{z}\sum_{i=1}^nx_i-\bar{x}\sum_{i=1}^nz_i+\bar{z}\bar{x}\right|
    = \left|\sum_{i=1}^nx_iz_i-n\bar{z}\bar{x}\right|
\end{equation}

First, observe that at any round $i$, the product $x_i=1$ only when agent $i$ is a non-compliant always-taker or when $z_i=1$ and agent $i$ is compliant. Formally, for any agent $i$, action choice $x_i=1$ is equivalent to the following: 
\begin{equation}
\label{eq:approximation-bound-denominator-proof-2}
    x_i=1 \equiv (z_i=1 \land u_i\in U_c) \lor (u_i\in U_1 \land u_i\not\in U_c)
\end{equation}

Then, the sum $\sum_{i=1}^nx_i$ can be expressed as follows:
\begin{align}
    \sum_{i=1}^nx_i
    &= \sum_{i=1}^n\1\left[(z_i=1 \land u_i\in U_c) \lor (u_i\in U_1 \land u_i\not\in U_c)\right]\nonumber\\
    &= \sum_{i=1}^n\1\left[(z_i=1 \land u_i\in U_c\right] + \sum_{i=1}^n\1\left[u_i\in U_1 \land u_i\not\in U_c\right]\nonumber\\
    &= \left(\sum_{i=1}^nz_i\right)\left(\hat{p}_{c:z_i=1} + n\hat{p}_{nc1}\right)\nonumber\\
    &= n\left(\bar{z}\hat{p}_{c:z_i=1} + \hat{p}_{nc1}\right)\label{eq:approximation-bound-denominator-proof-3}
\end{align}
where we define $\hat{p}_{c:z_i=1}$ as the empirical proportion of agents with types in $U_c$ when the recommendation $z=1$ and $\hat{p}_{nc1}$ as the empirical proportion of non-compliant always-takers. Formally, $\hat{p}_{c:z_i=1}=\frac{1}{n}\sum_{i=1}^n\1[u_i\in U_c, z_i=1]$ and $\hat{p}_{nc1}=\frac{1}{n}\sum_{i=1}^n\1[u_i\in U_1 \land u_i\not\in U_c]$. Define $p_{nc1}$ to be the proportion of non-compliant always-takers in the population of agents. Then, in expectation over the randomness of how agents arrive, $\E[\hat{p}_{c:z_i=1}]=p_c$ and $\E[\hat{p}_{nc1}]=p_{nc1}$.

Next, we rewrite the sum $\sum_{i=1}^nx_iz_i$ in terms of $\bar{z}$ and some population constants. Observe that at any round $i$, the product $x_iz_i=1$ only when both $x_i=1$ and $z_i=1$. Thus, by \Cref{eq:approximation-bound-denominator-proof-2}, for any agent $i$, the event $x_iz_i=1$ is equivalent to the following:
\begin{align*}
    x_iz_i=1 
    &\equiv z_i=1 \land \left((z_i=1 \land u_i\in U_c) \lor (u_i\in U_1 \land u_i\not\in U_c)\right)\\
    &\equiv z_i=1 \land \left(u_i\in U_c \lor (u_i\in U_1 \land u_i\not\in U_c)\right)\\
\end{align*}

Then, the sum $\sum_{i=1}^nx_iz_i$ can be expressed as follows:
\begin{align}
    \sum_{i=1}^nx_iz_i
    &= \sum_{i=1}^n\1\left[z_i=1 \land \left(u_i\in U_c \lor (u_i\in U_1 \land u_i\not\in U_c)\right)\right]\nonumber\\
    &= \left(\sum_{i=1}^n\1[z_i=1]\right)\left(\hat{p}_c|_{z_i=1} + \hat{p}_{nc1}|_{z_i=1}\right)\nonumber\\
    &= n\bar{z}\left(\hat{p}_{c:z_i=1} + \hat{p}_{nc1:z_i=1}\right)\label{eq:approximation-bound-denominator-proof-4}
\end{align}
where we define $\hat{p}_{nc1:z_i=1}$ as the empirical proportions of non-compliant always-takers who arrive when $z_i=1$ ---i.e. $\hat{p}_{nc1:z_i=1}=\frac{1}{n}\sum_{i=1}^n\1[u_i\in U_1 \land u_i\not\in U_c, z_i=1]$. In expectation over the randomness of how agents arrive, $\E[\hat{p}_{nc1:z_i=1}]=p_{nc1}$.

Finally, by \Cref{eq:approximation-bound-denominator-proof-1,eq:approximation-bound-denominator-proof-3,eq:approximation-bound-denominator-proof-4}, we can provide a high probability lower bound on the denominator as such:
\small
\begin{align*} \left|\sum_{i=1}^n(x_i-\bar{x})(z_i-\bar{z})\right|
    &= \left|\sum_{i=1}^nx_iz_i-n\bar{z}\bar{x}\right| \tag{by \Cref{eq:approximation-bound-denominator-proof-1}}\\
    &= \left|\sum_{i=1}^nx_iz_i-n\bar{z}\bar{x}\right| \tag{by \Cref{eq:approximation-bound-denominator-proof-1}}\\
    &= \left|n\bar{z}\left(\hat{p}_{c:z_i=1} + \hat{p}_{nc1:z_i=1}\right) - n\bar{z}\left(\bar{z}\hat{p}_{c:z_i=1} + \hat{p}_{nc1}\right)\right| \tag{by \Cref{eq:approximation-bound-denominator-proof-3,eq:approximation-bound-denominator-proof-4}}\\
    &= \left|n\bar{z}\left((1-\bar{z})\hat{p}_{c:z_i=1} + \hat{p}_{nc1:z_i=1} - \hat{p}_{nc1}\right)\right|\\
    &\geq \left|n\bar{z}\left((1-\bar{z})\left(p_c - \sqrt{\frac{\log(1/\delta_1)}{2n\bar{z}}}\right) + p_{nc1} - \sqrt{\frac{\log(1/\delta_2)}{2n\bar{z}}} - \left(p_{nc1} + \sqrt{\frac{\log(1/\delta_3)}{2n}}\right)\right)\right|\tag{by \Cref{thm:bounded-chernoff}}\\
    &\geq \left|n\bar{z}\left((1-\bar{z})p_c - (1-\bar{z})\sqrt{\frac{\log(3/\delta)}{2n\bar{z}}} - \sqrt{\frac{\log(3/\delta)}{2n\bar{z}}} - \sqrt{\frac{\log(3/\delta)}{2n}}\right)\right|\tag{by \Cref{thm:union} where $\delta_1=\delta_2=\delta_3=\delta/3$}\\
    &\geq \left|n\bar{z}\left((1-\bar{z})p_c - (3-\bar{z})\sqrt{\frac{\log(3/\delta)}{2n\bar{z}}}\right)\right|\\
    &= n\bar{z}(1-\bar{z})p_c-(3-\bar{z})\sqrt{\frac{n\bar{z}\log(3/\delta)}{2(1-\bar{z})}}\\
\end{align*}
\normalsize
with probability at least $1-\delta$ for any $\delta\in(0,1)$.
\end{proof}


\section{Missing Proofs for \Cref{sec:sampling-control-treatment}}
\label{sec:sampling-stage-appendix}
\begin{claim}
\label{claim:bic-equiv} For any agent $t$ at round $t$ with recommendation policy $\pi_t$ with a positive probability of recommending either control or treatment, according to the prior $\cP^{(u_t)}$, i.e. $\Prob_{\pi_t,\cP^{(u_t)}}[z_t = 0]>0$ and $\Prob_{\pi_t,\cP^{(u_t)}}[z_t = 1]>0$. Furthermore, $a^{(u)}$ and $b^{(u)}$ denote the initially preferred and unpreferred actions for any type $u$, i.e. $a^{(u)}:=\1[\E_{\cP^{(u)}}[\theta]\geq0]$ and $b^{(u)}:=\1[\E_{\cP^{(u)}}[\theta]<0]$. Formally, the following holds:
\small
\[\left\{ (-1)^{a^{(u_t)}} \mkern-18mu \E_{\pi_t,\cP^{(u_t)}}[\theta|z_t = b^{(u_t)}]\Prob_{\pi_t,\cP^{(u_t)}}[z_t = b^{(u_t)}] \geq 0\right\}  \Rightarrow \left\{ (-1)^{b^{(u_t)}} \mkern-18mu \E_{\pi_t,\cP^{(u_t)}}[\theta|z_t = a^{(u_t)}]\Prob_{\pi_t,\cP^{(u_t)}}[z_t = a^{(u_t)}] < 0 \right\}\]
\normalsize
\end{claim}

\begin{proof}
Note that $a^{(u)}$ is defined in such a way that $(-1)^{a^{(u)}}\E_{\cP^{(u)}}[\theta]<0$ always: if agents of type $u$ prefer initially control, then $a^{(u)}=0$ and $(-1)^{a^{(u)}}\E_{\cP^{(u)}}[\theta]=\E_{\cP^{(u)}}[\theta]<0$; if agents of type $u$ initially prefer treatment, then $a^{(u)}=1$ and
$(-1)^{a^{(u)}}\E_{\cP^{(u)}}[\theta]=-\E_{\cP^{(u)}}[\theta]<0$. Then, 

Recall that we assume that type 0 agents prefer the control, i.e. the expected treatment effect $\displaystyle \E_{\cP^{(0)}}[\theta]< 0$. Then:
\begin{align*}
    &(-1)^{a^{(u_t)}} \mkern-18mu \E_{\pi_t,\cP^{(u_t)}}[\theta|z_t = b^{(u_t)}]\Prob_{\pi_t,\cP^{(u_t)}}[z_t = b^{(u_t)}] - (-1)^{b^{(u_t)}} \mkern-18mu \E_{\pi_t,\cP^{(u_t)}}[\theta|z_t = a^{(u_t)}]\Prob_{\pi_t,\cP^{(u_t)}}[z_t = a^{(u_t)}]\\
    &= (-1)^{a^{(u_t)}} \mkern-18mu \E_{\pi_t,\cP^{(u_t)}}[\theta|z_t = b^{(u_t)}]\Prob_{\pi_t,\cP^{(u_t)}}[z_t = b^{(u_t)}] + (-1)^{a^{(u_t)}} \mkern-18mu \E_{\pi_t,\cP^{(u_t)}}[\theta|z_t = a^{(u_t)}]\Prob_{\pi_t,\cP^{(u_t)}}[z_t = a^{(u_t)}]\\
    &= (-1)^{a^{(u_t)}} \mkern-18mu \E_{\pi_t,\cP^{(u_t)}}[\theta] < 0.
\end{align*}
Therefore, given that both $\Prob_{\pi_t,\cP^{(u_t)}}[z_t = 0]>0$ and $\Prob_{\pi_t,\cP^{(u_t)}}[z_t = 1]>0$ and, by assumption, $(-1)^{a^{(u_t)}}\E_{\pi_t,\cP^{(u_t)}}[\theta|z_t = b^{(u_t)}]\Prob_{\pi_t,\cP^{(u_t)}}[z_t = b^{(u_t)}] \geq 0$, then it must be that\\ $(-1)^{b^{(u_t)}}\E_{\pi_t,\cP^{(u_t)}}[\theta|z_t = a^{(u_t)}]\Prob_{\pi_t,\cP^{(u_t)}}[z_t = a^{(u_t)}] < 0$.
\end{proof}

\subsection{\Cref{alg:sampling-control-treatment} Proofs and Extension 1}\label{sec:sampling-bic-proof}
\bicsamplingcontroltreatment*

\begin{proof} 
    Let the event $\xi = \xi^{(0)}$ (as given by \Cref{def:sampling-control-treatment-extension}). By \Cref{lemma:bic-sampling-control-treatment-extension}, if $\rho$ satisfies the following condition, then any type 0 agent will comply with any recommendation of the last $\ell$ rounds of \Cref{alg:sampling-control-treatment}:
    \begin{equation}
        \rho \leq 1 + \frac{4\mu^{(0)}}{\Prob_{\cP^{(0)}}[\xi^{(0)}]-4\mu^{(0)}}
    \end{equation}
\end{proof}

\begin{definition}[Extension 1 of \Cref{alg:sampling-control-treatment}]
\label{def:sampling-control-treatment-extension}
    Here, we formalize the recommendation policy of Extension 1 in \Cref{sec:sampling-extensions}, which modifies \Cref{alg:sampling-control-treatment} in two ways:
    \begin{enumerate}
        \item We redefine event $\xi$ as $\xi^{(u)}$ such that it is relative to any type $u$, defined as follows: 
    \begin{equation}
        \label{eq:bic-sampling-proof-xi}
        \xi^{(u)} = \left\{\bar{y}^1 > \bar{y}^0 + \sigma_g\left( \sqrt{\frac{2\log(2/\delta)}{\ell_0}} + \sqrt{\frac{2\log(2/\delta)}{\ell_1}}\right) + G^{(u_t)} + \frac{1}{2} \right\},
    \end{equation}
    where $G^{(u_t)}$ is an upper bound on the difference between the prior mean of the treatment versus the control according to type $u$, i.e. $G^{(u_t)}>\E_{\cP^{(u)}}[g^{1}-g^{0}]$, and where $\E_{\cP^{(u)}}[g^0]$ and $\E_{\cP^{(u)}}[g^1]$ are the expected baseline rewards for initial never-takers and always-takers.
        \item If we are trying to incentivize compliance for always-takers, then those agents in the exploration set $E$ are recommended control (rather than treatment, as described in the pseudocode for \Cref{alg:sampling-control-treatment}).
    \end{enumerate}
\end{definition}

\begin{lemma}[Arbitrary Type Compliance with Extension 1 of \Cref{alg:sampling-control-treatment}]\label{lemma:bic-sampling-control-treatment-extension} Under \Cref{ass:know-sampling}, any type $u_t$ agent who arrives at round $t$ in the last $\ell$ rounds of Extension 1 of \Cref{alg:sampling-control-treatment} (given in \Cref{def:sampling-control-treatment-extension}) is compliant with any recommendation $z_t$, as long as the exploration probability $\rho$ satisfies:
    \begin{equation}
        \rho \leq 1 + \frac{4\mu^{(u_t)}}{\Prob_{\pi_t,\cP^{(u_t)}}[\xi^{(u_t)}]-4\mu^{(u_t)}}
    \end{equation}
where the event $\xi^{(u_t)}$ is defined in \Cref{def:sampling-control-treatment-extension}. 
\end{lemma}

\begin{proof} This proof follows a similar structure to the Sampling Stage BIC proof in \cite{mansour2015bic}.

We will prove compliance for any type $u$ in the more general Extension 1 of \Cref{alg:sampling-control-treatment}, as given in \Cref{def:sampling-control-treatment-extension}, which admits arbitrarily many types and the option to incentivize initial always-takers, instead of initial never-takers, to comply.

Let recommendation policy $\pi$ be that described in \Cref{def:sampling-control-treatment-extension}, i.e. Extension 1 of \Cref{alg:sampling-control-treatment} which admits arbitrarily many types and allows for the exploration recommendations to be given in order to incentivize initial always-takers, instead of initial never-takers, to comply. Throughout this proof, we will assume that the exploration set $E$ is defined relative to the initial preference of any agent of type $u_t$, who we are proving compliance for.

According to the selection function in \Cref{eq:control-treatment-selection-function}, if any agent $t$ expects the treatment effect $\theta$ to be positive, they will select the treatment $x_t=1$. Conversely, if they expect the treatment effect $\theta$ to be negative, they will select control $x_t=0$. Thus, for any agent of type $u_t$ at round $t$, proving compliance entails the expected treatment effect $\theta$ over the prior of type $u_t$ and policy $\pi_t$ is positive given that the recommendation $z_t=1$ and negative given that the recommendation $z_t=0$, i.e. \[\E_{\pi_t,\cP^{(u_t)}}[\theta|z_t=1]\geq0 \quad\quad\quad\text{and}\quad\quad \E_{\pi_t,\cP^{(u_t)}}[\theta|z_t=0]<0.\]

Next, we show that we can reduce our proof to demonstrating only one of the above statements, depending on the prior preference of type $u$. Let $a^{(u)}$ and $b^{(u)}$ denote the prior preferred and unpreferred actions for any type $u$, i.e. $a^{(u)}:=\1[\E_{\cP^{(u)}}[\theta]\geq0]$ and $b^{(u)}:=\1[\E_{\cP^{(u)}}[\theta]<0]$. Because policy $\pi$ (\Cref{alg:sampling-control-treatment} extension) is designed in a such way that at any round $t$ in the last $\ell$ rounds, treatment or control is recommended each with positive probability ---i.e. $\Prob_{\pi_t,\cP^{(u_t)}}[z_t = 1] > 0$ and $\Prob_{\pi_t,\cP^{(u_t)}}[z_t = 0] > 0$,--- \Cref{claim:bic-equiv} applies and the following holds:
\small
\[\left\{ (-1)^{a^{(u_t)}} \mkern-18mu \E_{\pi_t,\cP^{(u_t)}}[\theta|z_t = b^{(u_t)}]\Prob_{\pi_t,\cP^{(u_t)}}[z_t = b^{(u_t)}] \geq 0\right\}  \Rightarrow \left\{ (-1)^{b^{(u_t)}} \mkern-18mu \E_{\pi_t,\cP^{(u_t)}}[\theta|z_t = a^{(u_t)}]\Prob_{\pi_t,\cP^{(u_t)}}[z_t = a^{(u_t)}] < 0 \right\}.\]
\normalsize

Thus, at round $t$, in order to prove compliance for agents of type $u_t$ with prior preferred and unpreferred actions $a^{(u_t)}$ and $b^{(u_t)}$, respectively, it suffices to demonstrate that $(-1)^{a^{(u_t)}}\E_{\pi_t,\cP^{(u_t)}}[\theta|z_t = b^{(u_t)}]\Prob_{\pi_t,\cP^{(u_t)}}[z_t = b^{(u_t)}] \geq 0$. The remainder of the proof is devoted to demonstrating this.

We first rewrite $\displaystyle\E_{\pi_t,\cP^{(u_t)}}[\theta | z_t = b^{(u_t)}]\Prob_{\pi_t,\cP^{(u_t)}}[z_t = b^{(u_t)}]$ in terms of the event $\xi^{(u_t)}$:
\small
\begin{align}
&\E_{\pi_t,\cP^{(u_t)}}[\theta | z_t = b^{(u_t)}]\Prob_{\pi_t,\cP^{(u_t)}}[z_t = b^{(u_t)}] \nonumber\\
	&=\mkern-10mu \E_{\pi_t,\cP^{(u_t)}}[\theta |z_t = b^{(u_t)} \ \& \ t \not\in E]\mkern-10mu\Prob_{\pi_t,\cP^{(u_t)}}[z_t = b^{(u_t)} \ \& \ t \not\in E] + \mkern-10mu \E_{\pi_t,\cP^{(u_t)}}[\theta |z_t = b^{(u_t)} \ \& \ t \in E]\mkern-10mu\Prob_{\pi_t,\cP^{(u_t)}}[z_t = b^{(u_t)} \ \& \ t \in E] \tag{for explore set $E$ defined relative to recommend action $b^{(u_t)}$} \nonumber\\
	&=\mkern-10mu \E_{\pi_t,\cP^{(u_t)}}[\theta |\xi \ \& \ t \not\in E]\Prob_{\pi_t,\cP^{(u_t)}}[\xi \ \& \ t \not\in E] + \E_{\pi_t,\cP^{(u_t)}}[\theta |z_t = b^{(u_t)} \ \& \ t \in E]\Prob_{\pi_t,\cP^{(u_t)}}[z_t = b^{(u_t)} \ \& \ t \in E] \tag{since the only way $z_t = b^{(u_t)}$ when exploiting (i.e. when $t \not\in E$) is when event $\xi$ occurs} \nonumber\\
	&=\mkern-10mu \E_{\pi_t,\cP^{(u_t)}}[\theta |\xi \ \& \ t \not\in E]\Prob_{\pi_t,\cP^{(u_t)}}[\xi \ \& \ t \not\in E] + \E_{\pi_t,\cP^{(u_t)}}[\theta |t\in E]\Prob_{\pi_t,\cP^{(u_t)}}[t \in E], \ \ \tag{$t \in E \Rightarrow z_t=1$ by definition of $E$} \nonumber\\
	&=\mkern-10mu \E_{\pi_t,\cP^{(u_t)}}[\theta | \xi]\Prob_{\pi_t,\cP^{(u_t)}}[\xi]\Prob_{\pi_t,\cP^{(u_t)}}[t \not\in E] + \E_{\pi_t,\cP^{(u_t)}}[\theta]\Prob_{\pi_t,\cP^{(u_t)}}[t \in E] \tag{since $\theta \perp t \in E$ and $\xi\perp t \not\in E$} \nonumber\\
	&= (1 - \rho)\E_{\pi_t,\cP^{(u_t)}}[\theta | \xi]\Prob_{\pi_t,\cP^{(u_t)}}[\xi] + \rho\E_{\pi_t,\cP^{(u_t)}}[\theta] \tag{since agent $t \in E$ with probability $\rho$} \nonumber\\
	&= (1 - \rho)\E_{\pi_t,\cP^{(u_t)}}[\theta | \xi]\Prob_{\pi_t,\cP^{(u_t)}}[\xi] + \rho\mu^{(u)} \hspace{3cm}\text{(by definition, $\E_{\pi_t,\cP^{(u_t)}}[\theta]=\E_{\cP^{(u_t)}}[\theta]=\mu^{(u)}$)} \label{eq:part-1-compliance-proof-algo-1}
\end{align}
\normalsize

Now, we can rewrite our compliance condition as such: 
\small\[ (-1)^{a^{(u_t)}} \mkern-18mu \E_{\pi_t,\cP^{(u_t)}}[\theta|z_t = b^{(u_t)}]\Prob_{\pi_t,\cP^{(u_t)}}[z_t = b^{(u_t)}] \geq 0 \equiv
(-1)^{a^{(u_t)}} \mkern-8mu \left( (1 - \rho)\E_{\pi_t,\cP^{(u_t)}}[\theta | \xi]\Prob_{\pi_t,\cP^{(u_t)}}[\xi] + \rho\mu^{(u)} \right) \geq 0.\]\normalsize Now, we rewrite this compliance condition strictly in terms of the exploration probability $\rho$ and relative to a number of constants which depend on the prior $\cP^{(u_t)}$. Thus, if we set $\rho$ to satisfy the following condition (in \Cref{eq:part-2-compliance-proof-algo-1}), then all agents of type $u$ will comply with recommendations from policy $\pi$ (\Cref{alg:sampling-control-treatment} extension):
\begin{gather}
    (-1)^{a^{(u_t)}} \mkern-18mu \E_{\pi_t,\cP^{(u_t)}}[\theta|z_t = b^{(u_t)}]\Prob_{\pi_t,\cP^{(u_t)}}[z_t = b^{(u_t)}] \geq 0 \nonumber\\
    (-1)^{a^{(u_t)}} \mkern-8mu \left( (1 - \rho)\E_{\pi_t,\cP^{(u_t)}}[\theta | \xi^{(u_t)}]\Prob_{\pi_t,\cP^{(u_t)}}[\xi^{(u_t)}] + \rho\mu^{(u_t)} \right) \geq 0 \tag{by \Cref{eq:part-1-compliance-proof-algo-1}} \nonumber\\
	(-1)^{a^{(u_t)}} \mkern-8mu \left( \E_{\pi_t,\cP^{(u_t)}}[\theta | \xi^{(u_t)}]\Prob_{\pi_t,\cP^{(u_t)}}[\xi^{(u_t)}] - \rho\E_{\pi_t,\cP^{(u_t)}}[\theta |\xi^{(u_t)}] \Prob_{\pi_t,\cP^{(u_t)}}[\xi^{(u_t)}] + \rho\mu^{(u_t)} \right) \geq 0  \nonumber\\
	 (-1)^{a^{(u_t)+1}}\rho\left(\E_{\pi_t,\cP^{(u_t)}}[\theta |\xi^{(u_t)}] \Prob_{\pi_t,\cP^{(u_t)}}[\xi^{(u_t)}] - \mu^{(u_t)}\right) \geq (-1)^{a^{(u_t)+1}} \mkern-18mu \E_{\pi_t,\cP^{(u_t)}}[\theta | \xi^{(u_t)}]\Prob_{\pi_t,\cP^{(u_t)}}[\xi^{(u_t)}] \nonumber\\
	 \rho \leq \frac{\E_{\pi_t,\cP^{(u_t)}}[\theta | \xi^{(u_t)}]\Prob_{\pi_t,\cP^{(u_t)}}[\xi^{(u_t)}]}{\E_{\pi_t,\cP^{(u_t)}}[\theta |\xi^{(u_t)}] \Prob_{\pi_t,\cP^{(u_t)}}[\xi^{(u_t)}] - \mu^{(0)}} \nonumber \tag{since $(-1)^{a^{(u_t)+1}}\rho\left(\E_{\pi_t,\cP^{(u_t)}}[\theta |\xi^{(u_t)}] \Prob_{\pi_t,\cP^{(u_t)}}[\xi^{(u_t)}] - \mu^{(u_t)}\right)<0$ for any $u_t$\footnotemark}\\
	 \rho \leq 1+\frac{\mu^{(u_t)}}{\E_{\pi_t,\cP^{(u_t)}}[\theta |\xi^{(u_t)}] \Prob_{\pi_t,\cP^{(u_t)}}[\xi^{(u_t)}] - \mu^{(u_t)}} \label{eq:part-2-compliance-proof-algo-1}
\end{gather}

\footnotetext{This point is not entirely obvious: If $a^{(u_t)}=0$, then $(-1)^{a^{(u_t)+1}}<0$ and $\E_{\pi_t,\cP^{(u_t)}}[\theta |\xi^{(u_t)}]-\mu^{(u_t)}>0$, since $\E_{\pi_t,\cP^{(u_t)}}[\theta |\xi^{(u_t)}]>0$ and $\mu^{(u_t)}<0$. If $a^{(u_t)}=1$, then $(-1)^{a^{(u_t)+1}}>0$ and $\E_{\pi_t,\cP^{(u_t)}}[\theta |\xi^{(u_t)}]-\mu^{(u_t)}<0$, since $\E_{\pi_t,\cP^{(u_t)}}[\theta |\xi^{(u_t)}]<0$ and $\mu^{(u_t)}>0$.}

Finally, we can further simplify the upper bound on $\rho$ given in \Cref{eq:part-2-compliance-proof-algo-1} above by showing that $\E_{\pi_t,\cP^{(u_t)}}[\theta |\xi^{(u_t)}]$ satisfies some constant lower bound. This will complete our proof.

For any type $u$, the baseline reward $g^{(u)}$ is a random variable independently distributed according to a sub-Gaussian distribution with variance $\sigma^{(u)}$ which is bounded above by $\sigma_g$, i.e. $\sigma^{(u)}<\sigma_g$ for any $u$. Furthermore, recall that $G^{(u_t)}>\E_{\cP^{u_t}}[g^1-g^0]$, where $\E_{\cP^{u_t}}[g^1]$ and $\E_{\cP^{u_t}}[g^0]$ are the expected value of the baseline rewards of always-takers and never-takers over the prior of type $u_t$, respectively. 

Now, we define 3 clean events: $\cC_0$ and $\cC_1$ pertain to these baseline reward random variables, and $\cC_2$ occurs when the first stage of \Cref{alg:sampling-control-treatment} generates at least $\ell_0$ control samples and at least $\ell_1$ treatment samples:
\begin{align}
     \cC_0 &:= \left\{ \bar{y}^0=\frac{1}{\sum_{t=1}^{\ell}\1[\mu^{(u_t)}<0]} \sum_{t=1}^{\ell} g^{(u_t)}\1[\mu^{(u_t)}<0] \leq \sigma_g \sqrt{\frac{2\log(1/\delta_0)}{\ell_0}} - \E_{\pi_t,\cP^{(u_t)}}[g^0]\right\}\\
      \cC_1 &:= \left\{ \bar{y}^1=\frac{1}{\sum_{t=1}^{\ell}\1[\mu^{(u_t)}>0]} \sum_{t=1}^{\ell} g^{(u_t)}\1[\mu^{(u_t)}>0]  \geq -\sigma_g \sqrt{\frac{2\log(1/\delta_1)}{\ell_1}}- \E_{\pi_t,\cP^{(u_t)}}[g^1] \right\}\\
      \cC_2 &:= \left\{ \ell_1\leq\sum_{i=1}^{\ell'}x_i\leq\ell'-\ell_0\right\}
\end{align}
where $\ell'=2\max(\ell_0/p_0,\ell_1/p_1)$ is the number of rounds in the first stage of \Cref{alg:sampling-control-treatment}. Let $\delta_0=\delta_1=\Prob_{\pi_t,\cP^{(u_t)}}[\xi^{(u_t)}]/24$. Furthermore, event $\cC_2$ occurs when the binomial random variable with success $u_t=x_t=1$ (since $x_t=u_t$ in the first stage of \Cref{alg:sampling-control-treatment}) and success probability $p_1$ is lower bounded by $\ell_1$ and upper bounded by $\ell'-\ell_0$. For $\ell'=2\max(\ell_0/p_0,\ell_1/p_1)$ total trials, the probability of this event is less than $\Prob_{\pi_t,\cP^{(u_t)}}[\xi^{(u_t)}]/24$.

Now, define another clean event $\cC$ where all $\cC_0$, $\cC_1$, and $\cC_2$ happen simultaneously. Letting $\delta=\delta_0+\delta_1+\delta_2$, the event $\cC$ occurs with probability at least $1 - \delta$ where  $\displaystyle \delta < \Prob_{\pi_t,\cP^{(u_t)}}[\xi^{(u_t)}]/8$. We can now rewrite $\E_{\pi_t,\cP^{(u_t)}}[\theta | \xi^{(u_t)}] \Prob_{\pi_t,\cP^{(u_t)}}[\xi^{(u_t)}]$ in terms of event $\cC$:
\begin{align}
    &\E_{\pi_t,\cP^{(u_t)}}[\theta | \xi^{(u_t)}] \Prob_{\pi_t,\cP^{(u_t)}}[\xi^{(u_t)}]\\
    &= \E_{\pi_t,\cP^{(u_t)}}[\theta | \xi^{(u_t)}, \cC] \Prob_{\pi_t,\cP^{(u_t)}}[\xi^{(u_t)}, \cC] + \E_{\pi_t,\cP^{(u_t)}}[\theta | \xi^{(u_t)}, \neg \cC] \Prob_{\pi_t,\cP^{(u_t)}}[\xi^{(u_t)}, \neg \cC] \nonumber\\
    &\geq \E_{\pi_t,\cP^{(u_t)}}[\theta | \xi^{(u_t)}, \cC] \Prob_{\pi_t,\cP^{(u_t)}}[\xi^{(u_t)}, \cC] - \delta \tag{since $\Prob_{\pi_t,\cP^{(u_t)}}[\neg \cC]<\delta$ and $\theta \geq -1$ by definition} \nonumber\\
    &\geq \E_{\pi_t,\cP^{(u_t)}}[\theta | \xi^{(u_t)}, \cC] \left(\Prob_{\pi_t,\cP^{(u_t)}}[\xi^{(u_t)}] - \Prob_{\pi_t,\cP^{(u_t)}}[\neg \cC]\right) - \delta \nonumber\\
    &\geq \E_{\pi_t,\cP^{(u_t)}}[\theta | \xi^{(u_t)}, \cC] \left(\Prob_{\pi_t,\cP^{(u_t)}}[\xi^{(u_t)}] - \delta\right) - \delta \nonumber\\
    &= \E_{\pi_t,\cP^{(u_t)}}[\theta | \xi^{(u_t)}, \cC] \Prob_{\pi_t,\cP^{(u_t)}}[\xi^{(u_t)}] - \delta\left(1 + \E_{\pi_t,\cP^{(u_t)}}[\theta | \xi^{(u_t)}, \cC]\right) \nonumber\\
    &\geq \E_{\pi_t,\cP^{(u_t)}}[\theta | \xi^{(u_t)}, \cC] \Prob_{\pi_t,\cP^{(u_t)}}[\xi^{(u_t)}] - 2\delta \hspace{8mm} \text{(since $\E_{\pi_t,\cP^{(u_t)}}[\theta | \xi^{(u_t)}, \cC]\leq1$)} \label{eq:sampling-control-treatment-clean-event}
\end{align}

This comes down to finding a lower bound on the denominator of the expression above. We can reduce the dependency of the denominator to a single prior-dependent constant $\Prob_{\pi_t,\cP^{(u_t)}}[\xi^{(u_t)}]$ if we lower bound the prior-dependent expected value $\E_{\pi_t,\cP^{(u_t)}}[\theta | \xi^{(u_t)}]$. That way, assuming we know the prior and can calculate the probability of event $\xi^{(u_t)}$, we can pick an appropriate exploration probability $\rho$ to satisfy the compliance condition for all agents of type $0$. Then:
\small
\begin{align}
    &\E_{\pi_t,\cP^{(u_t)}}[\theta | \xi^{(u_t)}, \cC] \nonumber\\
    &= \E_{\pi_t,\cP^{(u_t)}}\left[\theta \middle| \bar{y}^1 > \bar{y}^0 + \sigma_g\left(\sqrt{\frac{2\log(1/\delta)}{\ell_0}} + \sqrt{\frac{2\log(1/\delta)}{\ell_1}}\right) + G^{(u_t)} + \frac{1}{2}, \cC\right] \nonumber\\
    &\geq \E_{\pi_t,\cP^{(u_t)}}\left[\theta \middle| \theta > \frac{1}{\ell_0}\sum_{t=1}^{\ell_0}g^0 - \frac{1}{\ell_1}\sum_{t=1}^{\ell_1}\theta + \sigma_g\left(\sqrt{\frac{2\log(1/\delta)}{\ell_0}} + \sqrt{\frac{2\log(1/\delta)}{\ell_1}}\right) + G^{(u_t)} + \frac{1}{2}, \cC\right] \nonumber\\
    &= \E_{\pi_t,\cP^{(u_t)}}\left[\theta \middle| \theta > - \bigg(\E_{\pi_t,\cP^{(u_t)}}[g^0] - \sigma_g\sqrt{\frac{2\log(1/\delta_1)}{\ell_0}}\bigg) + \E_{\pi_t,\cP^{(u_t)}}[g^1] - \sigma_g\sqrt{\frac{2\log(1/\delta_2)}{\ell_1}}\right.\nonumber\\
    & \hspace{6cm} \left. + \ \sigma_g\left(\sqrt{\frac{2\log(1/\delta)}{\ell_0}} + \sqrt{\frac{2\log(1/\delta)}{\ell_1}}\right) + G^{(u_t)} + \frac{1}{2}, \cC\right] \tag{by event $\cC$}\nonumber\\
    &= \E_{\pi_t,\cP^{(u_t)}}\left[\theta \middle| \theta > -\sigma_g\left(\sqrt{\frac{2\log(2/\delta)}{\ell_0}} -\sqrt{\frac{2\log(2/\delta)}{\ell_1}}\right) +\E_{\pi_t,\cP^{(u_t)}}[g^1-g^0] \right.\\
    & \hspace{6.25cm} \left. +\sigma_g\left(\sqrt{\frac{2\log(2/\delta)}{\ell_0}} + \sqrt{\frac{2\log(2/\delta)}{\ell_1}}\right) + G^{(u_t)} + \frac{1}{2}, \cC\right] \tag{by \Cref{thm:union}, where $\delta_1 = \delta_2 = \delta/2$ }\nonumber\\
    &> \E_{\pi_t,\cP^{(u_t)}}\left[\theta \middle| \theta > \E_{\pi_t,\cP^{(u_t)}}[g^1-g^0] - \E_{\pi_t,\cP^{(u_t)}}[g^1-g^0] + \frac{1}{2}\right]\nonumber \tag{by definition of $G^{(u_t)}$}\\
    &> \frac{1}{2} \label{eq:sampling-stage-two-arm-gap}
\end{align}
\normalsize
Hence, the term $\E_{\pi_t,\cP^{(u_t)}}[\theta | \xi^{(u_t)}] \Prob_{\pi_t,\cP^{(u_t)}}[\xi^{(u_t)}]$ satisfies the following lower bound:
\begin{align}
    \E_{\pi_t,\cP^{(u_t)}}[\theta | \xi^{(u_t)}] \Prob_{\pi_t,\cP^{(u_t)}}[\xi^{(u_t)}]
    &\geq \E_{\pi_t,\cP^{(u_t)}}[\theta | \xi^{(u_t)}, \cC]\Prob_{\pi_t,\cP^{(u_t)}}[\xi^{(u_t)}] - 2\delta \tag{by \Cref{eq:sampling-control-treatment-clean-event}}\\
    &> \frac{1}{2}\Prob_{\pi_t,\cP^{(u_t)}}[\xi^{(u_t)}] - 2\delta \tag{by \Cref{eq:sampling-stage-two-arm-gap}}\\ 
    &= \frac{\Prob_{\pi_t,\cP^{(u_t)}}[\xi^{(u_t)}]}{4} + \frac{\Prob_{\pi_t,\cP^{(u_t)}}[\xi^{(u_t)}]}{4} - 2\delta \nonumber\\
    &> \frac{\Prob_{\pi_t,\cP^{(u_t)}}[\xi^{(u_t)}]}{4} \tag{since $\delta < \Prob_{\pi_t,\cP^{(u_t)}}[\xi^{(u_t)}]/8$}
\end{align}
Substituting this into \Cref{eq:part-2-compliance-proof-algo-1}, we arrive at a lower bound to set the exploration probability $\rho$ for the agent any round $t$ with type $u_t$ to comply with recommendation policy $\pi_t$ (extension of \Cref{alg:sampling-control-treatment}):
\begin{align*}
    \rho &\leq 1 + \frac{4\mu^{(u_t)}}{\Prob_{\pi_t,\cP^{(u_t)}}[\xi^{(u_t)}]-4\mu^{(u_t)}}
\end{align*}
\end{proof}


\samplingestimationbound*
\begin{proof}
    First, \Cref{thm:treatment-approximation-bound} demonstrates, for any $\delta_1\in(0,1)$, with probability at least $1-\delta_1$ that the approximation bound
\begin{equation}
\label{eq:sampling-estimation-proof-part-1}
    |\theta-\hat{\theta}_{S_\ell}|\leq A(S_\ell,\delta)=\frac{2\sigma_g \sqrt{2\ell\log(2/\delta_1)}}{\left|\sum_{i=1}^\ell(x_i-\bar{x})(z_i-\bar{z})\right|}.
\end{equation}

Next, recall that the mean recommendation $\bar{z}=\rho$ for exploration probability $\rho$ in the second stage of \Cref{alg:sampling-control-treatment}. We assume \Cref{alg:sampling-control-treatment} to be initialized with parameters (see \Cref{lemma:bic-sampling-control-treatment} for details) such that its recommendations are compliant for agents of type 0. In the worst case, only type 0 agents are compliant. Therefore, \Cref{thm:approximation-bound-control-treatment-denominator} implies that, for any $\delta_2\in(0,1)$, with probability at least $1-\delta_2$ that
\begin{equation}
\label{eq:sampling-estimation-proof-part-2}
    \left|\sum_{i=1}^\ell(x_i-\bar{x})(z_i-\bar{z})\right| \geq \rho\ell\bigg(p_0(1-\rho)-\sqrt{\frac{(1-\rho)\log(1/\delta_2)}{2\ell}}\bigg).
\end{equation}
With a union bound over \Cref{eq:sampling-estimation-proof-part-1,eq:sampling-estimation-proof-part-2} while letting $\delta_1=\delta_2=\frac{\delta}{3}$ for any $\delta\in(0,1)$, we conclude: with probability at least $1-\delta$,
\begin{equation*}
    A(S_\ell,\delta)\leq\frac{2\sigma_g \sqrt{2\ell\log(3/\delta)}}{\rho\ell\bigg(p_0(1-\rho)-\sqrt{\frac{(1-\rho)\log(3/\delta)}{2\ell}}\bigg)}=\frac{2\sigma_g \sqrt{2\log(3/\delta)}}{\rho\bigg(p_0(1-\rho)\sqrt{\ell}-\sqrt{\frac{(1-\rho)\log(3/\delta)}{2}}\bigg)}
\end{equation*}
\end{proof}


\section{Missing Proofs for \Cref{sec:racing-control-treatment}}
\subsection{\Cref{alg:racing-two-types-mixed-preferences} Proofs}
\label{sec:bic-racing-type-0}
\bicracingcontroltreatmentzero*
\begin{proof} 

Just as in the proof for \Cref{lemma:bic-sampling-control-treatment-extension}, let $a^{(u)}$ and $b^{(u)}$ denote the prior preferred and unpreferred actions for agents of any type $u$, i.e. $a^{(u)}:=\1[\E_{\cP^{(u)}}[\theta]\geq0]$ and $b^{(u)}:=\1[\E_{\cP^{(u)}}[\theta]<0]$. Let $\pi$ denote the recommendation policy defined by \Cref{alg:racing-two-types-mixed-preferences}. At any round $t$ of \Cref{alg:racing-two-types-mixed-preferences}, recommendation policy $\pi_t$ has a positive probability of recommending either control or treatment, according to the prior $\cP^{(u_t)}$ for type $u_t$, i.e. $\Prob_{\pi_t,\cP^{(u_t)}}[z_t = 0]>0$ and $\Prob_{\pi_t,\cP^{(u_t)}}[z_t = 1]>0$. Thus, by \Cref{claim:bic-equiv}, the following holds:
\small
\[\mkern-8mu\left\{ (-1)^{a^{(u_t)}} \mkern-18mu \E_{\pi_t,\cP^{(u_t)}}[\theta|z_t = b^{(u_t)}]\Prob_{\pi_t,\cP^{(u_t)}}[z_t = b^{(u_t)}] \geq 0\right\}  \Rightarrow \left\{ (-1)^{b^{(u_t)}} \mkern-18mu \E_{\pi_t,\cP^{(u_t)}}[\theta|z_t = a^{(u_t)}]\Prob_{\pi_t,\cP^{(u_t)}}[z_t = a^{(u_t)}] < 0 \right\}\] \normalsize and it suffices to prove the premise $(-1)^{a^{(u_t)}}\E_{\cP^{(u_t)},\pi_t}[\theta|z_t=b^{(u_t)}]\Prob_{\cP^{(u_t)},\pi_t}[z_t=b^{(u_t)}]\geq0$ in order to prove that agent $t$ of type $u_t$ complies with recommendation $z_t$.

Recall that the sample set $S_q^{\text{BEST}}$ is made up of the best samples up until phase $q$ of \Cref{alg:racing-two-types-mixed-preferences}, i.e. the samples which produce the smallest approximation bound $A_q$. The treatment effect estimate derived from set $S_q^{\text{BEST}}$ is denoted $\hat{\theta}_q$. We define the event $\cC$ as the event that the treatment effect estimate $\hat{\theta}_q$ satisfies the approximation bound $A_q$ at every phase $q$ throughout \Cref{alg:racing-two-types-mixed-preferences}:
\begin{equation}
\label{eq:racing-stage-compliance-proof-event-C}
    \cC := \left\{ \forall q \geq 0: |\theta - \hat{\theta}_q| < A_q \right\}.
\end{equation}

By \Cref{thm:treatment-approximation-bound}, for event $\cC$, the failure probability $\Prob[\neg\cC] \leq \delta$. Furthermore, we assume here that 
\[\delta\leq 
\begin{cases}
    \frac{\tau\Prob_{\pi_t,\cP^{(u_t)}}[\theta\geq\tau]}{2(\tau\Prob_{\pi_t,\cP^{(u_t)}}[\theta\geq\tau]+1)} & \text{ if } \mu^{(u_t)}<0;\\
    \frac{\tau\Prob_{\pi_t,\cP^{(u_t)}}[\theta<-\tau]}{2(\tau\Prob_{\pi_t,\cP^{(u_t)}}[\theta<-\tau]+1)} & \text{ if } \mu^{(u_t)}\geq0.
\end{cases}\]

Therefore, since $|\theta|\leq 1$, we have: 
\begin{align*}
\label{eq:racing-bic}
    &\quad(-1)^{a^{(u_t)}}\E_{\cP^{(u_t)},\pi_t}[\theta|z_t=b^{(u_t)}]\Prob_{\cP^{(u_t)},\pi_t}[z_t=b^{(u_t)}]\\
    &= (-1)^{a^{(u_t)}}\left(\E_{\cP^{(u_t)},\pi_t}[\theta|z_t=b^{(u_t)},\cC]\Prob_{\cP^{(u_t)},\pi_t}[z_t=b^{(u_t)},\cC] + \E_{\cP^{(u_t)},\pi_t}[\theta|z_t=b^{(u_t)},\neg\cC]\Prob_{\cP^{(u_t)},\pi_t}[z_t=b^{(u_t)},\neg\cC] \right)\\
    &\geq (-1)^{a^{(u_t)}}\left(\E_{\cP^{(u_t)},\pi_t}[\theta|z_t=b^{(u_t)},\cC]\Prob_{\cP^{(u_t)},\pi_t}[z_t=b^{(u_t)},\cC] - (-1)^{a^{(u_t)}} \delta \right)\\
    &\geq (-1)^{a^{(u_t)}}\left(\E_{\cP^{(u_t)},\pi_t}[\theta|z_t=b^{(u_t)},\cC]\Prob_{\cP^{(u_t)},\pi_t}[z_t=b^{(u_t)},\cC]\right) - \frac{\tau\Prob_{\pi_t,\cP^{(u_t)}}[\theta\geq\tau]}{2\tau\Prob_{\pi_t,\cP^{(u_t)}}[\theta\geq\tau]+2}
\end{align*}
In order to lower bound the last line above, we marginalize\\ $\E_{\pi_t,\cP^{(u_t)}}[\theta|z_t=b^{(u_t)},\cC]\Prob_{\cP^{(u_t)},\pi_t}[z_t=b^{(u_t)},\cC]$ based on four possible ranges which $\theta$ lies on:
\begin{equation} 
    \begin{split}
    &\quad\E_{\cP^{(u_t)},\pi_t}[\theta|z_t=b^{(u_t)},\cC]\Prob_{\cP^{(u_t)},\pi_t}[z_t=b^{(u_t)},\cC]\\
    &= \E_{\pi_t,\cP^{(u_t)}}[\theta|z_t=b^{(u_t)}, \cC, (-1)^{a^{(u_t)}}\theta\geq \tau]\Prob_{\pi_t,\cP^{(u_t)}}[z_t=b^{(u_t)}, \cC, (-1)^{a^{(u_t)}}\theta\geq \tau] \\
    &\ + \E_{\pi_t,\cP^{(u_t)}}[\theta|z_t=b^{(u_t)}, \cC, 0 \leq (-1)^{a^{(u_t)}}\theta < \tau]\Prob_{\pi_t,\cP^{(u_t)}}[z_t=b^{(u_t)}, \cC, 0 \leq (-1)^{a^{(u_t)}}\theta < \tau] \\
    &\ + \E_{\pi_t,\cP^{(u_t)}}[\theta|z_t=b^{(u_t)}, \cC, -2A_q < (-1)^{a^{(u_t)}}\theta < 0]\Prob_{\pi_t,\cP^{(u_t)}}[z_t=b^{(u_t)}, \cC, -2A_q < (-1)^{a^{(u_t)}}\theta < 0] \\
    &\ + \E_{\pi_t,\cP^{(u_t)}}[\theta|z_t=b^{(u_t)}, \cC, (-1)^{a^{(u_t)}}\theta \leq -2A_q]\Prob_{\pi_t,\cP^{(u_t)}}[z_t=b^{(u_t)}, \cC, (-1)^{a^{(u_t)}}\theta\leq -2A_q] \label{eq:racing-bic-cases}
    \end{split}
\end{equation}

Because $A_q$ is the smallest approximation bound derived from samples collected over any phase $q$ of \Cref{alg:racing-two-types-mixed-preferences} (including the initial sample set $S_0$), the following holds:
\begin{align*}
    2A_q &\leq 2A(S_0,\delta)\\
    &\leq \frac{\tau\Prob_{\pi_t,\cP^{(u_t)}}[(-1)^{a^{(u_t)}}\theta\geq\tau]}{2}\tag{by assumption $A(S_0,\delta)\leq\tau\Prob_{\pi_t,\cP^{(u_t)}}[(-1)^{a^{(u_t)}}\theta\geq\tau]/4$}\\
    &\leq \tau
\end{align*}

Conditional on $\cC$, $|\theta-\hat{\theta}_q|<A_q$. Thus, we may reduce two of the terms in \Cref{eq:racing-bic-cases} above, based on the structure of \Cref{alg:racing-two-types-mixed-preferences}. First, from the first term in \Cref{eq:racing-bic-cases}, note that

\[(-1)^{a^{(u_t)}}\theta\geq\tau\geq2A_q \mkern20mu \Rightarrow \mkern20mu (-1)^{a^{(u_t)}}\hat{\theta}_q\geq\tau-A_q\geq A_q,\]
which invokes the stopping criterion for the while loop in \Cref{alg:racing-two-types-mixed-preferences}. Thus, type $u_t$'s preferred action $a^{(u_t)}$ must have been eliminated from the race before phase $q=1$ and the unpreferred action $b^{(u_t)}$ is recommended almost surely throughout \Cref{alg:racing-two-types-mixed-preferences}, i.e. \[\Prob_{\pi_t,\cP^{(u_t)}}[z_t=b^{(u_t)}, \cC, (-1)^{a^{(u_t)}}\theta \geq \tau] = \Prob_{\pi_t,\cP^{(u_t)}}[\cC, (-1)^{a^{(u_t)}}\theta\geq \tau].\]

Second, from the last term in \Cref{eq:racing-bic-cases}, note that 
\[(-1)^{a^{(u_t)}}\theta\leq-2A_q \mkern20mu \Rightarrow \mkern20mu (-1)^{a^{(u_t)}}\hat{\theta}_q\leq-A_q\] by phase $q=1$ and the unpreferred action $b^{(u_t)}$ is recommended almost never, i.e. \[\Prob_{\pi_t,\cP^{(u_t)}}[z_t=b^{(u_t)}, \cC, (-1)^{a^{(u_t)}}\theta<-2A_q] = 0.\] 

Substituting these probabilities back into \Cref{eq:racing-bic-cases}, we proceed:

\begin{align*} 
    & \mkern-25mu (-1)^{a^{(u_t)}} \mkern-18mu \E_{\cP^{(u_t)},\pi_t}[\theta|z_t=b^{(u_t)},\cC]\Prob_{\cP^{(u_t)},\pi_t}[z_t=b^{(u_t)},\cC]\\
    =& (-1)^{a^{(u_t)}}\left( \E_{\pi_t,\cP^{(u_t)}}[\theta|\cC, (-1)^{a^{(u_t)}}\theta\geq \tau]\Prob_{\pi_t,\cP^{(u_t)}}[\cC, (-1)^{a^{(u_t)}}\theta\geq \tau] \right.\\
    &\ + \E_{\pi_t,\cP^{(u_t)}}[\theta|z_t=b^{(u_t)}, \cC, 0 \leq (-1)^{a^{(u_t)}}\theta < \tau]\Prob_{\pi_t,\cP^{(u_t)}}[z_t=b^{(u_t)}, \cC, 0 \leq (-1)^{a^{(u_t)}}\theta < \tau] \\
    &\ + \left. \E_{\pi_t,\cP^{(u_t)}}[\theta|z_t=b^{(u_t)}, \cC, -2A_q < (-1)^{a^{(u_t)}}\theta < 0]\Prob_{\pi_t,\cP^{(u_t)}}[z_t=b^{(u_t)}, \cC, -2A_q < (-1)^{a^{(u_t)}}\theta < 0] \right)\\
    \geq&\ (-1)^{a^{(u_t)}}\left( (-1)^{a^{(u_t)}} \tau \mkern-12mu  \Prob_{\pi_t,\cP^{(u_t)}}[\cC, (-1)^{a^{(u_t)}}\theta\geq \tau] + 0\cdot  \mkern-16mu \Prob_{\pi_t,\cP^{(u_t)}}[z_t=b^{(u_t)}, \cC, 0 \leq (-1)^{a^{(u_t)}}\theta < \tau]\right.\\
    & \hspace{4cm}\left. -(-1)^{a^{(u_t)}}2A_q \Prob_{\pi_t,\cP^{(u_t)}}[z_t=b^{(u_t)}, \cC, -2A_q < (-1)^{a^{(u_t)}}\theta < 0]\right)\\
    \geq&\ (-1)^{a^{(u_t)}}\left( (-1)^{a^{(u_t)}}\tau \mkern-12mu  \Prob_{\pi_t,\cP^{(u_t)}}[\cC, (-1)^{a^{(u_t)}}\theta\geq \tau] -(-1)^{a^{(u_t)}}2A_q\right)\\
    \geq&\ \tau \mkern-12mu  \Prob_{\pi_t,\cP^{(u_t)}}[\cC, (-1)^{a^{(u_t)}}\theta\geq \tau] -\frac{\tau\Prob_{\pi_t,\cP^{(u_t)}}[(-1)^{a^{(u_t)}}\theta\geq\tau]}{2}\\
    \geq&\ \tau \mkern-12mu  \Prob_{\pi_t,\cP^{(u_t)}}[\cC| (-1)^{a^{(u_t)}}\theta\geq \tau] \Prob_{\pi_t,\cP^{(u_t)}}[(-1)^{a^{(u_t)}}\theta\geq \tau] -\frac{\tau\Prob_{\pi_t,\cP^{(u_t)}}[(-1)^{a^{(u_t)}}\theta\geq\tau]}{2}\\
    \geq&\ (1-\delta) \tau \mkern-12mu  \Prob_{\pi_t,\cP^{(u_t)}}[(-1)^{a^{(u_t)}}\theta\geq \tau] -\frac{\tau\Prob_{\pi_t,\cP^{(u_t)}}[(-1)^{a^{(u_t)}}\theta\geq\tau]}{2}\\
    =&\ \left(\frac{1}{2}-\delta\right) \tau \mkern-12mu  \Prob_{\pi_t,\cP^{(u_t)}}[(-1)^{a^{(u_t)}}\theta\geq \tau]\\
    \geq & \left(\frac{1}{2} - \frac{\tau \Prob_{\pi_t,\cP^{(u_t)}}[(-1)^{a^{(u_t)}}\theta\geq \tau]}{2\tau \Prob_{\pi_t,\cP^{(u_t)}}[(-1)^{a^{(u_t)}}\theta\geq \tau] + 2} \right) \tau \Prob_{\pi_t,\cP^{(u_t)}}[(-1)^{a^{(u_t)}}\theta\geq \tau]\\
    = & \mkern10mu \frac{\tau \Prob_{\pi_t,\cP^{(u_t)}}[(-1)^{a^{(u_t)}}\theta\geq \tau]}{2\tau \Prob_{\pi_t,\cP^{(u_t)}}[(-1)^{a^{(u_t)}}\theta\geq \tau] + 2}\\
\end{align*}

Putting everything together, we get that
\begin{align*}
    &(-1)^{a^{(u_t)}} \mkern-18mu \E_{\cP^{(u_t)},\pi_t}[\theta|z_t=b^{(u_t)}]\Prob_{\cP^{(u_t)},\pi_t}[z_t=b^{(u_t)}]\\
    &\geq (-1)^{a^{(u_t)}} \mkern-4mu \left( \E_{\cP^{(u_t)},\pi_t}[\theta|z_t=b^{(u_t)},\cC]\Prob_{\cP^{(u_t)},\pi_t}[z_t=b^{(u_t)},\cC]+\E_{\cP^{(u_t)},\pi_t}[\theta|z_t=b^{(u_t)},\neg\cC]\Prob_{\cP^{(u_t)},\pi_t}[z_t=b^{(u_t)},\neg\cC]\right)\\ 
    &\geq \frac{\tau \Prob_{\pi_t,\cP^{(u_t)}}[(-1)^{a^{(u_t)}}\theta\geq \tau]}{2\tau \Prob_{\pi_t,\cP^{(u_t)}}[(-1)^{a^{(u_t)}}\theta\geq \tau] + 2} - \frac{\tau \Prob_{\pi_t,\cP^{(u_t)}}[(-1)^{a^{(u_t)}}\theta\geq \tau]}{2\tau \Prob_{\pi_t,\cP^{(u_t)}}[(-1)^{a^{(u_t)}}\theta\geq \tau] + 2}\\
    &= 0
\end{align*}

Therefore, so long as \[A(S_0,\delta)\leq\tau\Prob_{\pi_t,\cP^{(u_t)}}[(-1)^{a^{(u_t)}}\theta\geq\tau]/4 \hspace{.45cm} \text{ and } \hspace{.45cm} \delta<\tau\Prob_{\pi_t,\cP^{(u_t)}}[(-1)^{a^{(u_t)}}\theta\geq\tau]/2,\] any agent of type $u_t$ will comply with recommendations from \Cref{alg:racing-two-types-mixed-preferences}.
\end{proof}


\subsection{\Cref{lemma:racing-full-compliance-control-treatment} and \Cref{thm:racing-estimation-bound-second}: Full Compliance and Subsequent Estimation Bound}

\begin{restatable}[\Cref{alg:racing-two-types-mixed-preferences} Full Compliance]{lemma}{racingfullcompliancecontroltreatment}\label{lemma:racing-full-compliance-control-treatment}
        Suppose that some fraction $p_c>0$ of agents is compliant from the beginning of \Cref{alg:racing-two-types-mixed-preferences} and assume that $p_c<1$. Of all types $u$ which were not compliant from the beginning, let type $u^*$ agents be the most resistant to compliance. Suppose that phase $q$ satisfies one of the following bounds (depending on whether type $u^*$ agents prefer control or treatment):
        \[q \geq \begin{cases}
             \left(\frac{1}{2hp_c}\bigg(\frac{(32\sigma_g\sqrt{2\log(5/\delta)}}{\tau\Prob_{\cP^{(u^*)}}[\theta>\tau]}+\sqrt{50\log(5/\delta)}\bigg)\right)^2 & \text{ if } \E_{\cP^{(u^*)}}[\theta]<0\\
             \left(\frac{1}{2hp_c}\bigg(\frac{(32\sigma_g\sqrt{2\log(5/\delta)}}{\tau\Prob_{\cP^{(u^*)}}[\theta<-\tau]}+\sqrt{50\log(5/\delta)}\bigg)\right)^2 & \text{ if } \E_{\cP^{(u^*)}}[\theta]\geq0,
        \end{cases}\]
        for some $\tau\in(0,1)$ and any $\delta\in(0,1)$.
        Then, with probability at least $1-\delta$, for any phase $q$ greater or equal to the following lower bound all agents will comply with recommendations from \Cref{alg:racing-two-types-mixed-preferences}.
    \end{restatable}

\begin{proof}
    First, recall that the set $S_q$ is made up of the input samples $S_0$ plus samples collected following \Cref{alg:racing-two-types-mixed-preferences} over all phases up to $q$. Let $S_{q-0}=(x_i,y_i,z_i)_{i=1}^{2hq}$ denote $S_q$ sans $S_0$ (i.e. just samples collected following \Cref{alg:racing-two-types-mixed-preferences} up to phase $q$). Note that for any $\delta\in(0,1)$, the approximation bound $A_q \leq A(S_{q-0},\delta)$.
    
    We want to prove that type $u^*$ is compliant by and beyond phase $q$. By \Cref{lemma:bic-racing-control-treatment-0}, it suffices to prove that the approximation bound $A_q$ satisfies the following upper bound with probability at least $1-\delta$:\footnotemark
    \[\begin{cases}
        A_q \leq \tau\Prob_{\cP^{(u^*)}}[\theta>\tau]/4 & \text{ if } \E_{\cP^{(u^*)}}[\theta]<0;\\
        A_q \leq \tau\Prob_{\cP^{(u^*)}}[\theta<-\tau]/4 & \text{ if } \E_{\cP^{(u^*)}}[\theta]\geq 0,
    \end{cases}\]
    for any $\delta\in(0,1)$ and some $\tau\in(0,1)$.
    
    \footnotetext{\Cref{lemma:bic-racing-control-treatment-0} doesn't exactly state this: it states that any type $u$ will be compliant if the input samples $S_0$ satisfy the above bounds. Yet, we can simply imagine that phase $q$ is . Proving compliance starting from any phase $q>0$ is just the same as proving compliance from phase 0. Intuitively, you can imagine we simply run \Cref{alg:racing-two-types-mixed-preferences} starting at phase $q$ initialized with the samples collected up until phase $q$.}
    
    In order to prove this, recall that each phase $q$ of \Cref{alg:racing-two-types-mixed-preferences} is $2hq$ rounds long and the mean recommendation $\bar{z}=\frac{1}{2}$. By assumption, $p_c$ proportion of agents are compliant. Thus, by \Cref{thm:approximation-bound-control-treatment-denominator}, with probability $1-\delta$ for any $\delta\in(0,1)$, the set $S_{q-0}$ satisfies:
    \[A(S_{q-0},\delta)\leq\frac{8\sigma_g\sqrt{2\log(3/\delta)}}{p_c\sqrt{|S_{q-0}|}-\sqrt{\log(3/\delta)}}\]
    
    By assumption, $q$ satisfies the following lower bound for some $\tau\in(0,1)$:
    \[q \geq \begin{cases}
             \left(\frac{1}{2hp_c}\bigg(\frac{(32\sigma_g\sqrt{2\log(5/\delta)}}{\tau\Prob_{\cP^{(u^*)}}[\theta>\tau]}+\sqrt{50\log(5/\delta)}\bigg)\right)^2 & \text{ if } \E_{\cP^{(u^*)}}[\theta]<0\\
             \left(\frac{1}{2hp_c}\bigg(\frac{(32\sigma_g\sqrt{2\log(5/\delta)}}{\tau\Prob_{\cP^{(u^*)}}[\theta<-\tau]}+\sqrt{50\log(5/\delta)}\bigg)\right)^2 & \text{ if } \E_{\cP^{(u^*)}}[\theta]\geq0,
        \end{cases}\]
    
    Substituting these lower bound values for $q$ in $A(S_{q-0},\delta)$, we get that the approximation bound $A_q$ satisfies the following inequalities (since $A_q \leq A(S_{q-0},\delta)$):
    
    \[A_q \leq A(S_{q-0},\delta) \leq \begin{cases}
        \tau\Prob_{\cP^{(u^*)}}[\theta>\tau]/4 & \text{ if } \E_{\cP^{(u^*)}}[\theta]<0;\\
        \tau\Prob_{\cP^{(u^*)}}[\theta<-\tau]/4 & \text{ if } \E_{\cP^{(u^*)}}[\theta]\geq 0.
    \end{cases}\]
\end{proof}

 
    Finally, after \Cref{alg:racing-two-types-mixed-preferences} has become compliant for both types of agents, we achieve the following accuracy guarantee for the final treatment estimate $\hat{\theta}_S$.
\begin{restatable}[Treatment Effect Confidence Interval from \Cref{alg:racing-two-types-mixed-preferences} with Full Compliance]{theorem}{racingestimationboundsecond}\label{thm:racing-estimation-bound-second}
    Suppose sample set $S=(x_i,y_i,z_i)_{i=1}^{|S|}$ is collected from \Cref{alg:racing-two-types-mixed-preferences} during $|S|$ rounds when all agents comply. We form estimate $\hat{\theta}_S$ of the treatment effect $\theta$. For any $\delta\in(0,1)$, with probability at least $1-\delta$,
    \[\left|\hat{\theta}_S-\theta\right| \leq 8\sigma_g\sqrt{\frac{2\log(2/\delta)}{|S|}}\]
\end{restatable}

\begin{proof}
    We assume \Cref{alg:racing-two-types-mixed-preferences} is initialized and allowed to run long enough such that both types 0 and 1 become compliant at some point. From samples $S=(x_i,y_i,z_i)_{i=1}^{|S|}$ collected during these rounds (from $i$ to $|S|$)), we form an estimate $\hat{\theta}_S$ of the treatment effect $\theta$. By \Cref{thm:treatment-approximation-bound}, this estimate satisfies the following bound with probability at least $1-\delta$ for any $\delta\in(0,1)$:
    \begin{equation}
    \label{eq:racing-second-estimation-proof-part-1}
        |\hat{\theta}_S-\theta|\leq A(S,\delta)=\frac{2\sigma_g \sqrt{2|S|\log(2/\delta)}}{\left|\sum_{i=1}^{|S|}(x_i-\bar{x})(z_i-\bar{z})\right|}.
    \end{equation}
    Recall that $\bar{z}=\frac{1}{2}$ throughout \Cref{alg:racing-two-types-mixed-preferences}.
    Then, by \Cref{thm:approximation-bound-control-treatment-denominator}, the denominator of the bound in \Cref{eq:racing-second-estimation-proof-part-1} above satisfies the following bound:
    \begin{equation}
    \label{eq:racing-second-estimation-proof-part-2}
        \left|\sum_{i=1}^{|S|}(x_i-\bar{x})(z_i-\bar{z})\right| \geq \frac{|S|}{4}.
    \end{equation}
    Therefore, by \Cref{eq:racing-second-estimation-proof-part-1,eq:racing-second-estimation-proof-part-2}, the confidence interval $|\hat{\theta}_S-\theta|$ satisfies the following upper bound:
    \[|\hat{\theta}_S-\theta|\leq 8\sigma_g \sqrt{\frac{2\log(2/\delta)}{|S|}}\]
\end{proof}


\subsection{Proof of \Cref{lemma:racing-compliance-sampling-ell-control-treatment}}
\label{sec:racing-compliance-sampling-ell-control-treatment-proof}
\racingcompliancesamplingellcontroltreatment*
\begin{proof}
    By assumption, \Cref{alg:sampling-control-treatment} is initialized so that agents of type 0 comply and we collect $S_\ell$ samples from the second stage. Then, for any $\delta\in(0,1)$ and some $\tau\in(0,1)$, approximation bound $A(S_\ell,\delta)$ satisfies:
    \begin{align*}
        A(S_\ell,\delta)
        &\leq \frac{2\sigma_g \sqrt{2\log(3/\delta)}}{\rho\bigg(p_0(1-\rho)\sqrt{\ell}-\sqrt{\frac{(1-\rho)\log(3/\delta)}{2}}\bigg)}\tag{by \Cref{thm:sampling-estimation-bound}}\\
        &\leq \frac{2\sigma_g \sqrt{2\log(3/\delta)}}{\rho\bigg(p_0(1-\rho)\left(\frac{8\sigma_g\sqrt{2\log(3/\delta)}}{p_0\rho(1-\rho)\tau\Prob_{\cP^{(0)}}[\theta>\tau]}+\frac{\sqrt{(1-\rho)\log(3/\delta)}}{2p_0(1-\rho)}\right)-\sqrt{\frac{(1-\rho)\log(3/\delta)}{2}}\bigg)}
        \tag{by \Cref{eq:racing-compliance-proof-ell-condition}}\\
        &\leq \frac{\tau\Prob_{\cP^{(0)}}[\theta>\tau]}{4}
    \end{align*}
    
    Thus, by \Cref{lemma:bic-racing-control-treatment-0}, if we let the samples $S_\ell$ collected from the second stage of \Cref{alg:sampling-control-treatment} be the input samples $S_0$ in \Cref{alg:racing-two-types-mixed-preferences}, i.e.  $S_0=S_\ell$, then that agents of type 0 will comply with recommendations of \Cref{alg:racing-two-types-mixed-preferences}.

\end{proof}

\subsubsection{Racing Stage First Part Estimation Bound}
\label{sec:racing-estimation-bound-first}
\racingestimationboundfirst*
\begin{proof}
    First, \Cref{thm:treatment-approximation-bound} demonstrates, for any $\delta_1\in(0,1)$, with probability at least $1-\delta_1$ that the approximation bound
\begin{equation}
\label{eq:racing-estimation-proof-part-1}
    \left|\hat{\theta}_S-\theta\right|\leq A(S,\delta)=\frac{2\sigma_g \sqrt{2|S|\log(2/\delta_1)}}{\left|\sum_{i=1}^\ell(x_i-\bar{x})(z_i-\bar{z})\right|}.
\end{equation}

Next, recall that the mean recommendation $\bar{z}=\frac{1}{2}$ throughout \Cref{alg:racing-two-types-mixed-preferences}. We assume \Cref{alg:racing-two-types-mixed-preferences} to be initialized with parameters such that its recommendations are compliant for agents of type 0. In the worst case, only type 0 agents are compliant. Therefore, \Cref{thm:approximation-bound-control-treatment-denominator} implies that, for any $\delta_2\in(0,1)$, with probability at least $1-\delta_2$ that
\begin{equation}
\label{eq:racing-estimation-proof-part-2}
    \left|\sum_{i=1}^{|S|}(x_i-\bar{x})(z_i-\bar{z})\right| \geq \frac{|S|}{4}\left(p_0-\sqrt{\frac{\log(1/\delta_2)}{|S|}}\right).
\end{equation}
With a union bound over \Cref{eq:racing-estimation-proof-part-1,eq:racing-estimation-proof-part-2} while letting $\delta_1=\delta_2=\frac{\delta}{3}$ for any $\delta\in(0,1)$, we conclude: with probability at least $1-\delta$,
\begin{equation*}
    A(S,\delta)\leq\frac{8\sigma_g\sqrt{2\log(3/\delta)}}{p_0\sqrt{|S|}-\sqrt{\log(3/\delta)}}
\end{equation*}
\end{proof}

\section{Missing Regret Proofs for \Cref{sec:combined-control-treatment}}
\subsection{Pseudo-regret Proof}
\label{sec:expost-proof}
\expostregret*

\begin{proof}
Recall that the clean event $\cC$, as defined in the proof of \Cref{lemma:bic-racing-control-treatment-0}, entails that the approximation bound over all rounds. If event $\cC$ fails (i.e. $\neg\cC$ holds), then the pseudo-regret may only be bounded by the maximum possible value, which is at most $T|\theta|$.

Assume now that $\cC$ holds for every round $L_2$ in \Cref{alg:racing-two-types-mixed-preferences}. Then, the absolute value of the treatment $\abs{\theta} \leq |\hat{\theta}_{S_{L_2}}| + A(S_{L_2},\delta)$ where $A(S_{L_2},\delta)$ is the approximation bound based on the samples $S_{L_2}$ collected from $L_2$ rounds of \Cref{alg:racing-two-types-mixed-preferences}. Before the stopping criterion of \Cref{alg:racing-two-types-mixed-preferences} is invoked, we also have $|\hat{\theta}_{S_{L_2}}| \leq A(S_{L_2},\delta)$. Hence, the treatment effect absolute value satisfies the following inequalities:
\begin{align*}
     &\abs{\theta} \leq 2A(S_{L_2},\delta) \leq \frac{16\sigma_g\sqrt{2\log(5T/\delta)}}{p_{c_2}\sqrt{L_2}-\sqrt{50\log(5T/\delta)}}
\end{align*}

Assuming that $L_2\geq\frac{200\log(5T/\delta)}{p_{c_2}^2}$, then $p_{c_2}\sqrt{L_2}-\sqrt{50\log(5T/\delta)} \geq \frac{p_{c_2}\sqrt{L_2}}{2}$ and, to carry on:
\begin{align*}
    &\abs{\theta} \leq \frac{32\sigma_g\sqrt{2\log(5T/\delta)}}{p_{c_2}\sqrt{L_2}}\\
    \Rightarrow \ & L_2 \leq \frac{2048\sigma_g^2\log(5T/\delta)}{p_{c_2}^2\abs{\theta}^2}
\end{align*}

Therefore, a winner is declared ---i.e. the treatment effect is definitively either positive or not--- by the following round of \Cref{alg:racing-two-types-mixed-preferences}:
\begin{align*}
    L_2^* = \frac{2048\sigma_g^2\log(5T/\delta)}{p_{c_2}^2\abs{\theta}^2}
\end{align*}

After this, the winner is recommended for the remainder of the rounds and (because we assume event $\cC$ holds) no regret is accumulated for the remaining rounds (until time horizon $T$).

Note that, by the length the assumption that $L_2\geq\frac{200\log(5T/\delta)}{p_{c_2}^2}$ holds trivially for $L_2^*$ if $|\theta|\leq\frac{16\sigma_g}{5}$. Otherwise, we simply assume that $L_2\geq\frac{200\log(5T/\delta)}{p_{c_2}^2}$ holds. We may incorporate this bound into a necessary lower bound on the time horizon $T$, such that we assume $T\geq L_1 + L_2 \geq L_1 + \frac{200\log(5T/\delta)}{p_{c_2}^2}$.

Now, we demonstrate the amount of regret accumulated during these $n^*$ rounds of \Cref{alg:racing-two-types-mixed-preferences} before a winner is declared. During each phase \Cref{alg:racing-two-types-mixed-preferences}, each control and treatment each get recommended $n^*/2$ times.

Furthermore, recall that $p_{c_2}$ fraction of agents are compliant throughout all rounds of \Cref{alg:racing-two-types-mixed-preferences}. Without loss of generality, assume that these agents initially prefer control and the rest of the $1-p_{c_2}$ fraction of agents prefer treatment (and do not comply). Then, if the treatment $\theta<0$, then in expectation over the randomness of the arrival of agents, the regret is on average $(1-p_{c_2}/2)\abs{\theta}$. Then, the total accumulated regret throughout \Cref{alg:racing-two-types-mixed-preferences} in policy $\pi_c$ is given as such:
\begin{equation}
\label{eq:regret-arm-0}
    R_2(T) \leq \frac{2048(1-p_{c_2}/2)\sigma_g^2\log(5T/\delta)}{p_{c_2}^2\abs{\theta}}
\end{equation}
On the other hand, if treatment $\theta\geq0$, then on average, the regret for each phase is $p_{c_2}\abs{\theta}/2$, then the total accumulated regret for \Cref{alg:racing-two-types-mixed-preferences} is given:
\begin{align}
\label{eq:regret-arm-1}
    R_2(T) &\leq \frac{1024\sigma_g^2\log(5T/\delta)}{p_{c_2}\abs{\theta}}
\end{align}
Observe that the pseudo-regret for each round $t$ of the policy $\pi_c$ over the entire $T$ rounds is at most that of \Cref{alg:racing-two-types-mixed-preferences} plus $\abs{\theta}$ per round of \Cref{alg:sampling-control-treatment}. Recall that, by policy $\pi_c$, there are $L_1$ total rounds of \Cref{alg:sampling-control-treatment}. Alternatively, we can also upper bound the total pseudo-regret by $\abs{\theta}$ per each round. Therefore, the total accumulated pseudo-regret for policy $\pi_c$ once we reach the time horizon $T$ is bounded as follows.

If the treatment effect $\theta\geq0$, then the total regret of policy $\pi_c$ satisfies the following bound with probability at least $1-\delta$, for any $\delta\in(0,1)$ which satisfies the condition that $L_1\geq\frac{2\sqrt{\log(5T/\delta)}}{p_{c_2}}$:
\begin{align*}
    R(T) &\leq \min \Bigg\{L_1\rho\abs{\theta} + \frac{1024\sigma_g^2\log(5T/\delta)}{p_{c_2}\abs{\theta}}, T\abs{\theta} \Bigg\}
\end{align*}

We can solve for $\abs{\theta}$ in terms of $T$ at the point when $T|\theta|$ is the better regret:
\begin{align*}
    T\abs{\theta}
    &=\frac{1024\sigma_g^2\log(5T/\delta)}{p_{c_2}\abs{\theta}}\\
    \Rightarrow \abs{\theta}^2&=\frac{1024\sigma_g^2\log(5T/\delta)}{p_{c_2}T}\\
    \Rightarrow \abs{\theta}&=32\sigma_g\sqrt{\frac{\log(5T/\delta)}{p_{c_2}T}}
\end{align*}
Substituting this expression for $\abs{\theta}$ back into our expression for the total pseudo-regret, we get the following:
\begin{align*}
    R(T) &\leq \min \Bigg\{L_1\rho\abs{\theta} + \frac{1024\sigma_g^2\log(5T/\delta)}{32\sigma_gp_{c_2}\sqrt{\frac{\log(5T/\delta)}{p_{c_2}T}}}, 32\sigma_gT\sqrt{\frac{\log(5T/\delta)}{p_{c_2}T}} \Bigg\}\\
    &= \min \Bigg\{L_1\rho\abs{\theta} + 32\sigma_g\sqrt{\frac{T\log(5T/\delta)}{p_{c_2}}}, 32\sigma_g\sqrt{\frac{T\log(5T/\delta)}{p_{c_2}}} \Bigg\}\\
    &\leq L_1\rho + O(\sqrt{T\log(T/\delta)})
\end{align*}
The $L_1\rho$ regret for \Cref{alg:sampling-control-treatment} in the last line above is given because $|\theta|\leq1$.

Following a similar analysis, if the treatment effect $\theta<0$, then the total pseudo-regret accumulated following policy $\pi_c$ satisfies the following bound with probability at least $1-\delta$ for any $\delta\in(0,1)$ which satisfies the condition that $L_1\geq\frac{2\sqrt{\log(5T/\delta)}}{p_{c_2}}$: 
\begin{align}
    R(T) &\leq \min \Bigg \{L_1\abs{\theta} + \frac{2048(1-p_{c_2}/2)\sigma_g^2\log(5T/\delta)}{p_{c_2}^2\abs{\theta}}, T\abs{\theta} \Bigg \}\\
    &\leq L_1 + O(\sqrt{T\log(T/\delta)})
\end{align}

Note that (as stated above) this regret holds only if the time horizon $T$ is sufficiently large such that $T\geq L_1 + \frac{200\log(5T/\delta)}{p_{c_2}^2}$.
\end{proof}

\subsection{Regret Proof}
\label{sec:expected-regret-proof}
\expectedregret*

\begin{proof}
We can set parameters $\delta,\ell_0,\ell_1,\ell,$ and $\rho$ in terms of the time horizon $T$, in order to both guarantee compliance throughout policy $\pi_c$ and to obtain sublinear (expected) regret bound relative to $T$.

First, to guarantee sublinear expected regret, we must guarantee that $\delta = 1/T^2$. To meet our compliance conditions for \Cref{alg:racing-two-types-mixed-preferences}, we must set \[\delta \leq \frac{\tau \Prob_{\cP^{(u)}}[|\theta| \geq \tau]}{2\left( \tau \Prob_{\cP^{(u)}}[|\theta| \geq \tau] + 1 \right)},\]
for some $\tau$. These may be expressed as conditions on the time horizon $T$: for any $\delta\in(0,1)$ which satisfies the above compliance conditions, we set $T$ sufficiently large to satisfy the following condition:
\begin{equation}
    \label{eq:regret-T-condition-1}
    T\geq\frac{1}{\sqrt{\delta}}\geq\sqrt{\frac{2\left( \tau \Prob_{\cP^{(u)}}[|\theta| \geq \tau] + 1 \right)}{\tau \Prob_{\cP^{(u)}}[|\theta| \geq \tau]}}
\end{equation}

Second, recall that $p_0$ and $p_1$ denote the fractions in the population of agents who are never-takers and always-takers, respectively. Furthermore, recall that $p_{c_1}$ and $p_{c_2}$ denote the fractions of agents who comply with \Cref{alg:sampling-control-treatment} and \Cref{alg:racing-two-types-mixed-preferences}, respectively. Assume that the length of the first stage of \Cref{alg:sampling-control-treatment} is non-zero and the exploration probability $\rho$ is set to be small enough in order to guarantee compliance throughout \Cref{alg:sampling-control-treatment}. The length $L_1$ of \Cref{alg:sampling-control-treatment} must be sufficiently large so that $p_c$ fraction of agents comply in \Cref{alg:racing-two-types-mixed-preferences}, as well. However, in order to guarantee sublinear regret, we also need that 
\begin{equation}
    \label{eq:regret-T-condition-2}
    T \geq L_1^2 = (2\max(\ell_0/p_0,\ell_1/p_1)+\ell)^2
\end{equation}

Recall that the clean event $\cC$, as defined in the proof of \Cref{lemma:bic-racing-control-treatment-0}, entails that the approximation bound over all rounds. This event $\cC$ holds with probability at least $1 - \delta$ for any $\delta \in (0,1)$. Conditional on the failure event $\neg\cC$, policy $\pi_c$ accumulates at most linear pseudo-regret in terms of $T$, i.e. $T|\theta|$. Thus, in expectation it accumulates at most $T|\theta|\delta$ regret.\\

Then, with the above assumptions on $T$ in mind, the expected regret of policy $\pi_c$ is:
\begin{align*}
    \E_{\cP^{(u)}}[R(T)] &= \E[R(T)|\neg \cC]\Prob_{\cP^{(u)}}[\neg \cC] + \E[R(T)|\cC]\Prob_{\cP^{(u)}}[\cC] \\
    &\leq T\delta + \left(L_1+ O\left(\sqrt{T\log(T/\delta)}\right)\right)\\
    &= \frac{1}{T} + \left(\sqrt{T} + O\left( \sqrt{T\log(T^3)}\right)\right)\\
    &= \frac{1}{T} + O\left(\sqrt{T\log(T)}\right)\\
    &= O\left(\sqrt{T\log(T)} \right)
\end{align*}

Therefore, assuming that all hyperparameters $\delta,\ell_0,\ell_1,\ell,$ and $\rho$ are set to incentivize compliance of some nonzero proportion of agents throughout $\pi_c$ and assuming that $T$ is sufficiently large so as to satisfy both \Cref{eq:regret-T-condition-1,eq:regret-T-condition-2} above, policy $\pi_c$ achieves sublinear regret.
\end{proof}

\lsdelete{
\subsection{Type-specific Regret Proof}
\label{sec:type-specific-regret-proof}
\typespecifictwoarms*

\begin{proof}
The regret of our algorithm, as given by Lemma~\eqref{lemma:expost-regret} is 
\begin{equation}
\label{eq:ex-post-2-arm}
    R(T) \leq \ell_1 \rho + O(\sqrt{T\log T}) 
\end{equation}
where $\ell_1$ is the number of treatment samples collected in the sampling stage, $\rho$ is the number of phases in the sampling stage.\\
Let $L_1$ be the number of samples needed to make agents of type 1 follow our recommendation for control. Half of the population is agents of type 0 and the other half is agents of type 1. We can derive type-specific regret for each type of agents using our algorithm conditioned on which arm is the best arm overall.
\begin{enumerate}
    \item If control is the best arm overall:\\
    Type 0 regret: Since our algorithms guarantee that agents of type 0 will always follow our recommendations, the regret for agents of type 0 is derived similar to the pseudo-regret in Lemma \ref{lemma:expost-regret}. 
    \begin{equation}
        R_0 \leq 0.5\ell_1\rho + O(\sqrt{T\log T})
    \end{equation}
    
    Type 1 regret: Since agents of type 1 only follow our recommendation and take control in the second part of the racing stage, the regret of type 1 is then the regret accumulated by all agents of type 1 in the sampling stage, the first part of the racing stage and the second part of the racing stage. 
    \begin{align}
        R_1 &\leq 0.5(\ell_1 \rho + L_1 - \ell_1)\abs{\theta} + O(\sqrt{T\log T })\\
        &\leq 0.5\ell_1 \rho + L_1 + O(\sqrt{T\log T})
    \end{align}
    \item If treatment is the best arm overall:\\
    Type 0 regret: Similar to the case where control is the best arm overall, the regret for agents of type 0 can be derived from the ex-post regret in Lemma \ref{lemma:expost-regret}
    \begin{align}
        R_0 &\leq 0.5\ell_1 \abs{\theta} + O(\sqrt{T\log T})
    \end{align}
    Type 1 regret: Since agents of type 1 only follow our recommendation and take control in the second part of the racing stage, the regret for agents of type 1 is then the regret accumulated in the second part of the racing stage.
    \begin{align}
        R_1 \leq O(\sqrt{T\log T})
    \end{align}
\end{enumerate}
\end{proof}
}

\newpage
\section{Missing Proofs and Materials for \Cref{sec:many-arms}}
\label{sec:many-arms-appendix}

\subsection{Model}
\label{sec:general-model}
We now consider a general setting for the sequential game between a social planner and a sequence of agents over $T$ rounds, as first mentioned in \Cref{sec:model}. In this setting, there are $k$ treatments of interest, each with unknown treatment effect. In each round $t$, a new agent indexed by $t$ arrives with their private type $u_t$ drawn independently from a distribution $\cU$ over the set of all private types $U$. Each agent $t$ has $k$ actions to choose from, numbered 1 to $k$. Let $x_t\in\RR^k$ be a one-hot encoding of the action choice at round $t$, i.e. a $k$-dimensional unit vector in the direction of the action. For example, if the agent at round $t$ chooses action 2, then $x_t=\e_2=(0,1,0,\cdots,0)\in\RR^k$. Additionally, agent $t$ receives an action recommendation $z_t\in\RR^k$ from the planner upon arrival. After selecting action $x_t\in\RR^k$, agent $t$ receives a reward $y_t\in\mathbb{R}$, given by
\begin{equation}
    \label{eq:reward-model-general}
    y_t = \langle\theta,x_t\rangle + g^{(u_t)}_t
\end{equation}
where $g^{(u_t)}_t$ denotes the confounding \emph{baseline reward} which depends on the agent's private type $u_t$. Each $g^{(u_t)}_t$ is drawn from a sub-Gaussian distribution with a sub-Gaussian norm of $\sigma_g$. The social planner's goal is to estimate the treatment effect vector $\theta\in\RR^k$ and maximize the total expected reward of all $T$ agents.

\xhdr{History, beliefs, and action choice.}
As in the body of the paper, the history $H_t$ is made up of all tuples $(z_i, x_i, y_i)$ over all rounds from $i=1$ to $t$. Additionally, before the game starts, the social planner commits to recommendation policy $\pi$, which is known to all agents. Each agent also knows the number of the round $t$ when they arrive. Their private type $u_t$ maps to their prior belief $\cP^{(u_t)}$, which is a joint distribution over the treatment effect $\theta$ and noisy error term $g^{(u)}$. With all this information, the agent $t$ selects the action $x_t$ which they expect to produce the most reward:
\begin{equation}
    \label{eq:selection-function-many-arms}
    \quad\quad x_t := \e_{a_t} \quad\quad\text{ where }\quad\quad a_t := \argmax_{1\leq j\leq k} \E_{\cP^{(u_t)}, \pi_t}\left[\theta^j \ | \ z_t, t \right].
\end{equation}

\subsection{Instrumental Variable Estimate and Finite Sample Approximation Bound}
\label{sec:finite-sample-many-arms}

As in the body of the paper, we view the planner's recommendations as instruments and perform \textit{instrumental variable (IV) regression} to estimate $\theta$.

\xhdr{IV Estimator for $k>1$ Treatments}
Our mechanism periodically solves the following IV regression problem:
given a set $S$ of $n$ observations $(x_i, y_i, z_i)_{i=1}^n$, compute an estimate $\hat \theta_S$ of $\theta$. We consider the following two-stage least square (2SLS) estimator:
\begin{equation}
    \label{eq:theta-hat-many-arms}
    \hat{\theta}_S  = \left(\sum_{i=1}^n z_ix_i^\intercal\right)^{-1}\sum_{i=1}^n z_iy_i,
\end{equation}
where $(\cdot)^{-1}$ denotes the pseudoinverse.

To analyze the 2SLS estimator, we introduce a \textit{compliance matrix} of conditional probabilities that an agent chooses some treatment given a recommendation $\hat{\Gamma}$, given as proportions over a set of $n$ samples $S=(x_i,z_i)_{i=1}^n$, where any entry in $\hat{\Gamma}$ is given as such:
\begin{equation}
    \hat{\Gamma}_{ab}(S) = \hat{\Prob}_S[x=\e_a|z=\e_b]= \frac{\sum_{i=1}^n\1[x=\e_a,z=\e_b]}{\sum_{i=1}^n\1[z=\e_b]}
\end{equation}
Then, we can write the action choice $x_i$ as such:
\begin{equation}
    x_i = \hat{\Gamma} z_i + \eta_i,
\end{equation}
where $\eta_i = x_i - \hat{\Gamma} z_i$. Now, we can rewrite the reward $y_i$ at round $i$ as such:
\begin{align*}
    y_i &= \big\langle\theta,(\hat{\Gamma} z_i + \eta_i)\big\rangle + g^{(u_i)}_i\\
    &= \langle\,\underbrace{\theta\,\hat{\Gamma}}_{\beta}\,, z_i\rangle + \langle\theta,\eta_i\rangle + g^{(u_i)}_i\\
    &= \langle\beta, z_i\rangle + \langle\theta,\eta_i\rangle + g^{(u_i)}_i.
\end{align*}

This formulation allows us to express and bound the error between the treatment effect $\theta$ and its IV estimate $\hat{\theta}$ in \Cref{thm:general-treatment-approximation-bound}.


\subsection{Proof of \Cref{thm:general-treatment-approximation-bound}}
\label{sec:general-treatment-approximation-bound-proof}
\generaltreatmentapproximationbound*

\begin{proof}
Given a sample set $S=(x_i,y_i,z_i)_{i=1}^n$ of size $n$, we form an estimate of the treatment effect $\hat{\theta}_S$ via Two-Stage Least Squares regression (2SLS). In the first stage, we regress $y_i$ onto $z_i$ to get the empirical estimate $\hat{\beta}_S$ and $x_i$ onto $z_i$ to get $\hat{\Gamma}_S$ as such:
\begin{equation}
    \hat{\beta}_S := \left(\sum_{i=1}^n z_iz_i^\intercal \right)^{-1}\left(\sum_{i=1}^n z_iy_i \right)
    \hspace{1cm}\text{and}\hspace{1cm} 
    \hat{\Gamma}_S:=
    \left(\sum_{i=1}^n z_iz_i^\intercal \right)^{-1}\left(\sum_{i=1}^n z_ix_i^\intercal \right)
    \label{eq:beta-gamma-hat-many-arms}
\end{equation}
Now, note that by definition $\theta = (\hat{\Gamma})^{-1}\beta$. 
In the second stage, we take the inverse of $\hat{\Gamma}$ times $\hat{\beta}$ as the predicted causal effect vector $\hat{\theta}_S$, i.e. 
\begin{align*}
    \hat{\theta}_S
    &= \hat{\Gamma}_S^{-1}\hat{\beta}_S\\
    &= \bigg(\sum_{i=1}^nz_ix_i^\intercal\bigg)^{-1}\bigg(\sum_{i=1}^nz_iz_i^\intercal\bigg)\bigg(\sum_{i=1}^nz_iz_i^\intercal\bigg)^{-1}\sum_{i=1}^nz_iy_i\\
    &= \bigg(\sum_{i=1}^nz_ix_i^\intercal\bigg)^{-1}\sum_{i=1}^nz_iy_i
\end{align*}

Hence, the L2-norm of the difference between $\theta$ and $\hat{\theta}_S$ is given as:
\begin{align*}
    \norm{\hat{\theta}_S - \theta}_2
    &= \norm{\bigg(\sum_{i=1}^nz_ix_i^\intercal\bigg)^{-1}\sum_{i=1}^nz_iy_i - \theta}_2\\
    &= \norm{\bigg(\sum_{i=1}^nz_ix_i^\intercal\bigg)^{-1}\sum_{i=1}^nz_i\left(\langle\theta,x_i\rangle + g^{(u_i)}_i\right)^\intercal - \theta}_2\\
    &= \norm{\bigg(\sum_{i=1}^nz_ix_i^\intercal\bigg)^{-1}\bigg(\sum_{i=1}^nz_ix_i^\intercal\theta + \sum_{i=1}^nz_ig^{(u_i)}_i\bigg) - \theta}_2\\
    &= \norm{\theta + \bigg(\sum_{i=1}^nz_ix_i^\intercal\bigg)^{-1}\sum_{i=1}^nz_ig^{(u_i)}_i - \theta}_2\\
    &= \norm{\bigg(\sum_{i=1}^nz_ix_i^\intercal\bigg)^{-1}\sum_{i=1}^nz_ig^{(u_i)}_i}_2\\
    &\leq \norm{\bigg(\sum_{i=1}^nz_ix_i^\intercal\bigg)^{-1}}_2\norm{\sum_{i=1}^nz_ig^{(u_i)}_i}_2\tag{by \Cref{thm:cauchy-schwarz}}\\
    &= \frac{\norm{\sum_{i=1}^nz_ig^{(u_i)}_i}_2}{\sigma_{\min}\left(\sum_{i=1}^nz_ix_i^\intercal\right)}
\end{align*}

Finally, we may bound $\norm{\hat{\theta}_S-\theta}_2$ by upper bounding $\norm{\sum_{i=1}^nz_ig^{(u_i)}_i}_2$ in the following \cref{eq:approximation-bound-numerator-many-arms}.
\begin{lemma}\label{eq:approximation-bound-numerator-many-arms}
For any $\delta \in (0,1)$, with probability at least $1-\delta$, we have
\begin{equation}
    \norm{\sum_{i=1}^nz_ig^{(u_i)}_i}_2 \leq \sigma_g\sqrt{2nk\log(k/\delta)}
\end{equation}
\end{lemma}
\begin{proof}
Recall that the baseline reward $g^{(u)}$ is an independently distributed random variable which, by assumption, has a mean of zero, i.e. $\E[g^{(u)}]=0$. Because of these properties of $g^{(u)}$, with probability at least $1-\delta$ for any $\delta\in(0,1)$, the numerator above satisfies the following upper bound: 
    \begin{align*}
        \norm{\sum_{i=1}^n g^{(u_i)}_iz_i}_2
        &=\sqrt{\sum_{j=1}^k\left(\sum_{i=1}^n g^{(u_i)}_i\1[z_i=\e_j]\right)^2}\\
        &=\sqrt{\sum_{j=1}^k\left(\sum_{i=1}^{n_j} g^{(u_i)}_i\right)^2}\tag{where $n_j=\sum_{i=1}^n\1[z_i=\e_j]$}\\
        &\leq \sqrt{\sum_{j=1}^k\left( \sigma_g\sqrt{2n_j\log(1/\delta_j)}\right)^2} \tag{by \Cref{thm:high-prob-unbounded-chernoff} and, by assumption, $\E[g^{(u)}]=0$}\\
        &\leq \sqrt{\sum_{j=1}^k\left(\sigma_g\sqrt{2n_j\log(k/\delta)}\right)^2} \tag{by a \Cref{thm:union} where $\delta_j=\frac{\delta}{k}$ for all $j$}\\
        &\leq\sqrt{k\left(\sigma_g\sqrt{2n\log(k/\delta)}\right)^2} \tag{since $n_j\leq n$ for all $j$}\\
        &= \sigma_g\sqrt{2nk\log(k/\delta)}
    \end{align*}
 \end{proof}
    
This recovers the stated bound and finishes the proof for \Cref{thm:general-treatment-approximation-bound}.
\end{proof}

Next, we demonstrate a lower bound which the denominator of the approximation bound $A(S,\delta)$ in \Cref{thm:general-treatment-approximation-bound} equals $O(1/\sqrt{|S|})$, where $|S|$ is the size of sample set $S$.

\begin{theorem}[Treatment Effect Confidence Interval for General $k$ Treatments]
\label{thm:confidence-interval-many-arms}
    Let $z_1,\dots,z_n\in\{0,1\}^k$ be a sequence of instruments. Suppose there is a sequence of $n$ agents such that each agent $i$ has private type $u_i$ drawn independently from $\cU$, selects $x_i$ under instrument $z_i$ and receives reward $y_i$. Assume that each agent initially prefers treatment 1, i.e. $x=\e_1$. Let sample set $S=(x_i,y_i,z_i)_{i=1}^n$. Let $r$ be the proportion of recommendations for each treatment $j>1$ and let $1-(k-1)r$ be the proportion of recommendations for treatment 1. Let $p_c$ fraction of agents in the population of agents be compliant over the rounds from which $S$ is collected. For any $\delta\in(0,1)$, if $n\geq\frac{rp_c^2}{\log(k/\delta)}$, then the approximation bound $A(S,\delta)$ is given as such:
        \[ A(S,\delta) \leq \frac{\sigma_g\sqrt{2k\log(k/\delta)}}{\alpha\sqrt{n}}=O\left(\sqrt{\frac{\log(1/\delta)}{n}}\right),\]
        and the IV estimator given by \Cref{eq:iv-estimate-general} satisfies
        \[\norm{\hat{\theta}_S-\theta}_2\leq A(S,\delta)\]
        with probability at least $1-\delta$, where $\alpha>0$ is a constant of proportionality given in \Cref{claim:confidence-interval-denominator-many-arms} below.
\end{theorem}

\begin{proof}
Note that \Cref{thm:general-treatment-approximation-bound} holds in this case and it suffices to demonstrate that the denominator is bounded by $\frac{1}{\alpha n}$.

\begin{claim}[Proportionality of the Denominator of the Approximation Bound for $k$ Treatments]
\label{claim:confidence-interval-denominator-many-arms}
Given all assumptions in \Cref{thm:confidence-interval-many-arms} above, the denominator $\sigma_{\min}\left(\sum_{i=1}^nz_ix_i^\intercal\right)$ of the approximation bound $A(S,\delta)$ is positive and increases proportionally to $n$. Formally, $\sigma_{\min}\left(\sum_{i=1}^nz_ix_i^{\intercal}\right)=\Omega(n)$.
\end{claim}
\begin{proof}
Recall that we assume that every agent initially prefers treatment 1. Thus, whenever agent is recommended any treatment greater than 1 and does not comply, the agent takes treatment 1. At any round $i$, if agent $i$ is always compliant, then $x_i=z_i$; if not, then $x_i=\e_1$. (If $z_i=\e_1$, then $x_i=z_i=\e_1$ always.) Furthermore, at any round $i$ when $x_i=z_i$, the outer product $z_ix_i^\intercal=diag(z_i)$, i.e. a diagonal matrix where the diagonal equals $z_i$. If $x_i=\e_1$, then the outer product 
\[z_ix_i^\intercal= 
\begin{pmatrix}
\uparrow & \uparrow  &  & \uparrow \\
z_i & \mathbf{0} & \cdots & \mathbf{0}\\
\downarrow & \downarrow & & \downarrow
\end{pmatrix},\] which is a $k\times k$ matrix where the first column is $z_i$ and all other entries are 0. Thus, as long as we have at least one sample of each treatment, i.e. at least one round $i$ where $x_i=z_i=\e_j$ for all $1\leq j\leq k$, then the sum $\sum_{i=1}^nz_ix_i^{\intercal}$ is a lower triangular matrix with all positive entries in the diagonal. To illustrate this, let $\mathbf{A}$ denote the expected mean values of the sum $\sum_{i=1}^nz_ix_i^{\intercal}$, such that
\[\E\left[\sum_{i=1}^nz_ix_i^{\intercal}\right] =
n\begin{pmatrix}
1-rk & 0  & \cdots  & & \cdots & 0 \\
r(1-p_c) & rp_c & 0 & \cdots & \cdots& \vdots \\
r(1-p_c)  & 0 & rp_c & 0 & \cdots & \\
\vdots  &  \vdots & 0 & \ddots &  & \vdots \\
\vdots  & \vdots & \vdots &  & \ddots & 0 \\
r(1-p_c)  & 0 & \cdots & \cdots & 0 & rp_c
\end{pmatrix} = n\mathbf{A}.\]

Note that 
\begin{align*}
    &\E\left[\left(\sum_{i=1}^nz_ix_i^{\intercal}\right)^\intercal\left(\sum_{i=1}^nz_ix_i^{\intercal}\right)\right] = \E\left[\left(\sum_{i=1}^nz_ix_i^{\intercal}\right)\right]^\intercal\E\left[\sum_{i=1}^nz_ix_i^{\intercal}\right] \\
    &= n^2\begin{pmatrix}
    (1-rk)^2+(k-1)r^2(1-p_c)^2 & r^2p_c(1-p_c)  & \cdots  & & \cdots & r^2p_c(1-p_c) \\
    r^2p_c(1-p_c) & r^2p_c^2 & 0 & \cdots & \cdots& 0 \\
    r^2p_c(1-p_c)  & 0 & r^2p_c^2 & 0 & \cdots & \vdots \\
    \vdots  &  \vdots & 0 & \ddots &  & \vdots \\
    \vdots  & \vdots & \vdots &  & \ddots & 0 \\
    r^2p_c(1-p_c)  & 0 & \cdots & \cdots & 0 & r^2p_c^2 
    \end{pmatrix}.
\end{align*}

Furthermore, let $\hat{\mathbf{A}}$ denote the empirical approximation of $\mathbf{A}$ over our $n$ samples, given as such:
\[\sum_{i=1}^nz_ix_i^{\intercal} =
n\begin{pmatrix}
1-rk & 0  & \cdots  & & \cdots & 0 \\
r(1-\hat{p}_{c,2}) & r\hat{p}_{c,2} & 0 & \cdots & \cdots & \vdots \\
r(1-\hat{p}_{c,3})  & 0 & r\hat{p}_{c,3} & 0 & \cdots & \vdots \\
\vdots  &  \vdots & 0 & \ddots &  & \vdots \\
\vdots  & \vdots & \vdots &  & \ddots & 0\\
r(1-\hat{p}_{c,k})  & 0 & \cdots & \cdots & 0 & r\hat{p}_{c,k}
\end{pmatrix} = n\hat{\mathbf{A}},\]
where for any $j\geq2$, the empirical proportion of agents who comply with the recommended treatment $j$ is denoted as $\hat{p}_{c,j}$. Note that, since $\E[\hat{p}_{c,j}]=p_c$ for all $j\geq2$, the expected value $\E\big[\hat{\mathbf{A}}\big]=\mathbf{A}$. We may bound the difference between $\hat{p}_{c,j}$ and $p_c$ with high probability, based on the number of times each treatment $j$ is recommended, which is $rn$. Over $n$ samples, with probability at least $1-\delta_j$ for any $\delta_j\in(0,1)$ for any treatment $j$, the proportion $\hat{p}_{c,j}$ satisfies the following:
$\hat{p}_{c,j} \geq p_c - \sqrt{\frac{\log(1/\delta_j)}{2rn}}$. In order for this bound to hold for all $j\geq2$, let $\delta_2=\delta_3=\cdots=\delta_k=\delta/k$. Then, by a union bound, with probability $1-\delta$ for any $\delta\in(0,1)$, the bound $\hat{p}_{c,j} \geq p_c - \sqrt{\frac{\log(k/\delta)}{2rn}}$ holds simultaneously for all $2\leq j\leq k$. Thus, for any $\delta\in(0,1)$ and $n\geq\frac{rp_c^2}{\log(k/\delta)}$, each entry in the diagonal of $\hat{\mathbf{A}}$ is positive. Thus, (since it is a triangular matrix) the eigenvalues of $\hat{\mathbf{A}}$ equal the entries in the diagonal and are all positive. Furthermore, because $\text{rank}(\hat{\mathbf{A}})=\text{rank}(\hat{\mathbf{A}}^\intercal\hat{\mathbf{A}})$, the singular values of $\hat{\mathbf{A}}$ are all positive, as well.

Thus, for $n\geq\frac{rp_c^2}{\log(k/\delta)}$, the minimum singular value
\[\sigma_{\min}\left(\sum_{i=1}^nz_ix_i^{\intercal}\right) = n\sigma_{\min}\big\{\hat{\mathbf{A}}\big\}=n\alpha = \Omega(n),\]
where $\alpha=\sigma_{\min}\big(\hat{\mathbf{A}}\big)>0$ is some (possibly small) constant of proportionality.
\end{proof}

Thus, by \Cref{claim:confidence-interval-denominator-many-arms} and \Cref{thm:general-treatment-approximation-bound}, the approximation bound
\[ A(S,\delta) \leq \frac{\sigma_g\sqrt{2nk\log(k/\delta)}}{n\alpha }=O\left(\sqrt{\frac{\log(1/\delta)}{n}}\right).\]
\end{proof}

\begin{corollary}[Treatment Effect Confidence Interval for General $k$ Treatments]
\label{cor:approximation-bound-general}
Given all assumptions in \Cref{thm:confidence-interval-many-arms}, plus the assumptions that the minimum compliance rate for any arm is at least $1/k$ and the minimum proportion of  treatment 1 recommendations is at least $1/k$, for any $\delta \in (0,1)$, with a large enough sample size $n$, the approximation bound $A(S, \delta)$ is given as such: 
\begin{equation}
    A(S, \delta) = O\left(k\sqrt{\frac{k\log(1/\delta)}{n}} \right)
\end{equation}
\end{corollary}

\begin{proof}
Note that \Cref{claim:confidence-interval-denominator-many-arms} holds in this case and it suffices to demonstrate that the $\alpha$ is bounded by $\frac{1}{k }$.\lsdelete{
Recall that by design, \dndelete{the confidence interval for \Cref{alg:sampling-many-arms} and \Cref{alg:racing-many-arms} can be analyzed separately.} the minimum compliance rate assumption applies for \Cref{alg:racing-many-arms}.}
\dndelete{{\bf Part I (Confidence Interval for \Cref{alg:sampling-many-arms})}
Observe that in \Cref{alg:sampling-many-arms}, for each exploration phase $i$, we determine the recommendation by only looking at the $\ell$ samples of treatment $1$ collected from the first phase and any samples of treatment $2 \leq j < i$ collected thus far. Hence, for a particular treatment $2 \leq j < k$, we can derive the treatment effect confidence interval in a similar fashion to the control-treatment setting in \Cref{sec:iv-estimator}. \\ 
After obtaining the estimates for all treatments $2 \leq j < k$, we have to find the treatment effect for treatment $1$. Without loss of generality, we use the average of all treatment $1$ estimates obtained previously as the  estimate of treatment $1$. Hence, the confidence interval for \Cref{alg:sampling-many-arms} is the same as the one in \Cref{sec:iv-estimator}.\\
Part II (Confidence Interval for \Cref{alg:racing-many-arms})}
We focus on the denominator $\sigma_{\min} \left(\sum_{i=1}^n z_i x_i^\intercal\right)$ of the approximation bound $A(S, \delta)$. Note that since $z_i$ and $x_i$ are one-hot encoded vectors, we have:
\begin{align*}
    \E\left[\sum_{i=1}^n z_i x_i^\intercal \right] &= \sum_{j=1}^k \sum_{i \in S_j} x_i \left( \sum_{i \in S_j} x_i  \right)^\intercal \tag{where $S_j = \{i: z_i = \e_j \}$}\\
    &= n\sum_{j=1}^k v_j v_j^\intercal \tag{where vector $v_j = (v_{j1}, 0, \cdots 0, v_{jj}, 0, \cdots 0)\in\RR^k$}
\end{align*}
where $\forall j: v_{j1} = \frac{r}{k} (1 - p_j)$ is the probability of getting a treatment 1 sample when the recommendation is $j>1$, the term $v_{11} = 1 - r$ is the probability of recommending treatment 1 (since we assume agents always comply with treatment 1 recommendations), and the term $v_{jj} = \frac{r}{k} p_j$ is the probability of getting a treatment $j$ sample when the recommendation is $j>1$. By definition, we can write the denominator term squared as:
\begin{align*}
    \sigma_{\min}\left( \E\left[  \sum_{i=1}^n z_i x_i^\intercal \right] \right)^2 &= \min_{a: \norm{a} = 1} a^\intercal \left( n \sum_{j=1}^k v_j v_j^\intercal\right) a\\
    &= n \left[\min_{a: \norm{a} = 1} (a_1 v_{11})^2 + (a_1 v_{21} + a_2 v_{22})^2 + \dots + (a_1 v_{k1} + a_k v_{kk})^2 \right].
\end{align*}
Also, without loss of generality, assume that $a_1 > 0$ and $\forall j > 1: a_j \leq 0 $.\\
Substituting the expression above with algorithm-specific variables, we have:
\begin{align*}
    &\quad (a_1 v_{11})^2 + (a_1 v_{21} + a_2 v_{22})^2 + \dots + (a_1 v_{k1} + a_k v_{kk})^2\\
    &= a_1^2 (1 - r)^2 + \frac{r^2}{k^2} a_1^2 \sum_{j=2}^k (1 - p_j)^2 + \sum_{j=2}^k a_j^2 p_j^2 + \frac{2r^2}{k^2}a_1 \sum_{j=2}^k a_j p_j(1 - p_j) \\
    &\geq a_1^2 (1 - r)^2 + \frac{r^2}{k^2} a_1^2 \sum_{j=2}^k (1 - p_j)^2 + p_{\min}^2 \sum_{j=2}^k a_j^2 + \frac{2r^2}{k^2} a_1 \sum_{j=2}^k a_j \frac{1}{4} \\
    &\geq a_1^2 (1 - r)^2 + \frac{r^2}{k^2} a_1^2 \sum_{j=2}^k (1 - p_j)^2 + p_{\min}^2 (1 - a_1^2) - \frac{r^2 a_1}{2k^2} \sqrt{(k-1)(1-a_1^2)} 
\end{align*}
where the second line is direct substitution, the third line comes from lower bounding all $p_j^2$ terms with the minimum compliance rate $p_{\min}^2$ and lower bounding $p_j(1 - p_j)$ by $1/4$. The last line comes from the fact that $\norm{a}=1$ and from applying \Cref{thm:cauchy-schwarz} on the last term.\lsdelete{Observe that in \Cref{alg:racing-many-arms}, } Since we assume that the probability of recommending treatment 1 is $1 - r \geq \frac{1}{k}$, we have: 
\begin{align*}
    &\quad \sigma_{\min}\left ( \E \left[  \sum_{i=1}^n z_i x_i^\intercal \right] \right )^2\\ &\geq n \left[ \min_{a: \norm{a} = 1} \frac{a^2}{k^2} + \frac{\left( 1 - \frac{1}{k} \right)^2}{k^2} a_1^2 \sum_{j=2}^k (1 - p_j)^2 + p_{\min}^2 (1 - a_1)^2 - \frac{\left(1 - \frac{1}{k} \right)^2 }{2k^2} a_1 \sqrt{(k-1)(1 - a_1^2)} \right]\\
    &\geq n \left[ \min_{a: \norm{a} = 1} \frac{a_1^2}{k^2} + p_{\min}^2 (1 - a_1)^2 - \frac{(k-1)^2}{2k^4} a_1 \sqrt{(k-1)(1 - a_1^2)}  \right]\\
    &\geq n \left[ \min_{a: \norm{a} = 1} \frac{a_1^2}{k^2} + p_{\min}^2 (1 - a_1)^2 - \frac{a_1 \sqrt{(k-1)(1 - a_1^2)}}{2k^2} \right] \\
    &=n \left[ \left( a_1 \sqrt{\frac{1}{k^2} - p_{\min}^2} - \frac{1}{4k^2} \sqrt{\frac{(k-1)(1 - a_1^2)}{\frac{1}{k^2} - p_{\min}^2}} \right)^2 + p_{\min}^2 - \frac{(k-1)(1 - a_1^2)}{16k^4 \left( \frac{1}{k^2} - p_{\min}^2 \right)} \right]\\
    &\geq n\left[ p_{\min}^2 - \frac{(k-1)(1 - a_1^2)}{16k^4 \left( \frac{1}{k^2} - p_{\min}^2 \right)} \right]\\
    &\geq n \left[ p_{\min}^2 - \frac{(1 - a_1^2)}{16k - 16k^3 p_{\min}^2} \right]
\end{align*}
Since we assume that the minimum compliance rate $p_{\min}\geq\frac{1}{k}$, we have:
\begin{align*}
    p_{\min}^2 \geq \frac{1}{k^2} \Rightarrow 16k - 16k^3 p_{\min}^2 \leq 0
\end{align*}
Therefore, we have $\alpha = \frac{1}{k}$ and 
\begin{align*}
     \sigma_{\min}\left ( \E \left [ \sum_{i=1}^n z_i x_i^\intercal \right] \right )^2 \geq \frac{n}{k^2}
\end{align*}

We apply \Cref{thm:matrix-chernoff} to this matrix to get that, with probability at least $1 - \delta$, for $\delta \in (0,1)$: 
\begin{align*}
    \sigma_{\min} \left(\sum_{i=1}^n z_i x_i^\intercal\right) \geq 
    \sqrt{\frac{n}{k^2} - \log(k/\delta) }
\end{align*}

Hence, we have the approximation bound $A(S, \delta)$ for \Cref{alg:racing-many-arms} is given as
\begin{equation*}
    A(S, \delta) \leq \frac{\sigma_g\sqrt{2nk\log(k/\delta)}}{\sqrt{\frac{n}{k^2} - \log(k/\delta)} \sqrt{n}} = O \left(k \sqrt{\frac{k\log(1/\delta)}{n}} \right)
\end{equation*}

\subsection{Extensions of \Cref{alg:sampling-control-treatment,alg:racing-two-types-mixed-preferences} and Recommendation Policy $\pi_c$ to $k$ Treatments}
\label{sec:general-extension-appendix}
We assume that every agent ---regardless of type--- shares the same prior ordering of the treatments, such that all agents prior expected value for treatment 1 is greater than their prior expected value for treatment 2 and so on. First, \Cref{alg:sampling-many-arms} is a generalization  of \Cref{alg:sampling-control-treatment} which serves the same purpose: to overcome complete non-compliance and incentivize some agents to comply eventually. The incentivization mechanism works the same as in \Cref{alg:sampling-control-treatment}, where we begin by allowing all agents to choose their preferred treatment ---treatment 1--- for the first $\ell$ rounds. Based on the $\ell$ samples collected from the first stage, we then define a number of events $\xi^{(u)}_j$ ---which are similar to event $\xi$ from \Cref{alg:sampling-control-treatment}--- that each treatment $j\geq2$ has the largest expected reward of any treatment and treatment 1 has the smallest, according to the prior of type $u$:
\begin{equation}
    \xi^{(u)}_i := \left( \bar{y}_{\ell}^1 + C \leq \min_{1<j<i} \bar{y}_{\ell}^j - C \ \text{ and } \ \max_{1<j<i} \bar{y}_{\ell}^j + C \leq \mu^{(u)}_i \right),
\end{equation}
where $C = \sigma_g\sqrt{\frac{2\log(3/\delta)}{\ell}} + \frac{1}{4}$ for any $\delta\in(0,1)$ and where $\bar{y}_{\ell}^1$ denotes the mean reward for treatment 1 over the $\ell$ samples of the first stage of \Cref{alg:sampling-many-arms}. Thus, if we set the exploration probability $\rho$ small enough, then some subset of agents will comply with all recommendations in the second stage of \Cref{alg:sampling-many-arms}.



Second, \Cref{alg:racing-many-arms} is a generalization  of \Cref{alg:racing-two-types-mixed-preferences}, which is required to start with at least partial compliance and more rapidly and incentivizes more agents to comply eventually. The incentivization mechanism works the same as in \Cref{alg:sampling-control-treatment}, where we begin by allowing all agents to choose their preferred treatment ---treatment 1--- for the first $\ell$ rounds. Based on the $\ell$ samples collected from the first stage, we then define a number of events ---which are similar to event $\xi$ from \Cref{alg:sampling-control-treatment}--- that each treatment $j\geq2$ has the largest expected reward of any treatment and treatment 1 has the smallest. Thus, if we set the exploration probability $\rho$ small enough, then some subset of agents will comply with all recommendations in the second stage of \Cref{alg:sampling-many-arms}.

\end{proof} 
\dndelete{\Cref{lemma:regret-many-arms}}

\begin{corollary} (Pairwise Treatment Effect Confidence Interval for General $k$ Treatments) 
Given all assumptions in \Cref{cor:approximation-bound-general}, the pairwise approximation bound between any two particular arms $a, b$ is given as
\begin{equation*}
    \abs{(\theta^a - \theta^b) - (\hat{\theta}^a - \hat{\theta}^b)} = A^{ab}(S, \delta) \leq \sqrt{2} A(S, \delta) 
\end{equation*}
where $\hat{\theta}^a$ and $\hat{\theta}^b$ are the IV estimate for the treatment effect of arm $a$ and arm $b$, respectively.
\label{cor:approximation-bound-pairwise}
\end{corollary}
\begin{proof}
We have:
\begin{align*}
    \abs{(\theta^a - \theta^b) - (\hat{\theta}^a - \hat{\theta}^b)} &= \abs{(\hat{\theta}^a - \theta^a) +  (\hat{\theta}^b - \theta^b)}\\
    &\leq \abs{\hat{\theta}^a - \theta^a} + \abs{\hat{\theta}^b - \theta^b} \tag{by Triangle Inequality}\\
    &\leq \sqrt{2} \sqrt{(\hat{\theta}^a - \theta^a)^2 + (\hat{\theta}^b - \theta^b)^2} \tag{by \Cref{thm:cauchy-schwarz}}\\
    &\leq \sqrt{2} \sqrt{\sum_{i=1}^k (\hat{\theta}^i - \theta^i)^2}\\
    &\leq \sqrt{2} A(S, \delta)
\end{align*}
This recovers the stated bound and we only pay a small constant ($\sqrt{2}$) to obtain a pairwise approximation bound from our IV estimator.
\end{proof}

\begin{theorem}
    \label{thm:racing-compliance-many-arms}
     Let $G^{vw}$ denote the gap between the causal effects of any arms $v$ and $w$, i.e. $G^{vw} := \theta^v - \theta^w$ and let $G^v$ denote the smallest gap for arm $v$, i.e. $G^v := \theta^v-\max_{w \neq v}\theta^w = \min_{w \neq v}\theta^v-\theta^w$.

    Let $A^{vw}(S,\delta)$ denote a high-probability upper bound (with probability at least $1-\delta$) on the difference between the true gap $G^{vw}$ (for causal effects $\theta^v$ and $\theta^w$ for arms $v$ and $w$) and its estimate $\widehat{G}^{vw}$ based on the sample set $S$, i.e.
    \[\left|G^{vw}-\widehat{G}^{vw}\right| = \left|\theta^{v}-\theta^w-\left(\hat{\theta}^{v}-\hat{\theta}^w\right)\right| < A^{vw}(S,\delta).\]
    
    Furthermore, let $A^v(S,\delta)$ denote a high-probability upper bound on the difference between the true minimum gap $G^v$ for arm $v$ and its empirical estimate $\widehat{G}^v$ based on sample set $S$, i.e.
    \[\left|G^{v}-\widehat{G}^{v}\right| = \left|\theta^{v}-\theta^{w_{\min}}- \left(\hat{\theta}^{v}-\hat{\theta}^{w_{\min}}\right)\right| < A^v(S,\delta),\]
    where $w_{\min}=\argmin_{w\neq v}\theta^v-\theta^w$. For shorthand, let $A^v_q$ denote the best (i.e. smallest) approximation bound $A^v(S_q^{\text{BEST}},\delta)$ by phase $q$.
    
    Recall that \Cref{alg:racing-many-arms} is initialized with samples $S_0 = (x_i,y_i,z_i)_{i=1}^{|S_0|}$. Any agent at time $t$ with type $u_t$ will comply with recommendation $z_t=\e_v$ for arm $v$ from policy $\pi_t$ according to \Cref{alg:racing-many-arms}, if the following holds for some $\tau\in(0,1)$:
    \[A^v(S_0,\delta) \leq \tau\Prob_{\pi_t,\cP^{(u_t)}}[G^v\geq\tau]/4.\]
\end{theorem}

\begin{proof} Any agent at time $t$ with type $u_t$ will comply with an arm $v$ recommendation $z_t=\e_v$ from policy $\pi_t$ following \Cref{alg:racing-many-arms}, if the following holds: For any two treatments $v,w \in B$,
\begin{align*}
    \E_{\pi_t,\cP^{(u_t)}}[\theta^v - \theta^w | z_t=\e_v]\Prob_{\pi_t,\cP^{(u_t)}}[z_t=\e_v] \geq 0.
\end{align*}
We will prove a stronger statement:
\begin{align*}
    \E_{\pi_t,\cP^{(u_t)}}[\theta^v - \max_{w \neq v} \theta^w| z_t=\e_v]\Prob_{\pi_t,\cP^{(u_t)}}[z_t=\e_v] \geq 0.
\end{align*}

We can prove this in largely the same way as we proved \Cref{lemma:bic-racing-control-treatment-0} in \Cref{sec:bic-racing-type-0}: we simply replace $\theta$ and $A^v_q$ in the proof for \Cref{lemma:bic-racing-control-treatment-0} with $G^v$ and $A^v_q$, respectively.

The clean event $C$ is given as:
\[\cC:=\left(\forall q\geq0: |G^v-\widehat{G}^v|< A^v_q\right).\] By \Cref{cor:approximation-bound-general}, for event $\cC$, the failure probability $\Prob[\neg\cC]\leq\delta.$ We assume that
\[\delta\leq\frac{\tau\Prob_{\pi_t,\cP^{(u_t)}}[G^v\geq\tau]}{2\tau\Prob_{\pi_t,\cP^{(u_t)}}[G^v\geq\tau]+2}.\]

First, we marginalize  $\E_{\pi_t,\cP^{(u_t)}}[\theta^v - \max_{w \neq v} \theta^w| z_t=\e_v]\Prob_{\pi_t,\cP^{(u_t)}}[z_t=\e_v]$ based on the clean event $\cC$, such that 
\begin{align*}
    &\quad\E_{\pi_t,\cP^{(u_t)}}[G^v| z_t=\e_v]\Prob_{\pi_t,\cP^{(u_t)}}[z_t=\e_v]\\
    &= \E_{\pi_t,\cP^{(u_t)}}[G^v| z_t=\e_v,\cC]\Prob_{\pi_t,\cP^{(u_t)}}[z_t=\e_v,\cC] + \E_{\pi_t,\cP^{(u_t)}}[G^v| z_t=\e_v,\neg\cC]\Prob_{\pi_t,\cP^{(u_t)}}[z_t=\e_v,\neg\cC]\\
    &\geq \E_{\pi_t,\cP^{(u_t)}}[G^v| z_t=\e_v,\cC]\Prob_{\pi_t,\cP^{(u_t)}}[z_t=\e_v,\cC] - \delta\\
    &\geq \E_{\pi_t,\cP^{(u_t)}}[G^v| z_t=\e_v,\cC]\Prob_{\pi_t,\cP^{(u_t)}}[z_t=\e_v,\cC] - \delta\\
    &\geq \E_{\pi_t,\cP^{(u_t)}}[G^v| z_t=\e_v,\cC]\Prob_{\pi_t,\cP^{(u_t)}}[z_t=\e_v,\cC] - \frac{\tau\Prob_{\pi_t,\cP^{(u_t)}}[G^v\geq\tau]}{2\tau\Prob_{\pi_t,\cP^{(u_t)}}[G^v\geq\tau]+2}.
\end{align*}

Next, we marginalize $\E_{\cP^{(u_t)},\pi_t}[G^v|z_t=\e_v,\cC]\Prob_{\cP^{(u_t)},\pi_t}[z_t=\e_v,\cC]$ based on four possible ranges which $G^v$ lies on:
\begin{equation} 
    \begin{split}
    &\quad\E_{\cP^{(u_t)},\pi_t}[G^v|z_t=\e_v,\cC]\Prob_{\cP^{(u_t)},\pi_t}[z_t=\e_v,\cC\\
    &=
    \E_{\pi_t,\cP^{(u_t)}}[G^v|z_t=\e_v, \cC, G^v\geq \tau]\Prob_{\pi_t,\cP^{(u_t)}}[z_t=\e_v, \cC, G^v\geq \tau] \\
    &\ + \E_{\pi_t,\cP^{(u_t)}}[G^v|z_t=\e_v, \cC, 0 \leq G^v < \tau]\Prob_{\pi_t,\cP^{(u_t)}}[z_t=\e_v, \cC, 0 \leq G^v < \tau] \\
    &\ + \E_{\pi_t,\cP^{(u_t)}}[G^v|z_t=\e_v, \cC, -2A^v_q < G^v < 0]\Prob_{\pi_t,\cP^{(u_t)}}[z_t=\e_v, \cC, -2A^v_q < G^v < 0] \\
    &\ + \E_{\pi_t,\cP^{(u_t)}}[G^v|z_t=\e_v, \cC, G^v \leq -2A^v_q]\Prob_{\pi_t,\cP^{(u_t)}}[z_t=\e_v, \cC, G^v\leq -2A^v_q] \label{eq:racing-bic-cases-many-arms}
    \end{split}
\end{equation}

Because $A^v_q$ is the smallest approximation bound derived from samples collected over any phase $q$ of \Cref{alg:racing-many-arms} (including the initial sample set $S_0$), the following holds:
\begin{align*}
    2A^v_q &\leq 2A^v(S_0,\delta)\\
    &\leq \frac{\tau\Prob_{\pi_t,\cP^{(u_t)}}[G^v\geq\tau]}{2}\tag{by assumption $A^v(S_0,\delta)\leq\tau\Prob_{\pi_t,\cP^{(u_t)}}[G^v\geq\tau]/4$}\\
    &\leq \tau
\end{align*}

Conditional on $\cC$, $|G^v-\widehat{G}^v_q|<A^v_q$. Thus, if $G^v\geq\tau\geq2A^v_q$, then $\widehat{G}^v_q\geq\tau-A^v_q\geq A^v_q$, which invokes the stopping criterion for the while loop in \Cref{alg:racing-many-arms}. Thus, all other arms must have been eliminated from the race before phase $q=1$ and arm $v$ is recommended almost surely throughout \Cref{alg:racing-many-arms}, i.e. $\Prob_{\pi_t,\cP^{(u_t)}}[z_t=\e_v, \cC, G^v \geq \tau] = \Prob_{\pi_t,\cP^{(u_t)}}[\cC, G^v\geq \tau]$. Similarly, if $G^v\leq-2A^v_q$, then $\widehat{G}^v_q\leq-A^v_q$ by phase $q=1$ and arm $v$ is recommended almost never, i.e. $\Prob_{\pi_t,\cP^{(u_t)}}[z_t=\e_v, \cC, G^v<-2A^v_q] = 0$. Substituting in these probabilities (and substituting minimum possible expected values), we proceed:
\begin{align*} 
    &\quad\E_{\cP^{(u_t)},\pi_t}[G^v|z_t=\e_v,\cC]\Prob_{\cP^{(u_t)},\pi_t}[z_t=\e_v,\cC]\\
    &\geq
   \tau\Prob_{\pi_t,\cP^{(u_t)}}[\cC, G^v\geq \tau]-2A^v_q\Prob_{\pi_t,\cP^{(u_t)}}[z_t=\e_v, \cC, -2A^v_q < G^v < 0] \\
   &\geq
   \tau\Prob_{\pi_t,\cP^{(u_t)}}[\cC, G^v\geq \tau]-\frac{\tau\Prob_{\pi_t,\cP^{(u_t)}}[G^v\geq\tau]}{2}\\
   &\geq
   \tau\Prob_{\pi_t,\cP^{(u_t)}}[\cC|G^v\geq \tau]\Prob_{\pi_t,\cP^{(u_t)}}[G^v\geq \tau]-\frac{\tau\Prob_{\pi_t,\cP^{(u_t)}}[G^v\geq\tau]}{2}\\
   &\geq
   (1-\delta)\tau\Prob_{\pi_t,\cP^{(u_t)}}[G^v\geq \tau]-\frac{\tau\Prob_{\pi_t,\cP^{(u_t)}}[G^v\geq\tau]}{2}\\
   &\geq
   \left(\frac{1}{2}-\frac{\tau\Prob_{\pi_t,\cP^{(u_t)}}[G^v\geq\tau]}{2\tau\Prob_{\pi_t,\cP^{(u_t)}}[G^v\geq\tau]+2}\right)\tau\Prob_{\pi_t,\cP^{(u_t)}}[G^v\geq \tau]\\
    &= \frac{\tau\Prob_{\pi_t,\cP^{(u_t)}}[G^v\geq\tau]}{2\tau\Prob_{\pi_t,\cP^{(u_t)}}[G^v\geq\tau]+2}.
\end{align*}

Putting everything together, we get that
\begin{align*}
   &\quad\E_{\cP^{(u_t)},\pi_t}[G^v|z_t=\e_v]\Prob_{\cP^{(u_t)},\pi_t}[z_t=\e_v]\\
    &= \E_{\cP^{(u_t)},\pi_t}[G^v|z_t=\e_v,\cC]\Prob_{\cP^{(u_t)},\pi_t}[z_t=\e_v,\cC] + \E_{\cP^{(u_t)},\pi_t}[G^v|z_t=\e_v,\neg\cC]\Prob_{\cP^{(u_t)},\pi_t}[z_t=\e_v,\neg\cC]\\
    &\geq \frac{\tau\Prob_{\pi_t,\cP^{(u_t)}}[G^v\geq\tau]}{2\tau\Prob_{\pi_t,\cP^{(u_t)}}[G^v\geq\tau]+2} -\frac{\tau\Prob_{\pi_t,\cP^{(u_t)}}[G^v\geq\tau]}{2\tau\Prob_{\pi_t,\cP^{(u_t)}}[G^v\geq\tau]+2}\\
    &= 0.
\end{align*}

Thus, as long as $A^v(S_0,\delta) \leq \tau\Prob_{\pi,\cP^{(u)}}[G^v\geq\tau]/4$, any agent of type $u$ will comply with a recommendation of arm $v$ from recommendation policy $\pi$ according to \Cref{alg:racing-many-arms}.
\end{proof}

\lsdelete{
\begin{remark}
    Let $S$ be a sample set which produces pairwise approximation bounds between arm 1 and any two arms $a,b>1$, i.e. $A^{a1}(S,\delta)+A^{b1}(S,\delta)$. These arm 1 pairwise bounds can be used as an approximation bound between arms $a$ and $b$:
    \begin{align*}
    |\theta^a-\theta^b - (\hat{\theta}^a - \hat{\theta}^b)| 
    &= |\theta^a-\theta^1 -\theta^b+\theta^1 - (\hat{\theta}^a -\hat{\theta}^1 - \hat{\theta}^b + \hat{\theta}^1)|\\
    &\leq |\theta^a-\theta^1 - (\hat{\theta}^a -\hat{\theta}^1)| + |\hat{\theta}^b - \hat{\theta}^1 - (\theta^b-\theta^1)| \tag{by \Cref{thm:cauchy-schwarz}}\\
    &= |\theta^a-\theta^1 - (\widehat{\theta^a-\theta^1})| + |\theta^b-\theta^1 - (\widehat{\theta^b-\theta^1})|\\
    &\leq A^{a1}(S,\delta) + A^{b1}(S,\delta) \tag{by \Cref{cor:approximation-bound-pairwise}}
\end{align*}
\end{remark}
}

Finally, we present the (expected) regret from the $k$ treatment extension of policy $\pi_c$ given in \Cref{def:policy-pi-many-arms}.

\regretmanyarms*
\begin{proof} Let $\theta^*$ be the best treatment effect overall and the gap between $\theta^*$ and any treatment effect $\theta^i$ be $\Delta_i = \abs{\theta^* - \theta_i}$.
Recall that the clean event $\cC$ entails that the approximation bound holds for all rounds. If event $\cC$ fails, then we can only bound the pseudo-regret by the maximum value, which is at most $T \min_i \Delta_i$. \\
For the rest of this proof, assume that the event $\cC$ holds for every round of \Cref{alg:racing-many-arms}. This proof follows the standard technique from \cite{active-arms-elimination-2006}. Since $\cC$ holds, we have $\Delta_i \leq A(S_{L_2}, \delta) + \abs{\hat{\theta^*} - \hat{\theta^i}}$ for any treatment $i$, where $A(S_{L_2}, \delta)$ is the approximation bound based on $S_{L_2}$ samples of \Cref{alg:racing-many-arms}. Before the stopping criteria is invoked, we also have $\abs{\hat{\theta^*} - \hat{\theta^i}} \leq A(S_{L_2, \delta})$. Hence, the gap between the best treatment effect and any other treatment effect is:
\begin{align*}
    \Delta_i \leq 2 A(S_{L_2}, \delta) \leq  \frac{2\sigma_g\sqrt{2k\log(2kT/\delta)}}{\sqrt{\frac{L_2}{k^2} - \log(k/\delta) }},
\end{align*}
where $\sigma_g$ is the variance parameter for the baseline reward $g^{(u)}$ and $\alpha_2$ is defined as in \Cref{claim:confidence-interval-denominator-many-arms} relative to the proportion $r_2=1/|B|$ of recommendations for each treatment during \Cref{alg:racing-many-arms} and the proportion of compliant agents $p_{c_2}$. Hence, we must have eliminated treatment $i$ by round \[L_2=\frac{8k\sigma_g^2\log(2kT/\delta)}{\Delta_i^2 \left(\frac{1}{k^2} - \log(k/\delta) \right)},\]
assuming that $L_2\geq\frac{p_c^2}{k\log(k/\delta)}$ (in order to satisfy the criterion for \Cref{thm:confidence-interval-many-arms}).
During \Cref{alg:racing-many-arms}, the social planner gives out $|B|$ recommendations for each treatment $i \in B$ sequentially. Hence, the contribution of each treatment $i$ for each phase is $\Delta$. Conditioned on event $\cC$, the treatment $a^*$ at the end of \Cref{alg:racing-many-arms} is the best treatment overall; so, no more regret is collected after \Cref{alg:racing-many-arms} is finished. 

If treatment 1 is not the winner, then we accumulate $R_1(T)=\Delta_i\left((1-k\rho)L_1+L_2/k\right)$ regret for treatment 1. If some other treatment $i>1$ is not the winner, then we accumulate $R_i(T)=\Delta_i\left(\rho L_1+L_2/k\right)$ regret for treatment $i$. Hence, the total regret accumulated in \Cref{alg:racing-many-arms} is:
\begin{align*}
    R(T) \leq \Delta_i\left((1-\rho)L_1+\left(\frac{k-1}{k}\right)L_2\right)
    \leq (1-\rho)L_1\Delta_i +\frac{8(k-1)\sigma_g^2\log(2kT/\delta)}{\Delta_i\left(\frac{1}{k^2} - \log(k/\delta) \right)}
\end{align*}
Observe that the pseudo-regret of the combined recommendation policy is at most that of \Cref{alg:racing-many-arms} plus $\Delta = \min_i \Delta_i$ per each round of \Cref{alg:sampling-many-arms}. Alternatively, we can also upper bound the regret by $\Delta$ per each round of the combined recommendation policy. Following the same argument as \Cref{lemma:expost-regret}, we can derive the pseudo-regret of the policy $\pi_c$ for $k$ treatments: 
\begin{align*}
    R(T) \leq \min\left(L_1(1-\rho)\Delta_i +\frac{8(k-1)\sigma_g^2\log(2kT/\delta)}{\Delta_i\left(\frac{1}{k^2} - \log \left(k/\delta\right) \right)}, T\Delta \right) \leq L_1 + O(k\sqrt{kT\log(kT/\delta)}).
\end{align*}

For the expected regret, we can set the parameters $\delta$ and $L_1$ in terms of the time horizon $T$, in order to both guarantee compliance throughout policy $\pi_c$ and to obtain sublinear expected regret bound relative to $T$. 

First, we must guarantee that the failure probability $\delta$ in \Cref{alg:racing-many-arms} is small, i.e. $\delta = 1/T^2$. To meet our compliance condition for \Cref{alg:racing-many-arms}, we must set 
\begin{align*}
    \delta \leq \frac{\tau \Prob_{\cP^{(u)}}[\theta \geq \tau]}{2(\tau \Prob_{\cP^{(u)}}[\theta \geq \tau] + 1)}
\end{align*}
for some constant $\tau \in (0,1)$. Hence, we can set $T$ sufficiently large such that, for any $\delta in (0,1)$, we have 
\begin{align*}
    T \geq \frac{1}{\sqrt{\delta}} \geq \sqrt{\frac{2(\tau \Prob_{\cP^{(u)}}[\theta \geq \tau] + 1)}{\tau \Prob_{\cP^{(u)}}[\theta \geq \tau]}}
\end{align*}
We also recall that the length $L_1$ of \Cref{alg:sampling-many-arms} needs to be sufficiently large so that $p_{c_2}$ fraction of agents comply in \Cref{alg:racing-many-arms}. Moreover, we accumulate linear regret in each round of \Cref{alg:sampling-many-arms}. Hence, in order to guarantee sublinear regret, we also require that $T$ satisfies the following:
\begin{align*}
    T \geq L_1^2 = (\ell+\ell/\rho)^2
\end{align*}

Finally, recall that the clean event $\cC$ in \Cref{alg:racing-many-arms} holds with probability at least $1 - \delta$ for any $\delta \in (0,1)$. Conditioned on the failure event $\neg \cC$, policy $\pi_c$ accumulates at most linear pseudo-regret in terms of $T$. Thus, in expectation, it accumulates at most $T \max_{i, j} \abs{\theta^i - \theta^j} \delta$ regret
 
Therefore, we can derive the expected regret of $k$ treatment recommendation policy $\pi_c$ as:
\begin{align*}
    \E_{\cP^{(u)}}[R(T)] &= \E[R(T) | \neg \cC] \Prob_{\cP^{(u)}}[\neg \cC] + \E[R(T)|\cC]\Prob_{\cP^{(u)}}[\cC]\\
    &\leq T\delta + (L_1 + O(\sqrt{kT\log(kT/\delta)}))\\
    &= \frac{1}{T} + (\sqrt{T} + O(k\sqrt{kT \log(kT)}))\\
    &= \frac{1}{T} + O(k\sqrt{kT\log(kT)})\\
    &= O(k\sqrt{kT\log(kT)})
\end{align*}
Therefore, assuming that all hyperparameters $\delta, L_1$ are set to incentivize compliance for some nonzero proportion of agents throughout $\pi_c$ and assuming that $T$ is sufficiently large so as to satisfy the conditions above, policy $\pi_c$ (for $k$ treatments) achieves sublinear regret.
\end{proof}

\section{Experiments Omitted from \Cref{sec:experiment}}

In this section, we present additional experiments to evaluate \Cref{alg:sampling-control-treatment} and \Cref{alg:racing-two-types-mixed-preferences}, which were previously omitted from \Cref{sec:experiment}. Our code is available here: \url{https://github.com/DanielNgo207/Incentivizing-Compliance-with-Algorithmic-Instruments}. We are interested in (1) the effect of different prior choices on the exploration probability $\rho$, (2) comparing the approximation bound in \Cref{alg:sampling-control-treatment} to that of \Cref{alg:racing-two-types-mixed-preferences} and (3) the total regret accumulated by the combined recommendation policy. Firstly, we observed that the exploration probability $\rho$ in \Cref{fig:plot} is small, leading to slow improvement in accuracy of \Cref{alg:sampling-control-treatment}. Since $\rho$ depends on the event $\xi$ (as defined in \Cref{eq:rho}), we want to investigate whether changes in the agents' priors would increase the exploration probability. Secondly, we claimed earlier in the paper that the estimation accuracy increases much quicker in \Cref{alg:racing-two-types-mixed-preferences} compared to \Cref{alg:sampling-control-treatment}. This improvement motivates the social planner to move to \Cref{alg:racing-two-types-mixed-preferences}, granted there is a large enough portion of agents that comply with the recommendations. Finally, while we provide a regret guarantee in \Cref{lemma:expost-regret}, it is not immediately clear how the magnitude of \Cref{alg:sampling-control-treatment} length $L_1$ would affect the overall regret. There is a tradeoff: if we run \Cref{alg:sampling-control-treatment} for a small number of rounds, then it would not affect the regret by a significant amount, but a portion of the agents in \Cref{alg:racing-two-types-mixed-preferences} may not comply. For our combined recommendation policy, we run \Cref{alg:sampling-control-treatment} until it is guaranteed that type 0 agents will comply in \Cref{alg:racing-two-types-mixed-preferences}.

\paragraph{Experimental Description}
For \Cref{alg:sampling-control-treatment}, we consider a setting with two types of agents: type 0 who are initially never-takers, and type 1 who are initially always takers. 
For \Cref{alg:racing-two-types-mixed-preferences}, we consider a setting with two types of agents: type 0 who are compliant, and type 1 who are initially always-takers. We let each agent's prior on the treatment be a truncated Gaussian distribution between $-1$ and $1$. The noisy baseline reward $g_t^{(u_t)}$ for each type $u$ of agents is drawn from a Gaussian distribution $\cN(\mu_{g^{(u)}}, 1)$, with its mean $\mu_{g^{(u)}}$ also drawn from a Gaussian prior. We let each type of agents have equal proportion in the population, i.e. $p_0 = p_1 = 0.5$. 

For the first experiment, we are interested in finding the correlation between the exploration probability $\rho$ and different prior parameters, namely the difference between mean baseline rewards $\mu_{g^{(1)}} - \mu_{g^{(0)}}$ and the variance of Gaussian prior on the treatment effect $\theta$. Similar to the experiment in \Cref{sec:experiment}, we use Monte Carlo simulation by running the first stage of \Cref{alg:sampling-control-treatment} with varying choices of the two prior parameters above. From these initial samples, we calculate the probability of event $\xi$, and subsequently the exploration probability $\rho$. For the second experiment, we are interested in finding when agents of type 1 also comply with the recommendations. This shift in compliance depends on a constant $\tau$ (as defined in \Cref{lemma:bic-racing-control-treatment-0}). We find two values of the constant $\tau$ that minimizes the number of samples needed to guarantee that agents of type 0 and type 1 are compliant in \Cref{alg:racing-two-types-mixed-preferences} (as defined in \Cref{lemma:racing-compliance-sampling-ell-control-treatment}). After this, \Cref{alg:racing-two-types-mixed-preferences} is run for increasing number of rounds. Similar to the \Cref{alg:sampling-control-treatment} experiment, we repeated calculate the IV estimate of the treatment effect and compare it to the naive OLS estimate over the same samples as a benchmark. On a separate attempt, we evaluate the combined recommendation policy by running \Cref{alg:sampling-control-treatment} and \Cref{alg:racing-two-types-mixed-preferences} successively using the priors above. We calculate the accumulated regret of this combined policy using the pseudo-regret notion (as defined in \Cref{def:pseudo-regret}).
\def \hfillx {\hspace*{-\textwidth} \hfill}

\begin{table}[ht]
        \begin{minipage}{0.475\textwidth}
            \centering
            \begin{tabular}[t]{|c|c|}
\hline
Expected gap & Exploration \\ $\E[\mu_{g^{(1)}}-\mu_{g^{(0)}}]$ & probability $\rho$ \\ \hline
-0.5 & 0.004480 \\ \hline
-0.4 & 0.004975 \\ \hline
-0.3 & 0.007936 \\ \hline
-0.2 & 0.003984 \\ \hline
-0.1 & 0.003488 \\ \hline
0.1  & 0.003984 \\ \hline
0.2  & 0.004480 \\ \hline
0.3  & 0.003488 \\ \hline
0.4  & 0.004480 \\ \hline
0.5  & 0.003488 \\ \hline
\end{tabular}
            \caption{Upper bounds on exploration probability $\rho$ to incentivize partial compliance with respect to different gaps $\E[\mu_{g^{(1)}}-\mu_{g^{(0)}}]$}
            \label{table:rho-gap}
        \end{minipage}
        \hfillx
        \begin{minipage}{0.5\textwidth}
            \centering
            \begin{tabular}[t]{|c|c|}
\hline
Variance in prior & Exploration \\ over treatment effect & probability $\rho$  \\ \hline
0.1 & 0.002561 \\ \hline
0.2 & 0.003112 \\ \hline
0.3 & 0.002561 \\ \hline
0.4 & 0.002561 \\ \hline
0.5 & 0.003982 \\ \hline
0.6 & 0.004643 \\ \hline
0.7 & 0.005422 \\ \hline
0.8 & 0.002790 \\ \hline
0.9 & 0.001389 \\ \hline
1   & 0.003488 \\ \hline
\end{tabular}
            \caption{Upper bounds on exploration probability $\rho$ to incentivize partial compliance with respect to different variances in the prior over treatment effect $\theta$}
            \label{table:rho-var}
        \end{minipage}
    \end{table}
    


\paragraph{Results}
In \Cref{table:rho-gap} and \Cref{table:rho-var}, we calculate the exploration probability $\rho$ with different initialization of the agents' priors. In \Cref{table:rho-gap}, we let the mean baseline reward of type 1  $\mu_{g^{(1)}}$ be drawn from $\cN(0.5, 1)$ and the mean baseline reward of type 0 $\mu_{g^{(0)}}$ be drawn from $\cN(c, 1)$ with $c \in [0, 1]$. The gap between these priors is defined as $\E[\mu_{g^{(1)}}] - \E[\mu_{g^{(0)}}]$. We observe that the exploration probability does not change monotonically with increasing gap between mean baseline reward. In \Cref{table:rho-var}, we calculate the exploration probability $\rho$ with different variance in prior over treatment effect $\theta$. Similarly, in \Cref{table:rho-gap}, we observe that the exploration probability $\rho$ does not change monotonically with increasing variance in prior over $\theta$. In both tables, $\rho$ value lies between $[0.001, 0.008]$, which implies infrequent exploration by \Cref{alg:sampling-control-treatment}. This slow rate of exploration is also reflected in \Cref{fig:plot}, which motivates the social planner to transition to \Cref{alg:racing-two-types-mixed-preferences}.

\begin{figure}[ht]
\includegraphics[width=.75\linewidth]{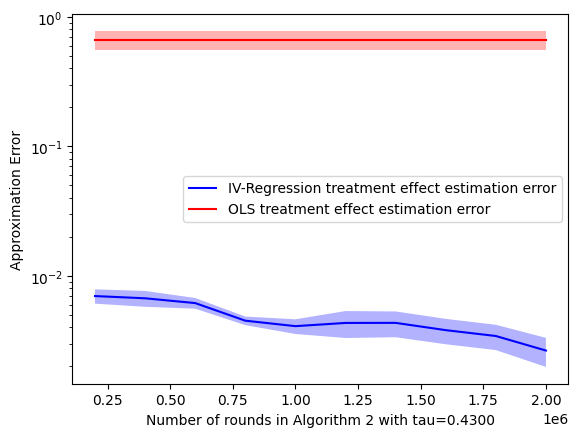}
\centering
\vspace{-6mm}
\caption{Approximation bound using IV regression and OLS during
\Cref{alg:racing-two-types-mixed-preferences} with $\tau = 0.43$. The $y$-axis uses a log scale. Results are averaged over 5 runs; error bars represent one standard error.}
\label{fig:racing-alone}
\end{figure}

In \Cref{fig:racing-alone}, we compare the approximation bound on $\abs{\theta - \hat{\theta}}$ between the IV estimate $\hat{\theta}$ and the naive estimate for \Cref{alg:racing-two-types-mixed-preferences}. In our experiments, the constant $\tau$ generally lies within $[0.4, 0.6]$. Similar to the experiment in \Cref{sec:experiment}, we let the hidden treatment effect $\theta = 0.5$, type 0 and type 1 agents' priors on the treatment effect be $\cN(-0.5, 1)$ and $\cN(0.9, 1)$ --- each truncated into $[-1,1]$ --- respectively. We also let the mean baseline reward for type 0 and type 1 agents be $\mu_{g^{(0)}} \sim \cN(0,1)$ and $\mu_{g^{(1)}} \sim \cN(0.1,1)$, respectively. With these priors, we have found a suitable value of $\tau = 0.43$ for \Cref{alg:racing-two-types-mixed-preferences}. Instead of using the theoretical bound on $\ell$ in \Cref{lemma:racing-compliance-sampling-ell-control-treatment}, we compare the approximation bound $\abs{\theta - \hat{\theta}}$ with the conditions in \Cref{lemma:bic-racing-control-treatment-0}. 
In \Cref{fig:racing-alone}, the IV estimate consistently outperform the naive estimate for any number of rounds. Furthermore, we observe that the scale of the IV estimate approximation bound in \Cref{fig:racing-alone} is much smaller than that of \Cref{fig:plot}. This difference shows the improvement of \Cref{alg:racing-two-types-mixed-preferences} over \Cref{alg:sampling-control-treatment} on estimating the treatment effect $\theta$. It takes \Cref{alg:racing-two-types-mixed-preferences} a small number of rounds to get a better estimate than \Cref{alg:sampling-control-treatment} due to the small exploration probability $\rho$.

\end{document}